\documentclass[twoside,11pt]{article}

\usepackage{jmlr2e}

\usepackage{amsmath,amsfonts}
\usepackage{amssymb}
\usepackage{algorithm}
\usepackage{algpseudocode}
\usepackage{pgfplots}
\usepackage{graphicx}
\usepackage{subfigure}
\usepackage{tikz}
\usetikzlibrary{decorations.fractals}
\usetikzlibrary{patterns,arrows,shapes,decorations.pathreplacing,backgrounds,automata}

\newtheorem{assumption}{Assumption}[theorem]

\renewcommand{\P}{\mathbb{P}}
\newcommand{\Q}{\mathbb{Q}}
\newcommand{\R}{\mathbb{R}}
\newcommand{\bR}{\mathbb{R}}
\newcommand{\E}{\mathbb{E}}

\newcommand{\N}{\mathbb{N}}
\newcommand{\F}{\mathcal{F}}

\newcommand{\B}{\mathcal{B}}

\newcommand{\cW}{\mathcal{W}}

\newcommand{\cX}{\mathcal{X}}

\newcommand{\cC}{\mathcal{C}}

\newcommand{\cP}{\mathcal{P}}

\newcommand{\cL}{\mathcal{L}}

\newcommand{\eps}{\varepsilon}

\newcommand{\JSD}{\operatorname{JSD}}
\newcommand{\KL}{D_{\operatorname{KL}}}
\newcommand{\rhod}{\rho_{\operatorname{d}}}
\newcommand{\mud}{\mu_{\operatorname{d}}}
\newcommand{\Div}{\operatorname{div}}

\newcommand{\TV}{\operatorname{TV}}
\newcommand{\Leb}{\operatorname{Leb}}

\newcommand{\norm}[1]{\| #1\|}

\newcommand{\ab}[1]{\langle #1\rangle}

\newcommand{\lbar}[1]{\underline{#1}}
\newcommand{\ubar}[1]{\overline{#1}}

\firstpageno{1}

\begin{document}

\title{GANs as Gradient Flows that Converge}

\author{\name Yu-Jui Huang \email yujui.huang@colorado.edu \\
       \addr Department of Applied Mathematics\\
       University of Colorado\\
       Boulder, CO 80309-0526, USA
       \AND
       \name Yuchong Zhang \email yuchong.zhang@utoronto.ca \\
       \addr Department of Statistical Sciences\\
       University of Toronto\\
       Toronto, ON M5G 1Z5, Canada}


\maketitle

\begin{abstract}
This paper approaches the unsupervised learning problem by gradient descent in the space of probability density functions. A main result shows that along the {\it gradient flow} induced by a distribution-dependent ordinary differential equation (ODE), the unknown data distribution emerges as the long-time limit. That is, one can uncover the data distribution by simulating the distribution-dependent ODE. Intriguingly, the simulation of the ODE is shown equivalent to the training of generative adversarial networks (GANs). This equivalence provides a new ``cooperative'' view of GANs and, more importantly, sheds new light on the divergence of GANs. In particular, it reveals that the GAN algorithm implicitly minimizes the mean squared error (MSE) between two sets of samples, and this MSE fitting alone can cause GANs to diverge.
To construct a solution to the distribution-dependent ODE, we first show that the associated nonlinear Fokker-Planck equation has a unique weak solution, by the Crandall-Liggett theorem for differential equations in Banach spaces. Based on this solution to the Fokker-Planck equation, we construct a unique solution to the ODE, using Trevisan's superposition principle. The convergence of the induced gradient flow to the data distribution is obtained by analyzing the Fokker-Planck equation.  
\end{abstract}


\section{Introduction}\label{sec:intro}
A central theme in machine learning is to uncover the underlying distribution of observed data points. Mathematically, this can be stated as 
\begin{equation}\label{unsupervised}
\min_{\rho\in\cP(\R^d)} d(\rho,\rhod),
\end{equation}
where $\cP(\R^d)$ is the set of probability density functions on $\R^d$, $\rhod\in\cP(\R^d)$ is the {\it unknown} data distribution, and $d(\cdot,\cdot)$ is a metric (or semi-metric) on $\cP(\R^d)$.
Unlike traditional parameter fitting by maximum likelihood methods, {\it generative adversarial networks} (GANs) in \cite{Goodfellow14} view \eqref{unsupervised} as a two-player non-cooperative game:  a {\it generator} actively produces samples similar to the real data and a {\it discriminator} strives to distinguish between the real and fake samples. 

In theory, GANs take the form of a min-max game where the discriminator chooses $D:\R^d\to [0,1]$ to maximize its chance of differentiating real samples from fake ones, while the generator, who observes its opponent's choice $D$, selects $G:\R^n\to\R^d$ (for a fixed $n\le d$) to neutralize the effect of $D$. In \cite{Goodfellow14}, this minimax problem is shown equivalent to \eqref{unsupervised} with $d(\cdot,\cdot)=\JSD(\cdot,\cdot)$, the Jensen-Shannon divergence; further, the generator's optimal strategy $G^*:\R^n\to\R^d$ exists and it reveals the data distribution $\rhod$. 

What is {\it not} answered by the theory of GANs is whether (and how) an initial guess $\rho_0\in\cP(\R^d)$ (or, an initial $G:\R^n\to\R^d$) can evolve gradually towards $\rhod\in\cP(\R^d)$ (or, $G^*:\R^n\to\R^d$). The GAN algorithm \cite[Algorithm 1]{Goodfellow14} serves to find this ``path to $\rhod$'' numerically: by using two artificial neural networks to represent the functions $G$ and $D$, it updates their parameters repetitively, through a recursive optimization performed sequentially by the two players. While the GAN algorithm has brought significant improvements to artificial intelligence applications (see e.g., \cite{Denton15}, 
\cite{Vondrick16}, \cite{Ledig17}, \cite{Wiese20}), 
it is {\it not} as stable 
as we wish---while a ``path'' can always be computed, it often diverges 
away from $\rhod$. 

Attempts to improve the stability of GANs are numerous, such as modifying training procedures (e.g., \cite{Salimans16}, \cite{Heusel17}, \cite{AB17}), adding regularizing terms (e.g., \cite{Gulrajani17}, \cite{Miyato18}), or using different metrics $d(\cdot,\cdot)$ (e.g., 
Wasserstein distance in \cite{WGAN}, maximum mean discrepancy in \cite{MMDGAN2} and \cite{Binkowski18}, characteristic function distance in \cite{Ansari20}). 
All the studies approached the stability issue at the {\it algorithmic} level, i.e., by inspecting the details of the recursive optimization procedures. 


In this paper, we give a complete characterization of the ``path to $\rhod$'' at the {\it theoretic} level.  In a nutshell, this ``path'' will be characterized as a {\it gradient flow} in the space $\cP(\R^d)$, whose evolution is governed by a {\it distribution-dependent} ordinary differential equation (ODE); see Theorem~\ref{thm:main} and Proposition~\ref{prop:equivalence}, the main results of this paper. This theoretic characterization has important practical implications. 

First, it sheds new light on the divergence of GANs. By the gradient-flow identification of GANs, we are able to decompose the GAN algorithm and discover an unapparent fact: the algorithm implicitly minimizes the mean squared error (MSE) between two sets of samples in $\R^d$. While the MSE fitting demands {\it point-wise} similarity (i.e., the $i^{th}$ sample in one set is similar to the $i^{th}$ sample in the other set), what we actually need is only {\it set-wise} similarity (i.e., the distribution on $\R^d$ of the samples in one set is similar to that of the samples in the other set). As it imposes too stringent a criterion, the MSE fitting may keep  producing inferior $G:\R^n\to\R^d$, such that $\rhod$ is never approached. 
This observation suggests a new potential route to alleviate instability: replacing the MSE fitting in the GAN algorithm by a measure of set-wise similarity; see Section~\ref{subsec:mode collapse} for detailed explanations.


Our main results, in addition, yield a new interpretation for GANs: the {\it non-cooperative} game between the generator and discriminator can be equivalently viewed as a {\it cooperative} game between a {\it navigator} and a {\it calibrator}. The {navigator} aims to navigate across the space $\cP(\R^d)$ following the aforementioned distribution-dependent ODE, and the {calibrator} serves to ``calibrate'' the ODE's dynamics. 
To the best of our knowledge, a possible ``cooperative'' view of GANs was only briefly touched on in \cite{Goodfellow16}. We materialize this idea with great specifics: the two players 
collaborate precisely to simulate an ODE that ensures smooth sailing to $\rhod$; 
see the discussion below Proposition~\ref{prop:equivalence} for details. \\
\vspace{-0.1in}

\noindent
{\bf Our Methodology.} As in \cite{Goodfellow14}, we take $d(\cdot,\cdot) = \JSD(\cdot,\cdot)$ in \eqref{unsupervised}. Our first observation is that $J(\rho):=\JSD(\rho,\rhod)$ is {\it strictly convex} on $\cP(\R^d)$, which suggests that \eqref{unsupervised} can potentially be solved by (deterministic) gradient descent in $\cP(\R^d)$. By taking the ``gradient of $J$ at $\rho\in\cP(\R^d)$'' to be $\nabla \frac{\delta J}{\delta \rho}(\rho,\cdot):\R^d\to\R^d$, where $\frac{\delta J}{\delta \rho}:\cP(\R^d)\times\R^d\to\R$ is the linear functional derivative of $J$ (Definition~\ref{def:delta G}) and $\nabla$ is the usual gradient operator in $\R^d$, we derive the gradient-descent ODE for \eqref{unsupervised}, which is given by
\begin{equation}\label{Y GAN}
dY_t 
=-\frac12 \left(\frac{\nabla\rho^{Y_t}(Y_t)}{\rho^{Y_t}(Y_t)} - \frac{\nabla\rhod(Y_t)+\nabla\rho^{Y_t}(Y_t)}{\rhod(Y_t)+\rho^{Y_t}(Y_t)}\right)dt,\quad  \rho^{Y_0}=\rho_0 \in\cP(\R^d).
\end{equation}
This ODE is distribution-dependent in a {distinctive} way. Unlike a typical McKean-Vlasov stochastic differential equation (SDE) that depends explicitly on $\mathcal L(Y_t)$, the law of the state process $Y$ at time $t$, \eqref{Y GAN} depends on the density $\rho^{Y_t}\in\cP(\R^d)$ of $Y_t$ and its gradient $\nabla \rho^{Y_t}$. The classical framework of interacting particle systems, crucial for the existence of solutions to McKean-Vlasov SDEs, does not easily accommodate this kind of dependence. 


In view of this, we take an unconventional approach motivated by \cite{BR20}:  by finding a weak solution $u(t,y)$ to the nonlinear Fokker-Planck equation associated with ODE \eqref{Y GAN}, one can in turn construct a solution $Y$ to \eqref{Y GAN}, made specifically from $u(t,y)$. 

To solve the nonlinear Fokker-Planck equation (i.e., \eqref{eq:FP} below), we transform it into a nonlinear Cauchy problem (i.e., \eqref{eq:v} below) in an appropriate Banach space. By showing that the involved operator is ``accretive'' (Lemma~\ref{lem:accretive}), we are able to find a weak solution on the strength of Crandall-Liggett's theorem for partial differential equations (PDEs) in Banach spaces; see Proposition~\ref{prop:CL thm}. The uniqueness of solutions is also established by a delicate generalization of \cite{BC79}, which leads to the characterization of a unique weak solution $u(t,y)$ in Theorem~\ref{thm:FPE}. 

By substituting $u(t,\cdot)$ for the density $\rho^{Y_t}(\cdot)$ in \eqref{Y GAN}, we obtain a standard ODE {\it without} distribution dependence (i.e., \eqref{Y GAN'} below). To construct a solution to \eqref{Y GAN}, we aim at finding a process $Y$ such that (i) $t\mapsto Y_t$ satisfies the above standard ODE, and (ii) the density of $Y_t$ exists and coincides with $u(t,\cdot)$, i.e., $\rho^{Y_t}(\cdot)=u(t,\cdot)\in\cP(\R^d)$, for all $t\ge 0$. Thanks to the superposition principle of \cite{Trevisan16}, we can choose a probability measure $\P$ on the canonical path space, so that the canonical process $Y_t(\omega):=\omega(t)$ fulfills (i) and (ii) above. This immediately gives a solution to \eqref{Y GAN} (Proposition~\ref{prop:solution Y}) and it is in fact unique up to time-marginal distributions under suitable regularity conditions (Proposition~\ref{prop:uniqueness Y}). 


Let us stress that while \eqref{Y GAN} takes the form of an ODE, it is nontrivial to find a suitable notion of solutions, due to two  kinds of randomness intertwined at time 0; see the second paragraph in Section~\ref{sec:ODE solution}. Ultimately, we define a solution to \eqref{Y GAN} using a random selection of deterministic paths, represented by a probability measure $\P$ on the canonical path space (Definition~\ref{def:solution Y}). This formulation perfectly fits Trevisan's superposition principle, and is reminiscent of L.C. Young's generalized curves in the deterministic theory and weak solutions to SDEs (despite our restriction to the path space with only $\P$ varying). 

With the unique solution $Y$ to ODE \eqref{Y GAN} characterized, it is time to check if our ``gradient descent'' idea for \eqref{unsupervised} works. Two key findings stem from a detailed analysis of the transformed Fokker-Planck equation (i.e., \eqref{eq:v} below). First, along the path of $Y$, the map $t\mapsto J(\rho^{Y_t}) = \JSD(\rho^{Y_t},\rhod)$ is nonincreasing (Proposition~\ref{prop:dJ/dt}), i.e., $\rho^{Y_t}$ moves closer to $\rhod$ continuously over time. Second, 
$\rho^{Y_{t_n}}\to \rhod$ in $L^1(\R^d)$, at least along a sequence $\{t_n\}$ in time with $t_n\uparrow\infty$ (Corollary~\ref{coro:v to 1}).  Putting these together gives the ultimate convergence result: $\rho^{Y_{t}}\to \rhod$ in $L^1(\R^d)$ as $t\to\infty$; see Theorem~\ref{thm:main}. That is to say, we can uncover $\rhod$ along the {\it gradient flow} $\{\rho^{Y_t}\}_{t\ge 0}$ in $\cP(\R^d)$, specified by the gradient-descent ODE \eqref{Y GAN}. 

Algorithm~\ref{alg:ODE} is designed to simulate ODE \eqref{Y GAN}. Intriguingly, it is shown equivalent to the GAN algorithm (Proposition~\ref{prop:equivalence}), which yields the ``cooperative'' view of GANs mentioned above: the generator and discriminator now take the roles of a {navigator} and {calibrator}, respectively, working closely to simulate ODE \eqref{Y GAN}; 
see the exposition in Section~\ref{subsec:simulation}. \\
\vspace{-0.1in}

\noindent
{\bf Relations to the Literature.}
In this paper, we work with the general space $\cP(\R^d)$ under the total variation distance. This notably differs from the classical setting---$\cP_2(\R^d)$ under the second-order Wasserstein distance $\cW_2$, where $\cP_2(\R^d)$ is the set of elements in $\cP(\R^d)$ with finite second moments. We take up the non-standard setup for two main reasons. 
First, working with $\cP_2(\R^d)$ implicitly assumes $\rhod\in\cP_2(\R^d)$, while $\rhod$, as the underlying data distribution, generally lies in $\cP(\R^d)$. Second, using the total variation distance ensures the ``right'' kind of convergence. 
As we chose $\JSD(\cdot,\cdot)$ in the first place to measure the distance between densities, such a distance function should be kept consistently throughout our analysis. 
As shown in Remark~\ref{rem:equivalence}, convergence in $\cP(\R^d)$ under $\JSD(\cdot,\cdot)$ is equivalent to that under total variation (or, under $L^1(\R^d)$). Hence, ``$\rho^{Y_t}\to \rhod$ in $L^1(\R^d)$'' in Theorem~\ref{thm:main} readily implies ``$\JSD(\rho^{Y_t},\rhod)\to 0$''. That is, convergence to $\rhod\in\cP(\R^d)$ is not only achieved, but achieved {\it specifically} under the desired  distance function originally chosen. 

Working with $\cP(\R^d)$ (under total variation) deprives us of the full-fledged theory of $\cP_2(\R^d)$ (under $\cW_2$), developed in e.g., \cite{Ambrosio-book-05}, \cite{CDLL15}, \cite{CD-book-18-I}, and \cite{DLR19}. First, it is unclear how to define the ``gradient'' of $J(\rho)= \JSD(\rho,\rhod)$ in $\cP(\R^d)$, as the standard Lions and Wasserstein derivatives are only defined in $\cP_2(\R^d)$ (or in $\cP_q(\R^d)$ for $q\in[1,\infty)$, which consists of elements in $\cP(\R^d)$ with finite $q^{th}$ moments). Since $\nabla \frac{\delta J}{\delta \rho}$ admits a gradient-type property (Proposition~\ref{prop:the gradient}) and it is well-defined in $\cP(\R^d)$ without relying on the $\cP_2(\R^d)$ structure (Lemma~\ref{lem:delta J}), we take it to be the proper ``gradient'' in $\cP(\R^d)$. While $\nabla \frac{\delta J}{\delta \rho}$ is more general, it coincides with the Lions and Wasserstein derivatives when restricted to $\cP_2(\R^d)$; see Section~\ref{subsec:Lions derivative}. 
On the other hand, to compute the evolution $t\mapsto J(\rho^{Y_t})$ and show its convergence to $J(\rhod)=0$, we do not employ any It\^{o}-type formula or compactness argument (which are unique to $\cP_2(\R^d)$), but rely on the aforementioned analysis of the transformed Fokker-Planck equation (Proposition~\ref{prop:dJ/dt} and Corollary~\ref{coro:v to 1}); see Remarks~\ref{rem:no Ito's} and \ref{rem:no compactness} for details.

Gradient flows have recently been used in many implicit generative models; see e.g., 
\cite{Gao19}, \cite{Gao20}, \cite{Ansari21}, and \cite{MN21}. Their algorithms are advocated as alternatives to GANs---in particular, the minimization part of GANs is replaced by a kind of gradient descent. Contrary to the common practice ``modifying GANs to form a gradient-flow algorithm'', Proposition~\ref{prop:equivalence} shows that the GAN algorithm itself, without any modification, already computes gradient flows. What sets us apart is a subtle difference in algorithm design. While our gradient-flow algorithm (i.e., Algorithm~\ref{alg:ODE}) resembles those in \cite{Gao19} and \cite{Gao20}, we particularly coordinate the estimation of $D$ with the gradient-descent update of $G$, so that the {\it zero-sum} game setup of GANs can be recovered. The algorithms in \cite{Gao19} and \cite{Gao20}, by contrast, represent {\it non-zero-sum} games. 

In the recent development of gradient-flow algorithms mentioned above, a common (yet unnoticed) issue is the ``{\it inconsistent use of distance functions}'': one proposes to solve \eqref{unsupervised} with a specific choice of $d(\cdot,\cdot)$ (e.g., an $f$-divergence),  
but carries out theoretic proofs under a different distance function (e.g., $\cW_2$). 
This can be problematic: even if $\cW_2(\rho^{Y_t},\rhod)\to 0$ is proved, since the metric $\cW_2$ is {\it not} equivalent to an $f$-divergence, $\rho^{Y_t}$ need not converge to $\rhod$ under an $f$-divergence, the {\it desired} distance function initially chosen. 
As noted above, Theorem~\ref{thm:main} settles this issue, as the target distance function (i.e., $\JSD(\cdot,\cdot)$) is equivalent to the actual distance function in use (i.e., the total variation distance). \\

\vspace{-0.1in}

\noindent
{\bf Organization of the paper.}
Section~\ref{sec:preliminaries} formulates a proper ``gradient'' in $\cP(\R^d)$ and derives the gradient-descent ODE \eqref{Y GAN}. Section~\ref{sec:main results} presents the main results (Theorem~\ref{thm:main} and Proposition~\ref{prop:equivalence}) and discusses their implications, including the identification of GANs with gradient flows and the cause of divergence of GANs. Section~\ref{sec:FP} derives the nonlinear Fokker-Planck equation associated with ODE \eqref{Y GAN} and establishes the existence of a unique weak solution $u(t,y)$. Based on $u(t,y)$, Section~\ref{sec:ODE solution} constructs a solution $t\mapsto Y_t$ to ODE \eqref{Y GAN} and shows that it is unique up to time-marginal distributions. Section~\ref{sec:to rhod} proves Theorem~\ref{thm:main}, i.e., the density of $Y_t$ converges to $\rhod$, as $t\to\infty$. 
Appendices collect some proofs.\\ 
\vspace{-0.1in}

\noindent
{\bf Notation.}
Let $\cP(\R^d)$ be the set of all probability density functions on $\R^d$. For any $q\in[1,\infty)$, let $\cP_q(\R^d)$ be the subset of $\cP(\R^d)$ that contains density functions with finite $q^{th}$ moments. We denote by $\mud$ the probability measure on $\R^d$ induced by the underlying data distribution $\rhod\in\cP(\R^d)$. For any $\R^d$-valued random variable $Y$, we denote by $\rho^Y\in \cP(\R^d)$ the density function of $Y$ (if it exists). Consider the second-order Wasserstein distance $\cW_2(\rho_1,\rho_2) := (\inf \{\E[|X_1-X_2|^2] : \rho^{X_1}=\rho_1,\rho^{X_2}=\rho_2\})^{1/2}$, for $\rho_1,\rho_2\in\cP_2(\R^d)$. 

Given $n\in\N$, $S\subseteq \R^n$, and a measure $\nu$ on $\R^n$, we let $L^p(S, \nu)$, $p\in[1,\infty)$, be the set of $f:S\to\R$ with $\norm{f}^p_{L^p(S, \nu)} := \int_{S} |f|^p d\nu <\infty$, and $L^\infty(S, \nu)$ be the set of $f:S\to\R$ that are bounded $\nu$-a.e. Also, we denote by $W^{1,p}(S, \nu)$ the set of $f:S\to\R$ whose weak derivatives of first order exist and lie in $L^p(S, \nu)$, and by $W^{1,p}_0(S, \nu)$ the closure of $C^\infty_c(S)$ (the set of infinitely differentiable $f:S\to\R$ with compact supports) under the $W^{1,p}(S, \nu)$-norm. For $p=2$, we write $H^1(S, \nu)=W^{1,2}(S, \nu)$ and $H^1_0(S, \nu)=W^{1,2}_0(S, \nu)$.
When $\nu$ is the Lebesgue measure (denoted by $\Leb$), we shorten the notation to $L^p(S)$, $W^{1,p}(S)$, and $H^1(S)$, etc. We denote by $\nabla$ the Euclidean gradient operator with respect to (w.r.t.) $y\in\R^n$. 

Let $\cX$ be an arbitrary collection of $f:S\to\R$. 
For any $u:[0,\infty)\to\cX$, we will often abuse the notation by writing $u(t,\cdot)= (u(t))(\cdot)\in\cX$ for all $t\ge 0$ and treat $u(t,y) = (u(t))(y)$ as a generic function on $[0,\infty)\times S$.


\section{Preliminaries}\label{sec:preliminaries}
We study \eqref{unsupervised} with $d(\cdot,\cdot)$ therein taken to be the {\it Jensen-Shannon divergence}, i.e., 
\begin{align}\label{JSD}
\JSD(\rho,\bar\rho) &:= \frac12 \KL\left(\bar\rho\ \middle\|\frac{\bar\rho+\rho}{2}\right)+ \frac12 \KL\left(\rho\ \middle\|\frac{\bar\rho+\rho}{2}\right),\quad \forall \rho,\bar\rho\in \cP(\R^d),
\end{align} 
where $\KL$ denotes the {\it Kullback-Leibler divergence}, defined by
\begin{equation}\label{D_KL}
\KL(\rho\|\bar\rho):= \int_{\R^d}\rho(x)\ln\left(\frac{\rho(x)}{\bar\rho(x)}\right) dx,\quad \forall \rho,\bar\rho\in \cP(\R^d). 
\end{equation}
As shown in \cite{ES03} and Theorem 1 in \cite{OV03}, $\sqrt{\JSD(\cdot,\cdot)}$ defines a metric on $\cP(\R^d)$, so that $\JSD(\cdot,\cdot)$ is a semi-metric. The problem \eqref{unsupervised} can now be stated as 
\begin{equation}\label{J}
\hbox{minimize}\quad J(\rho):= \JSD(\rho,\rhod)\quad \hbox{over}\ \cP(\R^d).
\end{equation}

\begin{remark}
\label{rem:boundedness}
For any $\rho,\bar\rho\in\cP(\R^d)$, $0\le \JSD(\rho,\bar\rho) \le \ln(2)$. This follows directly from 
\begin{align*}
\KL(\rho\|\bar\rho) &= - \int_{\R^d}\rho(x)\ln\left(\frac{\bar\rho(x)}{\rho(x)}\right) dx \ge -\int_{\R^d}\rho(x)\left(\frac{\bar\rho(x)}{\rho(x)}-1\right) dx=0,\\
\KL\left(\rho\ \middle\|\frac{\rho+\bar\rho}{2}\right) &= \int_{\R^d}\rho(x)\ln\left(\frac{2\rho(x)}{\rho(x)+\bar\rho(x)}\right) dx\le \int_{\R^d}\rho(x)\ln\left(\frac{2\rho(x)}{\rho(x)}\right) dx = \ln(2),
\end{align*}
where we used the fact $\ln(x)\le x-1$ $\forall x>0$ in the first inequality. 
\end{remark}

\begin{remark}
\label{rem:equivalence}
The metric $\sqrt{\JSD(\cdot,\cdot)}$ on $\cP(\R^d)$ is equivalent to the total variation distance
\begin{align}\label{TV}
\TV(\rho,\bar\rho) &:= \sup_{A\in\B(\R^d)} \left|\int_{A} \rho(x) dx - \int_{A} \bar\rho(x) dx \right|\quad \forall \rho,\bar\rho\in\cP(\R^d).
\end{align}
This follows from $\TV(\rho,\bar\rho)= \frac12 \|\rho-\bar\rho\|_{L^1(\R^d)}$ (Lemma~\ref{lem:equivalence}) and the estimate
\begin{equation}\label{equivalence}
\phi\left(\frac12 \|\rho-\bar\rho\|_{L^1(\R^d)}\right) \le2 \JSD(\rho,\bar\rho) \le {\ln(2)} \|\rho-\bar\rho\|_{L^1(\R^d)},\quad\forall \rho, \bar\rho\in\cP(\R^d), 
\end{equation}
where $\phi:[0,1]\to\R_+$ is strictly increasing with $\phi(0)=0$; see Theorem 2 (with $\beta=1$ therein) in \cite{OV03}. 
\end{remark}
 

\begin{remark}\label{lem:s convex}
For any fixed $\bar\rho\in\cP(\R^d)$, $\JSD(\cdot,\bar\rho):\cP(\R^d)\to\R$ is {\it strictly convex}. Indeed, by \cite{NN14} (see Table I therein), $\JSD(\cdot,\cdot)$ can be expressed as
\begin{equation}\label{JSD f}
\JSD(\rho_1,\rho_2) = \int_{\R^d} f\left(\frac{\rho_1}{\rho_2}\right)\rho_2 dy,\quad \forall \rho_1,\rho_2\in\cP(\R^d),
\end{equation}
with $f:[0,\infty)\to \R$ given by $f(x):= \frac12\big[ (x+1)\ln(\frac{2}{x+1})+x\ln x\big]$. For any $\rho_1,\rho_2, \bar\rho\in\cP(\R^d)$ and $\lambda\in (0,1)$, from \eqref{JSD f} and the strict convexity of $f$, we quickly obtain $\JSD(\lambda \rho_1 +(1-\lambda)\rho_2,\bar\rho) < \lambda \JSD(\rho_1,\bar\rho) + (1-\lambda) \JSD(\rho_2,\bar\rho)$, i.e., $\JSD(\cdot,\bar\rho)$ is {\it strictly convex}.
\end{remark}




\subsection{Gradient Descent for Functions on $\cP(\R^d)$}\label{subsec:GD}
For a strictly 
convex $f:\R^d\to \R$, it is well-known that the global minimizer $y^*\in\R^d$ of $f$, if it exists, can be found efficiently by (deterministic) gradient descent in the space $\R^d$. Specifically, for any initial point $y\in\R^d$, the ODE
\begin{equation}\label{GD}
d Y_t = -\nabla f(Y_t) dt,\quad Y_0=y \in\R^d,
\end{equation}
converges to $y^*$ as $t\to\infty$. Given that $\JSD(\cdot,\rhod):\cP(\R^d)\to\R$ is strictly convex (Remark~\ref{lem:s convex}), it is natural to ask if its minimizer, i.e., $\rhod\in\cP(\R^d)$, can also be found by gradient descent, now in the space $\cP(\R^d)$. The ultimate question is how we should modify ODE \eqref{GD} to accommodate a function defined on $\cP(\R^d)$---particularly, how its gradient should be defined. 

Given a strictly convex $G: \cP(\R^d)\to\R$, we expect that the corresponding gradient-descent ODE takes the form
\begin{equation}\label{GD'}
d Y_t = -\partial_\rho G(\rho^{Y_t},Y_t) dt,\quad \rho^{Y_0}=\rho_0 \in\cP(\R^d),
\end{equation}
where $\partial_\rho G(\rho^{Y_t},\cdot):\R^d\to \R^d$ denotes the ``gradient of $G$ at $\rho^{Y_t}\in\cP(\R^d)$'', which remains to be defined. As $Y_0$ is specified by a density $\rho_0 \in\cP(\R^d)$ (but {\it not} as a fixed state $y\in\R^d$ in \eqref{GD}), $Y_t$ is a random variable, with density $\rho^{Y_t}\in\cP(\R^d)$, for all $t\ge 0$. That is, \eqref{GD'} is now a {\it distribution-dependent} ODE. At any time $t\ge 0$, given the current density $\rho^{Y_t}\in\cP(\R^d)$,  $-\partial_\rho G(\rho^{Y_t},\cdot):\R^d\to\R^d$ dictates the direction along which each $y\in\R^d$ (i.e., a realization of $Y_t$) moves forward. Specifically, if we discretize ODE \eqref{GD'} with a fixed time step $\eps>0$, a realization $y\in\R^d$ of $Y_t$ is transported to 
\begin{equation}\label{y to bar y}
\bar y: = y-\eps\partial_\rho G(\rho^{Y_t},y)\in\R^d,
\end{equation}
which represents a realization of $Y_{t+\eps}$. It is highly nontrivial, however, to define the ``gradient'' $\partial_\rho G(\rho^{Y_t},\cdot)$. As the domain $\cP(\R^d)$ of $G$ is not even a vector space, differentiation cannot be easily defined in the usual Fr\'{e}chet or Gateaux sense. 

A natural idea is to rely on the convexity of $\cP(\R^d)$: for any $\rho,\bar\rho\in\cP(\R^d)$, we can move from $\rho$ to $\bar \rho$ along the line segment $\{(1-\lambda)\rho+\lambda \bar\rho :\lambda\in[0,1]\}$ in $\cP(\R^d)$ and study how $G$ changes along the way. This leads to the notion of a {\it linear functional derivative}.  

\begin{definition}\label{def:delta G} 
A linear functional derivative of $G:\cP(\R^d)\to \R$ is a function $\frac{\delta G}{\delta \rho}:\cP(\R^d)\times\R^d\to\R$ that satisfies
\begin{equation}\label{delta G}
G(\bar\rho)-G(\rho)= \int_0^1 \int_{\R^d} \frac{\delta G}{\delta \rho}\big((1-\lambda)\rho+\lambda \bar\rho,y\big) (\bar\rho-\rho)(y)dy d\lambda,\quad \forall \rho,\bar\rho\in\cP(\R^d). 
\end{equation}
\end{definition} 
Under appropriate growth and continuity conditions of $\frac{\delta G}{\delta \rho}$, \eqref{delta G} implies
\[
\lim_{\eps\to 0} \frac{G(\rho+\eps(\bar \rho-\rho)) - G(\rho)}{\eps} = \int_{\R^d} \frac{\delta G}{\delta \rho}\big(\rho,y\big) (\bar\rho-\rho)(y)dy.
\]
That is, $\frac{\delta G}{\delta \rho}(\rho,\cdot)$ can be viewed as the ``gradient of $G$'' if one moves along straight lines in $\cP(\R^d)$ from one density to another. 

The evolution of densities in ODE \eqref{GD'}, nonetheless, is much more involved. At any time $t\ge 0$, $-\partial_\rho G(\rho^{Y_t},\cdot):\R^d\to\R^d$ serves as a vector field that moves any current point $y\in\R^d$ to a new location. Put differently, the transportation in \eqref{y to bar y} is of the general form 
\begin{equation}\label{y to bar y'}
\bar y := y+\eps\xi(y) = (I+\eps\xi)(y)\in\R^d,\quad \hbox{for some $\xi:\R^d\to\R^d$}, 
\end{equation}
where $I$ denotes the identity map on $\R^d$. Under \eqref{y to bar y'}, an initial probability measure $\mu$ on $\R^d$ is transformed into the pushforward measure
\begin{equation}\label{pushforward}
\mu^\xi_\eps(\cdot):= \mu\left((I+\eps\xi)^{-1}(\cdot)\right). 
\end{equation}
As the next result shows, when $\mu$ is induced by a density $\rho\in\cP(\R^d)$, $\mu^{\xi}_\eps$ also has a density $\rho^\xi_\eps\in\cP(\R^d)$; moreover, when one moves (nonlinearly) in $\cP(\R^d)$ from $\rho$ to $\rho^\xi_\eps$, the appropriate ``gradient of $G$ at $\rho\in\cP(\R^d)$'' turns out to be a simple functional of $\frac{\delta G}{\delta \rho}(\rho,\cdot)$. 

\begin{proposition}\label{prop:the gradient}
Let $G:\cP(\R^d)\to\R$ has a linear functional derivative $\frac{\delta G}{\delta \rho}:\cP(\R^d)\times\R^d\to\R$. Consider $\rho\in\cP(\R^d)$ and the measure $\mu(\cdot) := \int_\cdot \rho(y) dy$ on $\R^d$. If $\rho$ is of $C^1(\R^d)$, 
for any $\xi\in C^2_c(\R^d;\R^d)$, the measure $\mu^\xi_\eps$ in \eqref{pushforward} has a density $\rho^\xi_\eps\in\cP(\R^d)$, for $\eps>0$  small. Moreover, if 
$\frac{\delta G}{\delta \rho}(\rho,\cdot)$ is locally integrable on $\R^d$ and 
$\sup_{\lambda\in [0,1]} |\frac{\delta G}{\delta \rho}((1-\lambda)\rho+\lambda\bar\rho,\cdot)-\frac{\delta G}{\delta \rho}(\rho,\cdot)|\rightarrow 0$ locally uniformly as $\bar\rho \rightarrow \rho$ uniformly,
then  
\begin{equation}\label{the gradient}
\lim_{\eps\to 0} \frac{G(\rho^\xi_\eps) - G(\rho)}{\eps} = \int_{\R^d} \nabla \frac{\delta G}{\delta \rho}\big(\rho,y\big)\cdot \xi(y)d\mu(y).
\end{equation}
\end{proposition}

Proposition~\ref{prop:the gradient}, whose proof is relegated to Section~\ref{subsec:proof of prop:the gradient}, conveys important messages. First, when every $y\in\R^d$ moves along the direction $\xi(y)\in\R^d$ (as in \eqref{y to bar y'}), the original density $\rho\in\cP(\R^d)$ is recast into $\rho^\xi_\eps\in\cP(\R^d)$. Furthermore, as \eqref{the gradient} demonstrates, $\nabla \frac{\delta G}{\delta \rho}(\rho,y)$ specifies how moving along $\xi(y)$ changes the value of $G$. That is, $\nabla \frac{\delta G}{\delta \rho}(\rho,\cdot):\R^d\to\R^d$ is the proper ``gradient of $G$ at $\rho\in\cP(\R^d)$'' for the transportation in \eqref{y to bar y'}, which is the discretization of an ODE like \eqref{GD'}. As a result, from now on we will write \eqref{GD'} as 
\begin{equation}\label{GD''}
d Y_t = -\nabla \frac{\delta G}{\delta \rho}(\rho^{Y_t},Y_t) dt,\quad \rho^{Y_0}=\rho_0 \in\cP(\R^d). 
\end{equation}
 

\subsection{Connection to the Lions and Wasserstein Derivatives}\label{subsec:Lions derivative}
If we restrict ourselves to the subset $\cP_2(\R^d)$ of $\cP(\R^d)$, our gradient $\nabla \frac{\delta G}{\delta \rho}$ in fact coincides with the Lions and Wasserstein derivatives in the literature under suitable conditions. Let us briefly recall the definitions of these two kinds of derivatives in $\cP_2(\R^d)$.

In a probability space $(\Omega,\F,\P)$ where $\P$ is an atomless measure, every $\R^d$-valued random variable $Y\in L^2(\Omega,\F,\P)$ has a density $\rho^Y\in \cP_2(\R^d)$. Hence, we can associate to each $G:\cP_2(\R^d)\to\R$ a lifted function $g: L^2(\Omega,\F,\P)\to \R$ defined by $g(Y):= G(\rho^Y)$. 
As Fr\'{e}chet differentiation is well-defined in the Banach space $L^2(\Omega,\F,\P)$, the ``derivative of $G$ at $\rho_0\in\cP_2(\R^d)$'' can be defined as $Dg(Y_0)$, the Fr\'{e}chet derivative of $g$ at $Y_0\in L^2(\Omega,\F,\P)$ with $\rho^{Y_0}=\rho_0$. By identifying $L^2(\Omega,\F,\P)$ with its dual (as it is reflexive), we have $Dg(Y_0)\in L^2(\Omega,\F,\P)$. Thanks to Proposition 5.25 in \cite{CD-book-18-I}, this map $Dg:L^2(\Omega,\F,\P)\to L^2(\Omega,\F,\P)$ can be identified with a Borel function $\xi_{\rho_0}:\R^d\to \R^d$, i.e., $Dg(Y) = \xi_{\rho_0}(Y)$ a.s.\ for all $Y\in L^2(\Omega,\F,\P)$ with $\rho^{Y} = \rho_0$. This Borel function $\xi_{\rho_0}$ is then called the {\it Lions derivative} of $G:\cP_2(\R^d)\to \R$ at $\rho_0\in\cP_2(\R^d)$ \cite[Definition 5.22]{CD-book-18-I}, and we will denote it by $\partial^L_\rho G(\rho_0,\cdot):= \xi_{\rho_0}(\cdot)$. 

On the other hand, when $\cP_2(\R^d)$ is equipped with the Wasserstein distance $\cW_2$, we can borrow the idea from convex analysis to define the subdifferential (and superdifferential) of $G:\cP_2(\R^d)\to\R$ at each $\rho\in\cP_2(\R^d)$; see e.g., Definition 5.62 in \cite{CD-book-18-I} and Definition 10.1.1 in \cite{Ambrosio-book-05}. The {\it Wasserstein derivative} of $G$ at $\rho\in\cP_2(\R)$, denoted by $\partial^W_{\rho}G(\rho,\cdot)$, is then defined as the unique element (if it exists) that resides in both the sub- and superdifferential of $G$ at $\rho\in\cP_2(\R^d)$. 

Now, consider $G:\cP_2(\R^d)\to\R$ such that $\nabla \frac{\delta G}{\delta \rho}(\rho,\cdot)$ is well-defined for all $\rho\in\cP_2(\R^d)$. By Proposition 5.48 and Theorem 5.64 in \cite{CD-book-18-I}, under additional continuity and growth conditions on $\nabla \frac{\delta G}{\delta \rho}(\rho,\cdot)$, if the Lions derivative $\partial^L_\rho G(\rho,\cdot)$ exists and is continuous, then all three kinds of derivatives in discussion coincide, i.e., 
\begin{equation}\label{derivatives same}
\partial^L_\rho G(\rho,\cdot) =  \partial^W_\rho G(\rho,\cdot) = \nabla \frac{\delta G}{\delta \rho}(\rho,\cdot),\quad \forall \rho\in\cP_2(\R^d).
\end{equation}
Despite this, $\nabla \frac{\delta G}{\delta \rho}(\rho,\cdot)$ is more general in that it can be defined on $\cP(\R^d)$, without the need of the $\cP_2(\R^d)$ structure. It thus suits our study particularly; see the last two paragraphs in Section~\ref{subsec:the problem}. 




\subsection{Problem Formulation}\label{subsec:the problem}
We aim at solving the problem \eqref{J} through ODE \eqref{GD''}, with $G(\cdot)$ therein taken to be $J(\cdot) = \JSD(\cdot,\rhod)$. The first task is to find the linear functional derivative $\frac{\delta J}{\delta \rho}$, which is characterized explicitly in the next result (whose proof is relegated to Appendix~\ref{subsec:proof of lem:delta J}).

\begin{lemma}\label{lem:delta J}
$J:\cP(\R^d)\to\R_+$ defined in \eqref{J} has a linear functional derivative given by
\begin{equation}\label{delta J}
\frac{\delta J}{\delta \rho} (\rho,y) = \frac12 \ln\left(\frac{2\rho(y)}{\rhod(y)+\rho(y)}\right),\quad \forall \rho\in\cP(\R^d)\ \hbox{and}\  y\in\R^d. 
\end{equation}
\end{lemma}

Thanks to \eqref{delta J}, once we replace $G$ by $J$ in ODE \eqref{GD''}, we obtain ODE \eqref{Y GAN}. Our goal is show that this distribution-dependent ODE has a (unique) solution $Y$ and the induced density flow $\{\rho^{Y_t}\}_{t\ge 0}$ converges in $\cP(\R^d)$ to the data distribution $\rhod$.  

Let us stress that we will work with the general space $\cP(\R^d)$ (under the total variation distance), instead of its subset $\cP_2(\R^d)$ (under the $\cW_2$ distance), although the latter is more tractable with a comprehensive theory well-developed (see e.g., \cite{Ambrosio-book-05}). Such a non-standard setup is essential to our study, for two reasons. 

First, working with $\cP_2(\R^d)$ implicitly assumes $\rhod\in\cP_2(\R^d)$, while $\rhod$, as the underlying data distribution, generally lies in $\cP(\R^d)$. 
Second, using the total variation distance ensures that our study is consistent inherently. Since we chose in the first place to measure the distance between densities by $\JSD(\cdot,\cdot)$ in \eqref{JSD}, such a distance function should be kept consistently throughout our analysis. Particularly, the desired convergence ``$\rho^{Y_t}\to\rhod$'' should be established as $\JSD(\rho^{Y_t},\rhod)\to 0$. As convergence in $\JSD$ is equivalent to that in total variation (Remark~\ref{rem:equivalence}), using the total variation distance is a valid, consistent choice. 


\section{Main Results and Contributions}\label{sec:main results}
We will impose the following standing assumption on $\rhod\in\cP(\R^d)$.

\begin{assumption}\label{asm}
$\rhod>0$ on $\R^d$, $\ln \rhod \in L^\infty_{\operatorname{loc}}(\bR^d)\cap H^1(\bR^d, \mud)$. 
\end{assumption}

The local boundedness of $\ln \rhod$ ensures that the weighted Sobolev space $W^{1,p}(\R^d, \mud)$ 
is a Banach space for all $p\ge 1$.  Also, $\ln \rhod \in  H^1(\bR^d, \mud)$ entails $\int_{\R^d}\frac{|\nabla \rhod|^2}{\rhod} dy <\infty$, i.e., the Fisher information of $\mud$ (relative to the Lebesgue measure) is finite. 

Let us present below the main theoretic result of this paper. 
 
\begin{theorem}\label{thm:main}
There exists a solution  $Y$ to ODE \eqref{Y GAN} such that $\eta(t):=\rho^{Y_t}\in\cP(\R^d)$ fulfills 
\begin{equation}\label{eta condition}
\eta \in C([0,\infty); L^1(\R^d)),\quad \eta/\rhod\in L^\infty_+([0,\infty)\times \R^d),\quad\hbox{and}\quad \nabla\eta \in L^1_{\operatorname{loc}}([0,\infty)\times \R^d). 
\end{equation}
Moreover, for any such a solution $Y$, 
$\rho^{Y_{t}}\to \rhod$ in $L^1(\R^d)$ as $t\to\infty$. 
\end{theorem}
Theorem~\ref{thm:main} constructs an ``amenable'' solution $Y$ to the gradient-descent ODE \eqref{Y GAN}, such that $\rho^{Y_t}$ evolves continuously in $t$, bounded by a multiple of $\rhod$, and its (weak) derivative is locally integrable. Moreover, its density flow $\{\rho^{Y_t}\}_{t\ge 0}$ leads us to the data distribution $\rhod$. This suggests that one can uncover $\rhod$ by simulating ODE \eqref{Y GAN} and a corresponding algorithm is developed in Section~\ref{subsec:simulation} below. The proof of Theorem~\ref{thm:main}, based on all the developments in Sections~\ref{sec:FP}, \ref{sec:ODE solution}, and \ref{sec:to rhod}, is relegated to the end of Section~\ref{sec:to rhod}. 

It is worth noting that gradient flows have recently been used in many implicit generative models 
(see \cite{Gao19}, \cite{Gao20}, \cite{Ansari21}, and \cite{MN21}, among others). Several mathematical issues, however, persist in this line of research. A common, and perhaps most critical, one is the ``{\it inconsistent use of distance functions}'': one proposes to minimize the distance (e.g., an $f$-divergence) between the current density $\rho_t$ and $\rhod$ along a gradient flow, but carries out theoretic proofs (partially or entirely) under a different distance function (e.g., $\cW_2$). 

As the choice of a distance function may not be arbitrary but dependent on the actual dataset (see \cite{MW20}), this can be problematic:  even if $\cW_2(\rho_t,\rhod)\to 0$ is proved, simply because the metric $\cW_2$ is {\it not} equivalent to an $f$-divergence, $\rho_t$ need not get close to $\rhod$ under an $f$-divergence. That is, convergence under the {\it desired} distance function (i.e., the initially chosen one) remains in question. 
Also, the use of $\cW_2$ readily requires $\rhod$ to lie in $\cP_2(\R^d)$, where $\cW_2$ is well-defined, while $\rhod$ generally belongs to $\cP(\R^d)$.  

Theorem~\ref{thm:main}, remarkably, overcomes all the mathematical issues. By taking the {\it target} distance function to be the Jensen-Shannon divergence (i.e., $\JSD$ in \eqref{JSD}, a specific $f$-divergence), we work directly with the general space $\cP(\R^d)$. The {\it actual} distance function in use is the total variation distance (i.e., the $L^1$ distance), which is equivalent to $\JSD$. Hence, ``$\rho^{Y_t}\to \rhod$ in $L^1(\R^d)$'' established in Theorem~\ref{thm:main} readily implies ``$\JSD(\rho^{Y_t},\rhod)\to 0$'' along the same gradient flow. That is, convergence to $\rhod\in\cP(\R^d)$ is not only achieved, but achieved under the originally chosen target distance function. 

This does not come at no cost. Working generally in $\cP(\R^d)$ (under total variation) deprives us of the full-fledged gradient-flow theory in $\cP_2(\R^d)$ (under $\cW_2$), as elaborated in \cite{Ambrosio-book-05}. 
The first major challenge is the existence of a solution to ODE \eqref{Y GAN}. While we can view \eqref{Y GAN} as a degenerate McKean-Vlasov SDE without the diffusion term,  the involved distribution-dependence 
is unusual. A McKean-Vlasov SDE typically involves $\mathcal L(Y_t)$, the law of the state process $Y$ at time $t$. In general, an interacting particle system can be constructed so that the particles' empirical measure approximates $\mathcal L(Y_t)$, which leads to a solution to the SDE in the limit. This classical framework does not easily accommodate other forms of distribution-dependence. Specifically, \eqref{Y GAN} depends {\it not} on $\mathcal L(Y_t)$ explicitly, but rather on the Radon-Nikodym derivative of $\mathcal L(Y_t)$ w.r.t.\ the Lebesgue measure (i.e., $\rho^{Y_t}\in\cP(\R^d)$) and its Euclidean derivative (i.e., $\nabla\rho^{Y_t}$). 
To overcome this, we first derive the nonlinear {Fokker-Planck equation} associated with \eqref{Y GAN} and show that it has a unique weak solution $u$, by the Crandall-Liggett theorem for differential equations in Banach spaces (Section~\ref{sec:FP}). Next, from this function $u$, we build a solution $\{Y_t\}_{t\ge 0}$ to \eqref{Y GAN} that satisfies \eqref{eta condition}, using the superposition principle for SDEs (Section~\ref{sec:ODE solution}). 

Another hurdle to Theorem~\ref{thm:main} is the lack of suitable stochastic calculus. In $\cP(\R^d)$ (under total variation), the dynamics of $t\mapsto \JSD(\rho^{Y_t},\rhod)$ does not follow from the standard It\^{o} formula for a flow of measures, which is established only in $\cP_2(\R^d)$ (under $\cW_2$). We instead carry out a detailed analysis of the nonlinear Fokker-Planck equation, which reveals how $\rho^{Y_{t}}$ evolves in $\cP(\R^d)$ (Section~\ref{sec:to rhod}). 


\subsection{Simulating ODE \eqref{Y GAN}}\label{subsec:simulation}
To simulate ODE \eqref{Y GAN}, one quickly realizes the challenge that its dynamics depends on $\rhod$, which is unknown and simply what we intend to uncover. To circumvent this, we choose {\it not} to estimate $\rho^{Y_t}\in\cP(\R^d)$ in \eqref{Y GAN} explicitly, {\it but} to find a function $G_t:\R^n\to\R^d$ (for a fixed $n\le d$) such that $G_t(Z)$ approximates $Y_t$, where $Z$ is a fixed $\R^n$-valued random variable (independent of $t$) that can be easily simulated (e.g., the multivariate Gaussian). By substituting $\rho^{G_t(Z)}$ for $\rho^{Y_t}$ in \eqref{Y GAN}, a direct calculation shows that \eqref{Y GAN} takes the form
\begin{equation}\label{Y GAN D_t}
dY_t 
= \frac{\nabla D_t(Y_t)}{2(1-D_t(Y_t))} dt,\quad{where}\quad D_t(y) := \frac{\rhod(y)}{\rhod(y)+\rho^{G_t(Z)}(y)}\ \ \forall y\in\R^d.
\end{equation}
Remarkably, when $G_t$ is given,  one can estimate $D_t$ {\it without} the knowledge of $\rhod$ or $\rho^{G_t(Z)}$. Indeed, by Proposition 1 in \cite{Goodfellow14}, $D_t$ is the unique maximizer, among all $D:\R^d\to [0,1]$, of $\E_{y\sim\rhod}[\ln D(y)] + \E_{z\sim\rho^Z} \left[  \ln\left(1-D\left(G_t(z)\right)\right)  \right]$.  
Therefore, it can be numerically approximated through data sampling. 

In Algorithm~\ref{alg:ODE} below, we take $G:\R^n\to \R^d$ and $D:\R^d\to [0,1]$ to be artificial neural networks and update their parameters (denoted by $\theta_G$ and $\theta_D$) recursively; namely, $(G,D)$ represents the time-varying $(G_t, D_t)$ introduced above. Specifically, (i) we follow the first half of Algorithm 1 in \cite{Goodfellow14} to estimate $D$ with $G$ given (Lines 2-4); (ii) plug $D$ into \eqref{Y GAN D_t} to simulate $Y$ over a small time step $\eps>0$, resulting in the new points $\{y^{(i)}\}_{i=1}^m$ (Lines 5-6); (iii) update $G$ by reducing the mean squared error (MSE) between $\{G(z^{(i)})\}_{i=1}^m$ and $\{y^{(i)}\}_{i=1}^m$ (Line 7). In \eqref{new points}, $\{G(z^{(i)})\}_{i=1}^m$ represents the realizations of $Y_t$ at some $t\ge 0$. Through ODE \eqref{Y GAN D_t}, $\{G(z^{(i)})\}_{i=1}^m$ are transferred to the points $\{y^{(i)}\}_{i=1}^m$, which represent the realizations of $Y_{t+\eps}$. Hence, the purpose of (iii) above is to modify $G$ so that $\{G(z^{(i)})\}_{i=1}^m$ can properly represent $Y_{t+\eps}$, instead of the previous $Y_t$. 

Algorithm~\ref{alg:ODE} is a specific, straightforward way to simulating ODE \eqref{Y GAN}. Modifications can potentially be made based on more sophisticated simulation techniques for ODEs. 

\begin{algorithm}
\caption{Simulating ODE \eqref{Y GAN}}
\begin{algorithmic}[1]
\Require $m\in\N$, $\eps>0$
\For{number of training iterations}
	\State $\bullet$ Sample minibatch of $m$ noise samples $\{z^{(1)},...,z^{(m)}\}$ from noise prior $\rho^Z$. 
	\State $\bullet$ Sample minibatch of $m$ examples $\{x^{(1)},...,x^{(m)}\}$ from the data distribution $\rhod$.
	\State $\bullet$ Update $D:\R^d\to [0,1]$ by ascending its stochastic gradient:
	\begin{equation}\label{calibrate}
	\nabla_{\theta_D} \frac1m \sum_{i=1}^m \left[\ln D(x^{(i)}) + \ln\left(1-D\left(G(z^{(i)})\right) \right)\right]. 
	\end{equation}
	\State $\bullet$ Sample minibatch of $m$ noise samples $\{z^{(1)},...,z^{(m)}\}$ from noise prior $\rho^Z$.
	\State $\bullet$ Set $Y=\{y^{(1)},...,y^{(m)}\}$ by
	\begin{equation}\label{new points}
	y^{(i)} := G(z^{(i)}) + \frac{\nabla D(G(z^{(i)}))}{2(1-D(G(z^{(i)})))} \eps,\quad \forall i=1,2,...,m.
	\end{equation}
	\State $\bullet$ Update $G:\R^n\to\R^d$ by descending its stochastic gradient:
	\begin{equation}\label{MSE}
	\nabla_{\theta_G} \frac1m\sum_{i=1}^m |G(z^{(i)})-y^{(i)}|^2.
	\end{equation}
\EndFor
\end{algorithmic}
\label{alg:ODE}
\end{algorithm}



\subsection{Connection to GANs}\label{subsec:GAN}
GANs in \cite{Goodfellow14} encode the competition between a generator and a discriminator. The generator produces fake data points by sampling $G(Z)$, where $Z$ is a fixed  random variable with a (simple) density $\rho^Z\in\cP(\R^n)$ and $G:\R^n\to\R^d$ is a (complicated) function chosen by the generator. To each data point $y\in\R^d$ (which can be real or fake), the discriminator assigns $D(y)\in [0,1]$, her subjective probability of $y$ being real, where $D:\R^d\to [0,1]$ is chosen by the discriminator. 
The resulting min-max game is given by
\begin{equation}\label{GAN}
\min_{G}\max_{D} \big\{ \E_{Y\sim\rhod}[\ln D(Y)] + \E_{Z\sim\rho^Z} \left[  \ln\left(1-D\left(G(Z)\right)\right)  \right]\big\}. 
\end{equation}
Thanks to Theorem 1 in \cite{Goodfellow14}, this minimax problem is equivalent to $\min_{G:\R^n\to\R^d} \JSD(\rho^{G(Z)},\rhod)$ and it admits a minimizer $G^*$ such that $\rho^{G^*(Z)}=\rhod$. The GAN algorithm (i.e., Algorithm 1 in \cite{Goodfellow14}) 
is proposed to find $G^*$ (and thus $\rhod$) numerically, through recursive optimization between the two players. 

Somewhat surprisingly, Algorithm~\ref{alg:ODE} above, designed to simulate ODE \eqref{Y GAN}, is equivalent to the GAN algorithm in \cite{Goodfellow14}. 

\begin{proposition}\label{prop:equivalence}
Algorithm~\ref{alg:ODE} above is equivalent to the GAN algorithm (i.e., Algorithm 1 in \cite{Goodfellow14}), up to an adjustment of the learning rates. 
\end{proposition}

\begin{proof}
As the update of $D$ in \eqref{calibrate} is the same as that in the first half of the GAN algorithm, 
it suffices to show that the update of $G$ via \eqref{new points}-\eqref{MSE} is equivalent to that in the second half of the GAN algorithm. 
By a direct calculation of \eqref{MSE}, we get
\begin{align*}
\nabla_{\theta_G} \frac1m\sum_{i=1}^m &|G(z^{(i)})-y^{(i)}|^2 = \frac{2}{m} \sum_{i=1}^m (G(z^{(i)})-y^{(i)}) \cdot \nabla_{\theta_G} G(z^{(i)})\\
&= \frac{2}{m} \sum_{i=1}^m \bigg( \frac{-\nabla D(G(z^{(i)}))}{2(1-D(G(z^{(i)})))} \eps \bigg)\cdot \nabla_{\theta_G} G(z^{(i)}) 
= \eps \nabla_{\theta_G} \frac{1}{m} \sum_{i=1}^m  \ln \big(1-D(G(z^{(i)}))\big),
\end{align*}
where the second equality follows from \eqref{new points}. The result follows by noting that the right-hand side above is precisely the update of $G$ specified in the GAN algorithm multiplied by $\eps>0$; see the last equation in \cite[Algorithm 1]{Goodfellow14}. 
\end{proof}

Proposition~\ref{prop:equivalence} yields an intriguing interpretation for GANs: the {\it non-cooperative} game between the generator and discriminator can be equivalently viewed as a {\it cooperative} game between a {\it navigator} and a {\it calibrator}. 
The navigator aims to sail across the space $\cP(\R^d)$ following ODE \eqref{Y GAN}. At any initial location in $\cP(\R^d)$ (represented by an initial $G:\R^n\to\R^d$), he finds that he cannot set off right away, as the dynamics of \eqref{Y GAN} involves the unknown $\rhod$. The calibrator comes into play to ``calibrate'' the dynamics of \eqref{Y GAN}---namely, estimate $D:\R^d\to[0,1]$ in \eqref{Y GAN D_t} using \eqref{calibrate}, which is the discriminator's optimization in the GAN algorithm. Once the calibration is done, the navigator travels a short distance in $\cP(\R^d)$ following \eqref{Y GAN}, updates  $G:\R^n\to\R^d$ to reflect his new location, and asks the calibrator to re-calibrate. This corresponds to \eqref{new points}-\eqref{MSE}, equivalent to the generator's optimization in the GAN algorithm (by Proposition~\ref{prop:equivalence}). 
To the best of our knowledge, a possible ``cooperative'' view of GANs was only briefly touched on, without details, at the end of p.\ 21 in \cite{Goodfellow16}. Proposition~\ref{prop:equivalence} materializes this idea with great specifics: the two players cooperate precisely to simulate the gradient-descent ODE \eqref{Y GAN}. 

Proposition~\ref{prop:equivalence} also contrasts with a common practice in implicit generative modeling. Many algorithms have replaced the minimization part of a given GAN algorithm 
by a kind of gradient descent. For example, \cite{Gao19}, \cite{Ansari21}, and \cite{Gao20} replace the minimization part of $f$-GAN by gradient descent in the Wasserstein space $\mathcal P_2(\R^d)$, and \cite{MN21} replaces the minimization part of MMD GAN by gradient descent on a statistical manifold. The resulting gradient-flow algorithms are then advocated as alternatives (and improvements) to the initially-given GAN algorithms. 
Contrary to the common practice ``modifying GANs to form a gradient-flow algorithm'', Proposition~\ref{prop:equivalence} shows that the (vanilla) GAN algorithm itself, without any modification, already computes gradient flows. 
As detailed in Section~\ref{subsec:mode collapse} below, this gradient-flow identification of GANs crucially reveals an unapparent fact: the (vanilla) GAN algorithm implicitly involves an MSE fitting, which can cause or aggravate the divergence of GANs.

Note that Algorithm~\ref{alg:ODE} resembles the gradient-flow algorithms in \cite{Gao19} and \cite{Gao20}, but a subtle difference sets them apart. In Algorithm~\ref{alg:ODE}, the estimation of $D$ (or equivalently, $\rho^{G(Z)}/\rhod$) is coordinated with the gradient-descent update of $G$, so that the {\it zero-sum} game \eqref{GAN} of GANs can be recovered. In \cite{Gao19} and \cite{Gao20}, as the estimation of $\rho^{G(Z)}/\rhod$ and the update of $G$ are attached to different objectives, their algorithms represent {\it non-zero-sum} games. 


\subsection{Implication for the Divergence of GANs}\label{subsec:mode collapse}

Mathematically, the recursive optimization in the GAN algorithm may compute either the desired min-max game value \eqref{GAN} or the corresponding max-min game value. If it computes the latter, the algorithm easily diverges in the form of {\it mode collapse}, i.e., the generator keeps producing $G:\R^n\to\R^d$ that maps many distinct $z\in\R^n$ to similar $y\in\R^d$ and fails to recover the entirety of $\rhod$; see Section 5.1.1 of \cite{Goodfellow16}. 
Beyond this typical ``max-min game'' diagnosis, our gradient-flow analysis zooms in on the GAN algorithm and points to an additional cause for the divergence of GANs. 

By the proof of Proposition~\ref{prop:equivalence}, the update of $G$ in the GAN algorithm can be decomposed into two steps: generating new points $\{y^{(i)}\}_{i=1}^m$ along ODE \eqref{Y GAN} (i.e., \eqref{new points}) and fitting $\{G(z^{i})\}_{i=1}^m$ to $\{y^{(i)}\}_{i=1}^m$ by minimizing the MSE (i.e., \eqref{MSE}). While the MSE fitting demands {\it point-wise} similarity (i.e., $G(z^{(i)})$ is close to $y^{(i)}$, for $i=1,2,...,m$), what we truly need is only {\it set-wise} similarity (i.e., 
the distribution of $\{G(z^{(i)})\}_{i=1}^m$ is similar to that of $\{y^{(i)}\}_{i=1}^m$). 

In Figure~\ref{figure new} (with $n=d=1$), for an initial $G:\R\to\R$ and samples $\{z^{(i)}\}_{i=1}^6$ of $Z$, each $x^{(i)}:= G(z^{(i)})$ is transported to $y^{(i)}$ via \eqref{new points}. Our goal is to modify $G$ such that $\{x^{(i)}\}_{i=1}^6$ becomes similar to $\{y^{(i)}\}_{i=1}^6$. To achieve point-wise similarity, the MSE fitting \eqref{MSE} 
moves each $x^{(i)}= G(z^{(i)})$ toward $y^{(i)}$. 
But since the distance from $x^{(i)}$ to $y^{(i)}$ can be quite large, the movement of $x^{(i)}$ will stop halfway, once the gradient updates $\nabla_{\theta_G}$ in \eqref{MSE} come to a halt near a local minimizer or saddle point in the parameter space for $\theta_G$. In Figure~\ref{figure new}(a), $x^{(i)}$ stops halfway at $\tilde x^{(i)}:= \tilde G(z^{(i)})$ before reaching $y^{(i)}$ (except $i=4$, for which $x^{(4)}$ and $y^{(4)}$ are close), where $\tilde G$ is the updated $G$ obtained from \eqref{MSE}. 
As the distribution of $\{\tilde x^{(i)}\}_{i=1}^6$ is distinct from that of $\{y^{(i)}\}_{i=1}^6$, we go astray from the gradient flow toward $\rhod$. 

The takeaway is that the MSE fitting {\it alone} can cause GANs to diverge. When $\rho^{G(Z)}$ is distinct from $\rhod$, the term $\frac{\nabla D(y)}{2(1-D(y))}$ in ODE \eqref{Y GAN D_t} can be large, so that $G(z^{(i)})$ and $y^{(i)}$ in \eqref{new points} can be far apart. In this case, as explained above, the MSE fitting \eqref{MSE} tends to drive us away from the gradient flow toward $\rhod$. More precisely, the updated $G$ moves $\rho^{G(Z)}$ {\it off} the gradient flow, such that it does not move closer to $\rhod$ in any reasonable sense. As the updated $\rho^{G(Z)}$ remains distinct from $\rhod$, the same argument above indicates that another training iteration is unlikely to help: it will simply update $\rho^{G(Z)}$ again to another distribution distinct from $\rhod$. That is, Algorithm~\ref{alg:ODE} (equivalent to the GAN algorithm) will generate indefinitely one distribution after another, all of which stay distinct from $\rhod$. This is reminiscent of {mode collapse} in Figure 2 of  \cite{Metz17}, where $\rho^{G(Z)}$ rotates among several distributions distinct from $\rhod$ and never converges. 
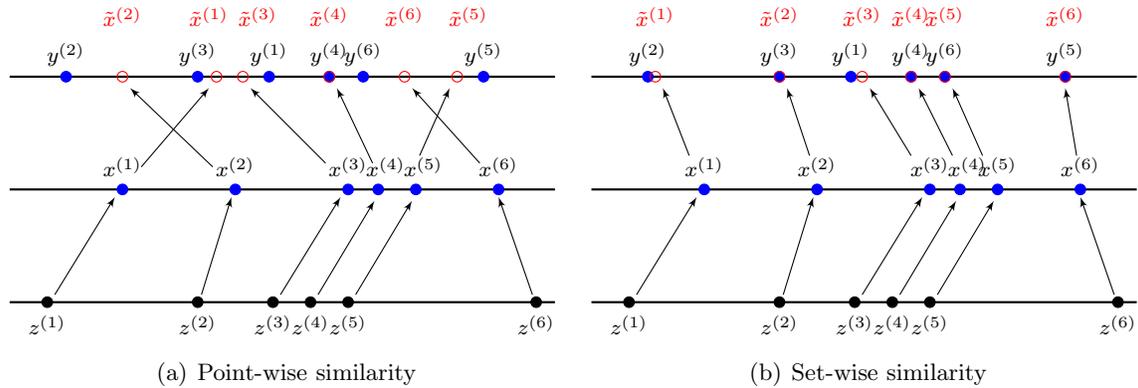
\begin{figure}[h]
\subfigure[Point-wise similarity]
{
\begin{tikzpicture}
\draw[thick] (0,0) -- (7.25,0);
\filldraw [] (0.5,0) circle (2pt);
\filldraw [] (2.5,0) circle (2pt);
\filldraw [] (3.5,0) circle (2pt);
\filldraw [] (4,0) circle (2pt);
\filldraw [] (4.5,0) circle (2pt);
\filldraw [] (7,0) circle (2pt);
\draw [] (0.5,-0.3) node{\scriptsize{$z^{(1)}$}};
\draw [] (2.5,-0.3) node{\scriptsize{$z^{(2)}$}};
\draw [] (3.5,-0.3) node{\scriptsize{$z^{(3)}$}};
\draw [] (4,-0.3) node{\scriptsize{$z^{(4)}$}};
\draw [] (4.5,-0.3) node{\scriptsize{$z^{(5)}$}};
\draw [] (7,-0.3) node{\scriptsize{$z^{(6)}$}};
\draw[thick] (0,1.5) -- (7.25,1.5);
\filldraw [blue] (1.5,1.5) circle (2pt);
\filldraw [blue] (3,1.5) circle (2pt);
\filldraw [blue] (4.5,1.5) circle (2pt);
\filldraw [blue] (4.9,1.5) circle (2pt);
\filldraw [blue] (5.4,1.5) circle (2pt);
\filldraw [blue] (6.5,1.5) circle (2pt);
\draw [] (1.5,1.8) node{\scriptsize{$x^{(1)}$}};
\draw [] (3,1.8) node{\scriptsize{$x^{(2)}$}};
\draw [] (4.5,1.8) node{\scriptsize{$x^{(3)}$}};
\draw [] (5,1.8) node{\scriptsize{$x^{(4)}$}};
\draw [] (5.5,1.8) node{\scriptsize{$x^{(5)}$}};
\draw [] (6.5,1.8) node{\scriptsize{$x^{(6)}$}};
\draw [->,>=latex'] (0.6,0.1) -- (1.4,1.4);
\draw [->,>=latex'] (2.55,0.1) -- (2.95,1.4);
\draw [->,>=latex'] (3.6,0.1) -- (4.4,1.4);
\draw [->,>=latex'] (4.1,0.1) -- (4.85,1.35);
\draw [->,>=latex'] (4.6,0.1) -- (5.35,1.35);
\draw [->,>=latex'] (6.95,0.1) -- (6.5,1.4);
\draw[thick] (0,3) -- (7.25,3);
\filldraw [blue] (3.45,3) circle (2pt);
\filldraw [blue] (0.75,3) circle (2pt);
\filldraw [blue] (2.5,3) circle (2pt);
\filldraw [blue] (4.25,3) circle (2pt);
\filldraw [blue] (6.3,3) circle (2pt);
\filldraw [blue] (4.7,3) circle (2pt);
\draw [] (3.45,3.3) node{\scriptsize{$y^{(1)}$}};
\draw [] (0.75,3.3) node{\scriptsize{$y^{(2)}$}};
\draw [] (2.5,3.3) node{\scriptsize{$y^{(3)}$}};
\draw [] (4.25,3.3) node{\scriptsize{$y^{(4)}$}};
\draw [] (6.3,3.3) node{\scriptsize{$y^{(5)}$}};
\draw [] (4.7,3.3) node{\scriptsize{$y^{(6)}$}};
\draw [red] (2.75,3) circle (2pt);
\draw [red] (1.5,3) circle (2pt);
\draw [red] (3.1,3) circle (2pt);
\draw [red] (4.25,3) circle (2pt);
\draw [red] (5.95,3) circle (2pt);
\draw [red] (5.25,3) circle (2pt);
\draw [red] (2.65,3.8) node{\scriptsize{$\tilde x^{(1)}$}};
\draw [red] (1.5,3.8) node{\scriptsize{$\tilde x^{(2)}$}};
\draw [red] (3.3,3.8) node{\scriptsize{$\tilde x^{(3)}$}};
\draw [red] (4.25,3.8) node{\scriptsize{$\tilde x^{(4)}$}};
\draw [red] (6.1,3.8) node{\scriptsize{$\tilde x^{(5)}$}};
\draw [red] (5.25,3.8) node{\scriptsize{$\tilde x^{(6)}$}};
\draw [->,>=latex'] (1.75,1.8) -- (2.65,2.85);
\draw [->,>=latex'] (2.7,1.8) -- (1.6,2.85);
\draw [->,>=latex'] (4.2,1.85) -- (3.2,2.85);
\draw [->,>=latex'] (4.825,1.85) -- (4.35,2.85);
\draw [->,>=latex'] (5.45,1.95) -- (5.85,2.85);
\draw [->,>=latex'] (6.25,1.85) -- (5.35,2.85);
\end{tikzpicture}
}
\subfigure[Set-wise similarity]
{
\begin{tikzpicture}
\draw[thick] (0,0) -- (7.25,0);
\filldraw [] (0.5,0) circle (2pt);
\filldraw [] (2.5,0) circle (2pt);
\filldraw [] (3.5,0) circle (2pt);
\filldraw [] (4,0) circle (2pt);
\filldraw [] (4.5,0) circle (2pt);
\filldraw [] (7,0) circle (2pt);
\draw [] (0.5,-0.3) node{\scriptsize{$z^{(1)}$}};
\draw [] (2.5,-0.3) node{\scriptsize{$z^{(2)}$}};
\draw [] (3.5,-0.3) node{\scriptsize{$z^{(3)}$}};
\draw [] (4,-0.3) node{\scriptsize{$z^{(4)}$}};
\draw [] (4.5,-0.3) node{\scriptsize{$z^{(5)}$}};
\draw [] (7,-0.3) node{\scriptsize{$z^{(6)}$}};
\draw[thick] (0,1.5) -- (7.25,1.5);
\filldraw [blue] (1.5,1.5) circle (2pt);
\filldraw [blue] (3,1.5) circle (2pt);
\filldraw [blue] (4.5,1.5) circle (2pt);
\filldraw [blue] (4.9,1.5) circle (2pt);
\filldraw [blue] (5.4,1.5) circle (2pt);
\filldraw [blue] (6.5,1.5) circle (2pt);
\draw [] (1.5,1.8) node{\scriptsize{$x^{(1)}$}};
\draw [] (3,1.8) node{\scriptsize{$x^{(2)}$}};
\draw [] (4.5,1.8) node{\scriptsize{$x^{(3)}$}};
\draw [] (5,1.8) node{\scriptsize{$x^{(4)}$}};
\draw [] (5.4,1.8) node{\scriptsize{$x^{(5)}$}};
\draw [] (6.5,1.8) node{\scriptsize{$x^{(6)}$}};
\draw [->,>=latex'] (0.6,0.1) -- (1.4,1.4);
\draw [->,>=latex'] (2.55,0.1) -- (2.95,1.4);
\draw [->,>=latex'] (3.6,0.1) -- (4.4,1.4);
\draw [->,>=latex'] (4.1,0.1) -- (4.85,1.35);
\draw [->,>=latex'] (4.6,0.1) -- (5.35,1.35);
\draw [->,>=latex'] (6.95,0.1) -- (6.5,1.4);
\draw[thick] (0,3) -- (7.25,3);
\filldraw [blue] (3.45,3) circle (2pt);
\filldraw [blue] (0.75,3) circle (2pt);
\filldraw [blue] (2.5,3) circle (2pt);
\filldraw [blue] (4.25,3) circle (2pt);
\filldraw [blue] (6.3,3) circle (2pt);
\filldraw [blue] (4.7,3) circle (2pt);
\draw [] (3.45,3.3) node{\scriptsize{$y^{(1)}$}};
\draw [] (0.75,3.3) node{\scriptsize{$y^{(2)}$}};
\draw [] (2.5,3.3) node{\scriptsize{$y^{(3)}$}};
\draw [] (4.25,3.3) node{\scriptsize{$y^{(4)}$}};
\draw [] (6.3,3.3) node{\scriptsize{$y^{(5)}$}};
\draw [] (4.7,3.3) node{\scriptsize{$y^{(6)}$}};
\draw [red] (0.85,3) circle (2pt);
\draw [red] (2.5,3) circle (2pt);
\draw [red] (3.6,3) circle (2pt);
\draw [red] (4.25,3) circle (2pt);
\draw [red] (4.7,3) circle (2pt);
\draw [red] (6.3,3) circle (2pt);
\draw [red] (0.85,3.8) node{\scriptsize{$\tilde x^{(1)}$}};
\draw [red] (2.5,3.8) node{\scriptsize{$\tilde x^{(2)}$}};
\draw [red] (3.6,3.8) node{\scriptsize{$\tilde x^{(3)}$}};
\draw [red] (4.25,3.8) node{\scriptsize{$\tilde x^{(4)}$}};
\draw [red] (4.7,3.8) node{\scriptsize{$\tilde x^{(5)}$}};
\draw [red] (6.3,3.8) node{\scriptsize{$\tilde x^{(6)}$}};
\draw [->,>=latex'] (1.3,1.95) -- (0.95,2.85);
\draw [->,>=latex'] (2.9,1.95) -- (2.6,2.85);
\draw [->,>=latex'] (4.25,1.95) -- (3.7,2.85);
\draw [->,>=latex'] (4.8,1.95) -- (4.35,2.85);
\draw [->,>=latex'] (5.2,1.95) -- (4.8,2.85);
\draw [->,>=latex'] (6.45,1.95) -- (6.3,2.85);
\end{tikzpicture}
}
\caption{\small{An initial $G:\R\to\R$ maps samples $\{z^{(i)}\}_{i=1}^6$ of $Z$ to $\{x^{(i)}\}_{i=1}^6$. Two distinct $\tilde G:\R\to\R$ serve to relocate $\{x^{(i)}\}_{i=1}^6$ to approximate $\{y^{(i)}\}_{i=1}^6$, the new points obtained in \eqref{new points}, under the criterion of point-wise similarity and set-wise similarity, respectively. The red circles represent $\tilde x^{(i)} := \tilde G(z^{(i)})$, $i=1,...,6$.}}
\label{figure new}
\end{figure}

In fact, if $x^{(i)}=G(z^{(i)})$ is allowed to move toward a nearby $y^{(j)}$, with $j$ possibly different from $i$, the update of $G$ can be made more effective. In Figure~\ref{figure new}(b), as each $x^{(i)}$ only needs to travel a short distance to a nearby $y^{(j)}$, it is likely to reach (or get very close to) $y^{(j)}$, before the gradient updates $\nabla_{\theta_G}$ come to a halt near a local minimizer or saddle point. The resulting distribution on $\R$ of $\{\tilde G(z^{(i)})\}_{i=1}^6$ is then very similar to that of $\{y^{(i)}\}_{i=1}^6$. This suggests a new potential route to encourage convergence of GANs: 
replacing the MSE (a measure of point-wise similarity) in \eqref{MSE} by a measure of set-wise similarity. 
This warrants a detailed analysis in itself and will be left for future research. 

\section{The Nonlinear Fokker-Planck Equation}\label{sec:FP} 
If there is a solution $Y$ to ODE \eqref{Y GAN} (which remains to be proved), a heuristic calculation shows that the density flow $u(t,\cdot) := \rho^{Y_t}(\cdot)\in\cP(\R^d)$, $t\ge 0$, satisfies the nonlinear Fokker-Planck equation
\begin{equation}\label{eq:FP}
\begin{split}
\frac{\partial u}{\partial t} (t,y) & = \frac12 \Div \left( \left(\frac{\nabla u(t,y)}{u(t,y)} -\frac{\nabla\rhod(y)+\nabla u(t,y)}{\rhod(y)+u(t,y)} \right) u(t,y)  \right), \quad u(0,y)=\rho_0(y).
\end{split}
\end{equation}
In view of this, we will construct a solution to \eqref{Y GAN} in a {\it backward} manner. In this section, we will establish the existence of a unique solution $u$ to \eqref{eq:FP}. The function $u$ will then be used, in Section~\ref{sec:ODE solution} below, to construct a solution $Y$ to \eqref{Y GAN}. 
This backward method was recently introduced in \cite{BR20}, for a McKean-Vlasov SDE depending on $\rho^{Y_t}$. As we will see, our case is markedly more involved than \cite{BR20}, due to the additional dependence on $\nabla\rho^{Y_t}$ in \eqref{Y GAN}. 


\subsection{Preparations}
For any fixed density flow $\{\rho_t\}_{t\ge 0}$ in $\cP(\R^d)$, whenever $\rho_t(\cdot)$ is weakly differentiable, we introduce the infinitesimal generator 
\begin{align}\label{cL}
\cL_{\rho_t}\varphi(t,y)&:= -\frac12 \left(\frac{\nabla \rho_t(t,y)}{\rho_t(t,y)} -\frac{\nabla\rhod(y)+\nabla \rho_t(t,y)}{\rhod(y)+\rho_t(t,y)} \right)\cdot \nabla \varphi, \quad \forall \varphi \in C^{1,1}((0,\infty)\times \R^d).
\end{align}
Thanks to integration by parts, we define a weak solution to \eqref{eq:FP} as follows.

\begin{definition}\label{def:weak sol u}
We say $u:[0,\infty)\to \mathcal{P}(\R^d)$ is a {\it weak solution} to the nonlinear Fokker-Planck equation \eqref{eq:FP}, if $u(t,\cdot)$ is weakly differentiable for a.e.\ $t\ge 0$ such that
\begin{equation}\label{eq:u2}
\int_0^\infty \int_{\R^d} \left(\varphi_t(t,y)  + \cL_{u(t,\cdot)} \varphi(t,y) \right)u(t,y) dydt =0,\quad \forall \varphi\in C^{1,2}_c((0,\infty)\times\mathbb{R}^d).
\end{equation}
\end{definition}

\begin{remark}\label{rem:u2'}
As we assume $\varphi(t,\cdot)\in C^2(\R^d)$ in \eqref{eq:u2}, using integration by parts twice gives
\[
\int_{\R^d} \left(\cL_{u(t,\cdot)} \varphi(t,y) \right)u(t,y) dy = \int_{\R^d} \bigg(\frac{1}{2}\frac{\nabla\rhod(y)+\nabla \rho_t(y)}{\rhod(y)+\rho_t(y)}\cdot \nabla \varphi +\frac{1}{2}\Delta \varphi\bigg) u(t,y) dy.
\]
Hence, \eqref{eq:u2} can be equivalently stated as 
\begin{equation}\label{eq:u2'}
\int_0^\infty \int_{\R^d} \left(\varphi_t(t,y)  + \widetilde\cL_{u(t,\cdot)} \varphi(t,y) \right)u(t,y) dydt =0,\quad \forall \varphi\in C^{1,2}_c((0,\infty)\times\mathbb{R}^d),
\end{equation}
where the generator $\widetilde \cL_{\rho_t}$ is defined, for any fixed density flow $\{\rho_t\}_{t\ge 0}$ in $\cP(\R^d)$, by
\begin{align}\label{cL'}
\widetilde \cL_{\rho_t}\varphi(t,y)&:=\frac{1}{2}\frac{\nabla\rhod(y)+\nabla \rho_t(y)}{\rhod(y)+\rho_t(y)}\cdot \nabla \varphi +\frac{1}{2}\Delta \varphi, \quad \forall \varphi \in C^{1,2}(\R_+\times \R^d).
\end{align}
That is to say, the Fokker-Planck equation \eqref{eq:FP} is equivalent (in the weak sense) to 
\begin{equation}\label{eq:FP'}
\begin{split}
\frac{\partial u}{\partial t} (t,y) & = -\Div \left(\frac12  \frac{\nabla\rhod(y)+\nabla u(t,y)}{\rhod(y)+u(t,y)} u(t,y) \right) + \frac12 \Delta u(t,y),\quad u(0,y)=\rho_0(y).
\end{split}
\end{equation}
\end{remark}



To find a weak solution $u$ to \eqref{eq:FP}, or equivalently \eqref{eq:FP'}, let us turn \eqref{eq:FP'} into a more amenable PDE. As $\rhod>0$ (Assumption~\ref{asm}), we consider the ansatz
\[
v(t,y):= \frac{u(t,y)}{\rhod(y)},\quad \forall (t,y)\in[0,\infty)\times\R^d. 
\]
It follows that $u_t=\rhod v_t$, so that
\begin{align*}
-\frac12\Div\left(\frac{\nabla (\rhod+u)}{\rhod+u} u\right)+\frac12 \Delta u &= -\frac12\Div\left(\nabla(\rhod+u)-\frac{\nabla (\rhod (1+v))}{1+v}\right)+\frac12 \Delta u\\
&= -\frac12 \Div\left(\nabla \rhod - \frac{\rhod \nabla (1+v)+(1+v)\nabla \rhod}{1+v}\right)\\
&= \frac12 \Div(\rhod \nabla \ln (1+v))=\frac12 \left(\rhod \Delta \ln (1+v) + \nabla \rhod \cdot \nabla \ln (1+v)\right).
\end{align*}
Hence, \eqref{eq:FP'} can be written in terms of $v$ as
\begin{equation}\label{eq:v0}
v_t=\frac{1}{2}\Delta_{\mud} \ln(1+v), \quad v(0)=\rho_0/\rhod.
\end{equation}
where the operator $\Delta_{\mud}$ is defined by
\begin{equation}\label{Laplacian}
\Delta_{\mud}:=\Delta + \nabla \ln \rhod \cdot \nabla.
\end{equation}

\begin{remark}\label{rem:Laplacian}
Thanks to integration by parts, 
\[
\int_{\R^d} (\Delta_{\mud} w) \varphi d\mud=\int_{\R^d} (\Delta w+ \nabla \ln \rhod \cdot \nabla w) \varphi d\mud = -\int_{\R^d} \nabla w \cdot \nabla \varphi d\mud, \quad \forall w, \varphi \in C^\infty_c(\R^d).
\]
That is, we can view $\Delta_{\mud}$ as the ``Laplacian operator" in the weighted Sobolev space $W^{1,2}(\R^d,\mud)$.
\end{remark}

In light of 
Remark~\ref{rem:Laplacian}, we define weak solutions to \eqref{eq:v0} as below.

\begin{definition}\label{def:weak sol v}
We say $v:[0,\infty)\to L^1(\R^d,\mud)$ is a {\it weak solution} to \eqref{eq:v0} {\it w.r.t.\ $\mud$}, if 
\begin{equation}\label{weak v}
\int_0^\infty \int_{\bR^d} \left( v(t,y) \varphi_t(t,y)+\frac{1}{2}\ln(1+v(t,y)) \Delta_{\mud} \varphi(t,y)\right)  d\mud dt = 0,\quad \forall \varphi\in C^{1,2}_c((0, \infty)\times \bR^d).
\end{equation}
\end{definition}
Note that a weak solution to \eqref{eq:v0} is defined {\it not} under the standard Lebesgue integration on $\R^d$ (in contrast to \eqref{eq:u2} and the classical definition of weak solutions), but under the integration w.r.t.\ $\mud$, the probability measure on $\R^d$ induced by $\rhod\in\cP(\R^d)$.
 
Now, let us recast the transformed Fokker-Planck equation \eqref{eq:v0} as a nonlinear Cauchy problem in the Banach space $L^1(\bR^d, \mud)$. Specifically, under the condition
\begin{equation}\label{rho_0 condition}
\rho_0 \le \beta \rhod\quad \hbox{for some}\ \ \beta>0,
\end{equation}
we define the operator $A: D(A)\subseteq L^1(\bR^d, \mud) \rightarrow L^1(\bR^d, \mud)$ by
\begin{equation}\label{A}
Av := -\frac{1}{2}\Delta_{\mud} \ln(1+v) = -\frac12 \Delta \ln (1+v) - \frac{1}{2} \nabla \ln \rhod \cdot \nabla \ln (1+v),
\end{equation}
where 
\[
D(A):= \{v\in L^1(\bR^d, \mud)\cap H^1_0(\bR^d, \mud): 0\le v\le \beta,\ Av\in L^1(\bR^d, \mud)\}
\]
and the differentiation in \eqref{A} is understood in the weak sense. Then, \eqref{eq:v0} can be expressed as the Cauchy problem
\begin{equation}\label{eq:v}
v_t+Av =0, \quad v(0)={\rho_0}/{\rhod}.
\end{equation}

\begin{remark}
Condition \eqref{rho_0 condition} states that the Radon-Nikodym derivative of the initial probability measure (induced by $\rho_0\in\cP(\R^d)$) w.r.t.\ the underlying measure $\mud$ is bounded. This covers all bounded and compactly supported $\rho_0\in\cP(\R^d)$. Also, note that $\beta\ge 1$ necessarily.
\end{remark}

Our first finding is that both operators $A$ and $-\Delta_{\mud}$ are ``accretive'' in $L^1(\R^d,\mud)$. To properly state this property, we will denote by $I$ the identity map on $L^1(\R^d,\mud)$ (i.e., $I(v)=v$ for all $v\in L^1(\R^d,\mud)$) in the next result, whose proof is relegated to Appendix~\ref{subsec:proof of lem:accretive}.

\begin{lemma}\label{lem:accretive}
\begin{itemize}
\item [(i)] The operator $A:D(A)\to L^1(\R^d,\mud)$ is accretive. That is, 
\[
\norm{v_1-v_2}_{L^1(\bR^d, \mud)}\le \norm{(I+\lambda A)v_1-(I+\lambda A)v_2}_{L^1(\bR^d, \mud)},\quad \forall v_1, v_2 \in D(A),\ \lambda>0. 
\]
\item [(ii)] The operator $-\Delta_{\mud}: D(-\Delta_{\mud})\to L^1(\R^d,\mud)$, with 
\[
D(-\Delta_{\mud}):=\{w\in L^1(\bR^d, \mud)\cap H^1_0(\bR^d, \mud): -\Delta_{\mud} w \in L^1(\bR^d, \mud)\},
\] 
is accretive. That is, 
\[
\norm{w_1-w_2}_{L^1(\bR^d, \mud)}\le \norm{(I-\lambda \Delta_{\mud}) w_1-(I-\lambda \Delta_{\mud}) w_2}_{L^1(\bR^d, \mud)},\ \ \forall w_1, w_2 \in D(-\Delta_{\mud}),\ \lambda>0.
\]
\end{itemize}
\end{lemma} 

As shown in Section~\ref{subsec:existence} below, the accretiveness of $A$ and $-\Delta_{\mud}$ will play a crucial role in proving the existence of a unique weak solution to \eqref{eq:v}. To get a glimpse of the potential use of the accretiveness, let us look at the equation $(I+\lambda A)v=f$. 

\begin{definition}
For any $\lambda>0$ and $f\in L^1(\R^d,\mud)$, we say $v\in D(A)$ is a weak solution to $(I+\lambda A)v=f$ w.r.t.\ $\mud$ if 
\begin{equation}\label{weak v'}
\int_{\bR^d} \left( v(y) \varphi(y)+\frac{\lambda}{2}\ln(1+v(y)) \Delta_{\mud} \varphi(y)\right)  d\mud = \int_{\R^d} f(y)\varphi(y) d\mud,\quad \forall \varphi\in C^{2}_c(\bR^d).
\end{equation}
\end{definition}


\begin{corollary}\label{coro:uniqueness}
For any $\lambda>0$ and $f\in L^1(\R^d,\mud)$, $(I+\lambda A)v=f$ has at most one weak solution $v\in D(A)$ w.r.t.\  $\mud$. Such a solution, if it exists, satisfies $\int_{\R^d} v d\mud = \int_{\R^d} f d\mud$. 
\end{corollary}

\begin{proof}
The uniqueness part follows from Lemma~\ref{lem:accretive} (i). Let $v\in D(A)$ be the unique weak solution to $(I+\lambda A)v=f$ w.r.t.\  $\mud$. Since $C^{2}_c(\bR^d)$ is dense in $H^1_0(\R^d,\mud)$, \eqref{weak v'} in fact holds for all $\varphi\in H^1_0(\R^d,\mud)$. Taking $\varphi\equiv 1\in H^1_0(\R^d,\mud)$ in \eqref{weak v'}  yields $\int_{\R^d} v d\mud = \int_{\R^d} f d\mud$.   
\end{proof}


\subsection{Existence and Uniqueness}\label{subsec:existence} 
Thanks to the accretiveness of $A$ in Lemma~\ref{lem:accretive}, we can invoke the Crandall-Liggett theorem to identify a sufficient condition, i.e., \eqref{domain condition} below, for \eqref{eq:v0} to admit a weak solution w.r.t.\ $\mud$. To properly state the result, we denote by $\ubar{D(A)}$ the closure of $D(A)$ under the $L^1(\R^d,\mud)$-norm. Also, for any $\lambda>0$, we introduce the range of $I+\lambda A$ in $L^1(\R^d,\mud)$, defined by
\[
R(I+\lambda A) := \{f\in L^1(\R^d,\mud) : \hbox{$(I+\lambda A)v=f$ has a weak solution $v\in D(A)$ w.r.t.\ $\mud$}\}. 
\]
By Corollary~\ref{coro:uniqueness}, for each $f\in R(I+\lambda A)$, the corresponding weak solution $v\in D(A)$ is unique. Hence, we will constantly write $v = (I+\lambda A)^{-1} f$.  

\begin{proposition}\label{prop:CL thm}
Let $\rho_0\in\cP(\R^d)$ satisfy \eqref{rho_0 condition}. If
\begin{equation}\label{domain condition}
\ubar{D(A)}\subseteq R(I+\lambda A)\quad \hbox{for $\lambda>0$ small enough},
\end{equation}
there exists a weak solution $v:[0,\infty)\to D(A)$ to \eqref{eq:v0} w.r.t.\ $\mud$ (cf.\ Definition~\ref{def:weak sol v}). 
Moreover, $v$ is continuous (i.e.,\ $v\in C([0,\infty);L^1(\R^d,\mud))$) and satisfies
\begin{equation}\label{exp form}
v(t)=\lim_{n\rightarrow \infty}\bigg(I+\frac{t}{n}A\bigg)^{-n}v(0)\quad\hbox{in}\ \ L^1(\R^d,\mud),\quad \forall t\ge 0,
\end{equation}
where the convergence is uniform in $t$ on compact intervals.
\end{proposition}

\begin{proof}
By definition, $\{v\in L^1(\R^d,\mud): v\in C^\infty_c(\bR^d),\ 0\le v\le \beta\}\subseteq D(A)$. As $C^\infty_c(\bR^d)$ is dense in $L^1(\bR^d, \mud)$, we have $\{v\in L^1(\R^d,\mud): 0\le v\le \beta\}\subseteq \ubar{D(A)}$. By \eqref{rho_0 condition}, this implies $\rho_0/\rhod\in \ubar{D(A)}$. Due to the accretiveness of $A$ in Lemma~\ref{lem:accretive} (i), \cite[p.99, (3.5)]{Barbu-book-2010} is satisfied with $\omega=0$; that is, $A$ is $\omega$-accretive with $\omega=0$, in the terminology of \cite{Barbu-book-2010}. Under the $\omega$-accretiveness of $A$, \eqref{domain condition}, and $\rho_0/\rhod\in \ubar{D(A)}$, the Crandall-Liggett theorem (see e.g., Theorem 4.3 in \cite{Barbu-book-2010}) asserts that \eqref{eq:v}, or equivalently \eqref{eq:v0}, has a unique {\it mild solution} $v\in C([0,\infty);L^1(\R^d,\mud))$ (in the sense of Definition 4.3 in \cite{Barbu-book-2010})), which fulfills \eqref{exp form}. 
Note that the operator $(I+\eps A)^{-1}:\ubar{D(A)}\to D(A)$ is well-defined for $\eps>0$ small, thanks to \eqref{domain condition} and Corollary~\ref{coro:uniqueness}, so that the right-hand side of \eqref{exp form} is well-defined. 
It remains to show that the mild solution $v$ is a weak solution w.r.t.\ $\mud$. 

As stated below Theorem 4.3 in \cite{Barbu-book-2010}, 
the mild solution $v\in C([0,\infty);L^1(\R^d,\mud))$ 
satisfies the following: for $\eps>0$ small enough, 
by taking $v^0 := \rho_0/\rhod$ and $v^i= (I+\eps A)^{-1}v^{i-1}\in D(A)$ for all $i\in\N$, 
we have
\begin{equation}\label{v_eps}
v(t) = \lim_{\eps\to 0} v_\eps(t)\ \ \hbox{in}\ L^1(\R^d,\mud),\ \ \hbox{uniformly in $t$ on compact intervals},
\end{equation}
where $v_\eps:[0, \infty)\to L^1(\R^d,\mud)$ is defined by $v_{\eps}(t) := v^{i-1}$ for $t\in [(i-1)\eps, i\eps)$, $i\in\N$. Observe from the construction of $v_\eps$ that
\[
\frac{v_\eps(t+\eps,y)-v_\eps (t,y)}{\eps} -\frac{1}{2}\Delta_{\mud} \ln(1+v_\eps) (t+\eps, y)=0.
\]
Hence, for any $\varphi\in C^{1,2}_c((0, \infty)\times \bR^d)$, by using the calculation in Remark~\ref{rem:Laplacian} twice, we get 
\begin{align*}
0&=\int_{\bR^d} \left( \frac{v_\eps(t+\eps,y)-v_\eps (t,y)}{\eps} \varphi(t+\eps,y) -\frac{1}{2} \ln (1+v_\eps) (t+\eps, y) \Delta_{\mud} \varphi(t+\eps,y)\right) d\mud\\
&= \int_{\bR^d} \bigg(\frac{(v_\eps \varphi)(t+\eps,y)- (v_\eps\varphi)(t,y)}{\eps}-v_\eps(t,y) \frac{\varphi(t+\eps, y)-\varphi(t,y)}{\eps}\\
&\hspace{2.5in}-\frac{1}{2}\ln (1+v_\eps)(t+\eps, y) \Delta_{\mud} \varphi(t+\eps,y)\bigg) d\mud.
\end{align*}
Integrating the above equation over $t\in (0,\infty)$ yields
\begin{align*}
&\int_0^\infty \int_{\R^d}\bigg(v_\eps(t,y) \frac{\varphi(t+\eps, y)-\varphi(t,y)}{\eps}+\frac{1}{2}\ln(1+v_\eps)(t+\eps, y) \Delta_{\mud} \varphi(t+\eps,y)\bigg) d\mud dt \\
&\hspace{4in}= -\frac{1}{\eps}\int_0^\eps \int_{\bR^d} (v_\eps \varphi)(t,y) d\mud dt. 
\end{align*}
As $\eps\rightarrow 0$, by $\varphi\in C^{1,2}_c((0, \infty)\times \bR^d)$ and $0\le v_\eps\le \beta$  $\forall \eps>0$ small (as $v_\eps\in\overline{D(A)}$ by construction), the left-hand side tends to $\int_0^\infty \int_{\bR^d} \left(v \varphi_t+\frac{1}{2}\ln(1+v) \Delta_{\mud} \varphi\right)  d\mud dt$ and the right-hand side vanishes. That is, \eqref{weak v} is satisfied, i.e., $v$ is a weak solution to \eqref{eq:v0} w.r.t.\ $\mud$. 
\end{proof}

To establish the sufficient condition \eqref{domain condition}, 
we need to show that as $\lambda>0$ is small enough, there exists a weak solution $v\in D(A)$ to $(I+\lambda A)v=f$ w.r.t.\ $\mud$, for any $f\in \ubar{D(A)}$. To this end, we consider the change of variable $w:=\ln (1+v)$ and rewrite $(I+\lambda A)v=f$ as
\begin{equation}\label{eq:nonlinear-w}
-\frac{\lambda}{2}\Delta_{\mud} w = -\frac{\lambda}{2} \Delta w - \frac{\lambda}{2} \nabla \ln \rhod \cdot \nabla w = f+1-e^{w}.
\end{equation}
To find a weak solution to \eqref{eq:nonlinear-w}, we use the method of subsolutions and supersolutions, motivated by Evans \cite[Section 9.3]{Evans-book-98}. Notably, the arguments therein needs to be extended from the integration w.r.t.\ $\Leb$ over a bounded domain to the integration w.r.t.\ $\mud$ over $\R^d$. 
The detailed derivation is relegated to Appendix~\ref{subsec:proof of lem:domain condition} and it leads to, in fact, a {\it stronger} version of \eqref{domain condition}. 

\begin{lemma}\label{lem:domain condition}
$\ubar{D(A)}\subseteq R(I+\lambda A)$ for all $\lambda>0$. That is, for any $\lambda>0$ and $f\in \ubar{D(A)}$, there exists a weak solution $v\in D(A)$ to $(I+\lambda A)v=f$ w.r.t.\ $\mud$. 
\end{lemma}

We now present the existence result for the nonlinear Fokker-Planck equation \eqref{eq:FP}. 

\begin{corollary}\label{coro:existence}
If $\rho_0\in\cP(\R^d)$ satisfies \eqref{rho_0 condition}, then the following hold.
\begin{itemize}
\item [(i)] There exists a weak solution $v\in C([0,\infty);L^1(\R^d,\mud))$ to \eqref{eq:v0} w.r.t.\ $\mud$ (see Definition~\ref{def:weak sol v}). Moreover, $0\le v\le \beta$, with $\beta>0$ as in \eqref{rho_0 condition}, and satisfies \eqref{exp form}. 
\item [(ii)] There exists a weak solution $u:[0,\infty)\to \cP(\R^d)$ to \eqref{eq:FP} (see Definition~\ref{def:weak sol u}), given by
\begin{equation}\label{u=rhod v}
u(t):= \rhod v(t)\quad\forall t\ge 0,\quad\hbox{with $v:[0,\infty)\to L^1(\R^d,\mud)$ taken from (i).}
\end{equation}
Moreover, $u$ is continuous in the sense that $u\in C([0,\infty);L^1(\R^d))$. 
\end{itemize}
\end{corollary}

\begin{proof}
By Lemma~\ref{lem:domain condition} and Proposition~\ref{prop:CL thm}, we obtain (i) immediately. For each $t\ge 0$, thanks to Corollary~\ref{coro:uniqueness}, $v^n(t) := (I+\frac{t}{n}A)^{-n} v(0)$ is well-defined and satisfies $\int_{\R^d}v^n(t)d\mud = \int_{\R^d}v(0)d\mud$ for all $n\in\N$. Since $v^n(t)\to v(t)$ in $L^1(\R^d,\mud)$ (recall \eqref{exp form}) and $\rho_0=u(0)=\rhod v(0)$, we get 
\begin{equation}\label{integrate v=1}
\int_{\R^d}v(t)d\mud = \int_{\R^d}v(0)d\mud =  \int_{\R^d}\rho_0 dy =1. 
\end{equation}
Using \eqref{u=rhod v} again then gives 
$
\int_{\R^d}u(t)dy = \int_{\R^d}\rhod v(t)dy = \int_{\R^d}v(t)d\mud  =1.
$
This, along with $u(t) = \rhod v(t)\ge 0$, shows that $u(t)\in \cP(\R^d)$. Now, for any $\{t_n\}$ in $[0,\infty)$ with $t_n\to t\ge 0$,
\begin{equation}\label{u to v in L^1}
\int_{\R^d} |u(t_n)-u(t)| dy = \int_{\R^d} |\rhod v(t_n)-\rhod v(t)| dy = \int_{\R^d} |v(t_n)-v(t)| d\mud\to 0\quad \hbox{as}\ n\to\infty,
\end{equation}
where the convergence stems from $v\in C([0,\infty);L^1(\R^d,\mud))$. This shows $u\in C([0,\infty);L^1(\R^d))$. 
Finally, for any $\varphi\in C^{1,2}_c((0,\infty)\times\mathbb{R}^d)$, observe from \eqref{u=rhod v} and integration by parts that 
\begin{align}
\int_0^\infty & \int_{\bR^d} \left(\varphi_t  + \widetilde\cL_{u(t)} \varphi \right)u  dydt = \int_0^\infty \int_{\bR^d} \left(u\varphi_t  - \frac{1}{2} (\rhod+u)\Div\left(\frac{u}{\rhod+u}\nabla \varphi\right) + \frac{1}{2}u \Delta \varphi \right)dydt\notag\\
&=\int_0^\infty \int_{\bR^d} \left( v\varphi_t  - \frac{1}{2} (1+v)\Div\left(\frac{v}{1+v}\nabla \varphi\right) + \frac{1}{2} v \Delta \varphi \right) d\mud dt\notag\\
&=\int_0^\infty \int_{\bR^d} \left(v\varphi_t  - \frac{\nabla v}{2(1+v)}\cdot \nabla \varphi \right) d\mud dt=\int_0^\infty \int_{\bR^d} \left(v\varphi_t - \frac{1}{2}  \nabla \ln (1+v)\cdot \nabla \varphi  \right) d\mud dt\notag\\
&=\int_0^\infty \int_{\bR^d} \left(v\varphi_t + \frac{1}{2}  \ln (1+v)\Delta_{\mud} \varphi  \right) d\mud dt =0,\label{weak v to weak u}
\end{align}
where the fifth equality stems from Remark~\ref{rem:Laplacian} and the last equality holds as $v$ satisfies \eqref{weak v'}. As such,  $u$ satisfies \eqref{eq:u2'}, and thus \eqref{eq:u2} by Remark~\ref{rem:u2'}, and is then a weak solution to \eqref{eq:FP}. 
\end{proof}



To establish the uniqueness of solutions $v$ to \eqref{eq:v0}, we observe that \eqref{eq:v0} is an initial-value problem that involves the {\it modified} Laplacian $\Delta_{\mud}$, which is applied to $v$ through a nonlinear function (i.e.,\ $\ln(1+v)$). In view of this, the uniqueness arguments in \cite{BC79}, for initial-value problems where the usual Laplacian $\Delta$ is applied in a nonlinear way, are expected to be useful. In Appendix~\ref{subsec:proof of prop:uniqueness}, a delicate plan is carried out to generalize the arguments in \cite{BC79} to our setting where $\Delta$ and the Lebesgue measure are replaced by $\Delta_{\mud}$ and $\mud$. It ultimately shows the following. 

\begin{proposition}\label{prop:uniqueness}
If $v_1,v_2\in C([0,\infty); L^1(\bR^d, \mud))\cap L_+^\infty([0,\infty)\times \bR^d,\mud)$ are weak solutions to \eqref{eq:v0} w.r.t.\ $\mud$, then 
\begin{equation}\label{uniqueness}
\int_0^\infty \int_{\bR^d} (v_1-v_2) \varphi d\mud dt =0,\quad \forall \varphi\in C^{1,2}_c((0,\infty)\times\R^d). 
\end{equation}
\end{proposition}


Now, we are ready to present the main result of this section.

\begin{theorem}\label{thm:FPE}
Let $\rho_0\in\cP(\R^d)$ satisfy \eqref{rho_0 condition}. Then, $u:[0,\infty)\to \cP(\R^d)$ in \eqref{u=rhod v} is the unique weak solution to the Fokker-Planck equation \eqref{eq:FP} among the class of functions
\begin{equation}\label{cC}
\cC:=\left\{\eta \in C([0,\infty); L^1(\R^d)): \eta/\rhod\in L^\infty_+([0,\infty)\times \R^d) \right\}.
\end{equation}
Specifically, if $\bar u:[0,\infty)\to \cP(\R^d)$ is a weak solution to \eqref{eq:FP} that belongs to $\cC$, then $\bar u(t,\cdot) =u(t,\cdot)$ $\Leb$-a.e.\ on $\R^d$ for all $t\ge0$. 
\end{theorem}

\begin{proof}
We know from Corollary~\ref{coro:existence} that $u:[0,\infty)\to \cP(\R^d)$ given in \eqref{u=rhod v} is a weak solution to \eqref{eq:FP} and it belongs to $\cC$. Suppose that there exists another weak solution $\bar u:[0,\infty)\to \cP(\R^d)$ to \eqref{eq:FP} that belongs to $\cC$. Then, $\bar v:= \bar u/\rhod$ is nonnegative and bounded. By the same calculation in \eqref{u to v in L^1} (with $u, v$ replaced by $\bar u, \bar v$), we see that $\bar u \in C(\R_+; L^1(\R^d))$ implies $\bar v \in C(\R_+; L^1(\R^d,\mud))$. Moreover, by repeating the arguments in \eqref{weak v to weak u}, but in a backward manner, we see that ``$\bar u$ fulfills \eqref{eq:u2'}'' (as $\bar u$ is a weak solution to \eqref{eq:FP}) implies ``$\bar v$ fulfills \eqref{weak v}'', i.e., $\bar v$ is a weak solution to \eqref{eq:v0} w.r.t.\ $\mud$. Now, in view of $\bar u =\rhod \bar v$ and \eqref{u=rhod v}, we conclude from Proposition~\ref{prop:uniqueness} that 
\[
\int_0^\infty \int_{\bR^d} (\bar u-u) \varphi dy dt = \int_0^\infty \int_{\bR^d} (\bar v-v) \varphi d\mud dt =0,\quad \forall \varphi\in C^{1,2}_c((0,\infty)\times\R^d), 
\]
where $v$ is given as in Corollary~\ref{coro:existence} (i). 
This implies $\bar u(t,\cdot) = u(t,\cdot)$ $\Leb$-a.e.\ on $\R^d$, for $\Leb$-a.e.\ $t\ge 0$. In the following, we show that this equality can be strengthened to hold for {\it all} $t\ge 0$. 
Recall $u\in C([0,\infty);L^1(\R^d))$ from Corollary~\ref{coro:existence}. For any $t\ge 0$, this implies that  
\begin{equation}\label{narrowly continuous}
\lim_{s\to t} \int_{\R^d} \psi(y) u(s,y)  dy = \int_{\R^d} \psi(y) u(t,y) dy\quad \forall \psi\in C_b(\R^d).
\end{equation}
Similarly, as $\bar u\in C([0,\infty);L^1(\R^d))$, \eqref{narrowly continuous} also holds with $u$ replaced by $\bar u$. Now, take $\{t_n\}$ in $[0,\infty)$ such that $t_n\to t$ and $\bar u(t_n,\cdot)=u(t_n,\cdot)$ $\Leb$-a.e.\ on $\R^d$ for all $n\in\N$. Then, for any $\psi\in C_b(\R^d)$,
\begin{align*}
\int_{\R^d} \psi(y) \bar u(t,y) dy &= \lim_{n\to \infty} \int_{\R^d} \psi(y) \bar u(t_n,y)  dy=\lim_{n\to \infty} \int_{\R^d} \psi(y) u(t_n,y)  dy = \int_{\R^d} \psi(y) u(t,y) dy, 
\end{align*}
which readily shows $\bar u(t,\cdot)=u(t,\cdot)$ $\Leb$-a.e.\ on $\R^d$. 
\end{proof}


\section{Solutions to Density-Dependent ODE \eqref{Y GAN}}\label{sec:ODE solution}
In this section, we focus on constructing a solution $Y$ to ODE \eqref{Y GAN}. The unique weak solution $u:[0,\infty)\to \cP(\R^d)$ to the Fokker-Planck equation \eqref{eq:FP} (Theorem~\ref{thm:FPE}) already gives important clues: as mentioned above \eqref{eq:FP}, if \eqref{Y GAN} admits a solution $Y$, the density flow $\{\rho^{Y_t}\}_{t\ge 0}$ should (heuristically) coincide with $\{u(t)\}_{t\ge 0}$. The question now is how we can actually construct such a solution $Y$ from the knowledge of $\{u(t)\}_{t\ge 0}$. 

To answer this, the first challenge is what constitutes a ``solution'' to \eqref{Y GAN}. 
As opposed to standard ODEs, \eqref{Y GAN} involves {\it two} different levels of randomness at time 0 and they trickle down as time passes by (through the otherwise deterministic dynamics), leaving $Y_t$ a random variable for all $t\ge 0$. To see this, we follow the idea 
``$\rho^{Y_t} = u(t)$ for all $t\ge 0$'', 
where $u:[0,\infty)\to\cP(\R^d)$ is the unique weak solution to \eqref{eq:FP}, and rewrite \eqref{Y GAN} as
\begin{equation}\label{Y GAN'}
dY_t = -\frac12 \left(\frac{\nabla u(t,Y_t)}{u(t,Y_t)} - \frac{\nabla \rhod(Y_t)+\nabla u(t,Y_t)}{\rhod(Y_t)+ u(t,Y_t)}\right)dt,\quad  \rho^{Y_0}=\rho_0 \in\cP(\R^d).
\end{equation}
Clearly, there is the randomness of the initial point $y\in\R^d$, stated explicitly through $\rho_0\in\cP(\R^d)$. Once an initial point $y\in\R^d$ is sampled, because the coefficient of the ODE in \eqref{Y GAN'} is not necessarily Lipschitz, there can be multiple continuous paths $\omega: [0,\infty)\to\R^d$ with $\omega(0)=y$ such that $t\mapsto Y_t := \omega(t)$ solves the ODE (with the initial condition $Y_0=y$). In other words, there remains the randomness of which continuous path to pick among those who solve the ODE in \eqref{Y GAN'} (with $Y_0=y$ fixed). In fact, these two levels of randomness can be jointly expressed by a probability measure $\P$ defined on the space of continuous paths
\[
(\Omega,\F) := (C([0,\infty);\R^d), \B(C([0,\infty);\R^d))),
\]
where $\B(C([0,\infty);\R^d))$ denotes the Borel $\sigma$-algebra of $C([0,\infty);\R^d)$.


\begin{definition}\label{def:solution Y}
A process $Y:[0,\infty)\times\Omega\to \R^d$ is said to be a solution to \eqref{Y GAN} if 
\begin{equation}\label{cmp}
Y_t(\omega):= \omega(t)\quad \forall (t,\omega)\in[0,\infty) \times\Omega, 
\end{equation}
and there exists a probability measure $\P$ on $(\Omega,\F)$ under which 
\begin{itemize}
\item [(i)] the density function $\eta_t\in \cP(\R^d)$ of $Y_t:\Omega\to\R^d$ exists and is weakly differentiable for all $t\ge 0$,  
and $\eta_0 = \rho_0$ $\Leb$-a.e.;
\item [(ii)] the collection of paths
\[
\Gamma := \left\{\omega\in\Omega : \omega(t) = \omega(0) -\frac12 \int_0^t   \left(\frac{\nabla \eta_s(\omega(s))}{\eta_s(\omega(s))} - \frac{\nabla \rhod(\omega(s))+\nabla \eta_s(\omega(s))}{\rhod(\omega(s))+ \eta_s(\omega(s))}\right)  ds,\ \forall t\ge 0 \right\}
\] 
has probability one.
\end{itemize} 
\end{definition}

\begin{remark}
In Definition~\ref{def:solution Y}, $\P$ serves to sample continuous paths $\omega:[0,\infty)\to\R^d$ from the set $\Gamma$, i.e., among those who fulfill the ODE in \eqref{Y GAN}, in a way that the density of $\omega(0)$ coincides with $\rho_0\in\cP(\R^d)$. 
\end{remark}

To show that a solution to \eqref{Y GAN} exists (in the sense of Definition~\ref{def:solution Y}), we will resort to Trevisan's superposition principle (Theorem 2.5 in \cite{Trevisan16}). To this end, appropriate integrability needs to be checked first. 

\begin{lemma}\label{lem:integrability}
Let $\rho_0\in\cP(\R^d)$ satisfy \eqref{rho_0 condition}. Then, the weak solution $v\in C([0,\infty);L^1(\R^d,\mud))$ to \eqref{eq:v0} w.r.t.\ $\mud$, specified as in Corollary~\ref{coro:existence} (i), satisfies 
\begin{equation}\label{gradient v bdd}
\int_0^\infty \norm{\nabla v}^2_{L^2(\bR^d, \mud)}dt \le (1+\beta)\beta^2.
\end{equation}
\end{lemma}

\begin{proof}
Recall that $v$ is constructed using the Crandall-Liggett theorem (see Proposition~\ref{prop:CL thm}) and thus satisfies \eqref{v_eps}, which involves $v_\eps:[0, \infty)\to L^1(\R^d,\mud)$ defined by $v_{\eps}(t) := (I+\eps A)^{-(i-1)} v(0)$ for $t\in [(i-1)\eps, i\eps)$ and $i\in\N$. 
By writing $z_i = v_\eps(i\eps)$ for $i\in\N\cup\{0\}$, we have $(I+\eps A)z_i=z_{i-1}$ for $i\in\N$. Let us multiply this equation by $z_i$ and integrate it w.r.t.\ $\mud$. By Remark~\ref{rem:Laplacian}, we get
\[
\int_{\R^d} z_i^2 d\mud+\frac{\eps}{2}\int_{\R^d} \nabla \ln(1+z_i)\cdot \nabla z_i d\mud =\int_{\R^d} z_{i-1}z_i d\mud,\quad \forall i\in\N.
\]
Thanks to $0\le z_i\le \beta$ (as $z_i\in D(A)$) and Young's inequality, the above yields 
\[
\norm{z_i}_{L^2(\bR^d, \mud)}^2+\frac{\eps}{2(1+\beta)}\norm{\nabla z_i}^2_{L^2(\bR^d, \mud)}\le \frac{1}{2}\norm{z_{i-1}}_{L^2(\bR^d, \mud)}^2+\frac{1}{2}\norm{z_i}_{L^2(\bR^d, \mud)}^2,\quad \forall i\in\N.
\]
That is to say, 
\[
{\eps}\norm{\nabla z_i}^2_{L^2(\bR^d, \mud)}\le (1+\beta)\left(\norm{z_{i-1}}_{L^2(\bR^d, \mud)}^2-\norm{z_{i}}_{L^2(\bR^d, \mud)}^2\right), \quad\forall i\in\N.
\]
By summing up the above relation over all $i\in\N$, we have
\begin{align*}
\int_\eps^{\infty} \norm{\nabla v_\eps(t)}^2_{L^2(\bR^d, \mud)} dt =\sum_{i=1}^{\infty} \eps\norm{\nabla z_i}^2_{L^2(\bR^d, \mud)}&\le (1+\beta) \norm{z_{0}}_{L^2(\bR^d, \mud)}^2 \le (1+\beta)\beta^2,
\end{align*}
where the last inequality follows from $0\le z_0\le \beta$, as $z_0 = v_\eps(0)=v(0)=\rho_0/\rhod\in \overline{D(A)}$; see the proof of Proposition~\ref{prop:CL thm}. As the above estimate holds for all $\eps>0$ small enough, we conclude that for any fixed $\bar\eps>0$ small,
\begin{equation}\label{bar eps}
\int_{\bar\eps}^{\infty} \norm{\nabla v_\eps(t)}^2_{L^2(\bR^d, \mud)} dt \le (1+\beta)\beta^2,\quad \forall \eps\in(0,\bar\eps]. 
\end{equation}
This shows that for any $j\in\{1, \ldots, d\}$, $\{\partial_{y_j} v_\eps\}_{\eps\in(0,\bar\eps]}$ is bounded in $L^2((\bar\eps,\infty)\times\R^d, \Leb\otimes\mud)$; hence, $\partial_{y_j} v_\eps$ converges weakly along a subsequence (without relabeling) to some $\eta_j \in L^2((\bar\eps,\infty)\times \bR^d, \Leb\otimes \mud)$.
We claim that for $\Leb$-a.e.\ $t\ge \bar\eps$, the weak derivative $\partial_{y_j} v(t, \cdot)$ exists and equals $\eta_j(t, \cdot)\in L^2(\bR^d, \mud)$.
For any $\psi\in L^\infty((\bar\eps,\infty))$ compactly supported and $\varphi\in C^\infty_c(\bR^d)$, we deduce from \eqref{v_eps}, the bounded convergence theorem, and integration by parts that
\begin{align*}
\int_{\bar\eps}^\infty \psi \int_{\R^d} &v \varphi_{y_j} dy dt=\lim_{\eps\to 0} \int_{\bar\eps}^\infty \psi \int_{\R^d} v_\eps \varphi_{y_j}dy dt=-\lim_{\eps\to 0} \int_{\bar\eps}^\infty \psi \int_{\R^d} \partial_{y_j} v_\eps \varphi dy dt\\
 &=- \lim_{\eps\rightarrow 0} \int_{\bar\eps}^\infty \int_{\R^d} \partial_{y_j} v_\eps  \frac{\psi\varphi}{\rhod} d\mud dt =- \int_{\bar\eps}^\infty  \int_{\R^d} \eta_j  \frac{\psi\varphi}{\rhod} d\mud dt =- \int_{\bar\eps}^\infty \psi \int_{\R^d} \eta_j \varphi dy dt.
\end{align*}
Note that the fourth equality above follows from $\psi \varphi/\rhod\in L^2((\bar\eps,\infty)\times \bR^d, \Leb\otimes \mud)$ and $\partial_{y_j} v_\eps\to \eta_j$ weakly in $L^2((\bar\eps,\infty)\times \bR^d, \Leb\otimes \mud)$. By taking $\psi\in L^\infty((\bar\eps,\infty))$ in the above equation as $\psi(t)=\text{sgn}\big(\int_{\R^d} (v(t)\varphi_{y_j} +\eta_j(t) \varphi) dy\big) 1_E(t)$ with $E\subset (\bar\eps,\infty)$, we see that
\[
\int_{\R^d} v(t)\varphi_{y_j}dy=-\int_{\R^d} \eta_j(t) \varphi dy \quad \text{for $\Leb$-a.e.}\ t\ge \bar\eps.
\]
As $\varphi\in C^\infty_c(\bR^d)$ is arbitrary, this readily shows that for $\Leb$-a.e.\ $t\ge \bar\eps$, the weak derivative $\partial_{y_j} v(t, \cdot)$ is $\eta_j(t, \cdot)\in L^2(\bR^d, \mud)$, i.e., the claim is proved. As a result, 
\begin{align*}
\int_{\bar\eps}^\infty \norm{\nabla v(t)}_{L^2(\bR^d, \mud)}^2dt=\int_{\bar\eps}^\infty \norm{\eta(t)}_{L^2(\bR^d, \mud)}^2dt &\le \int_{\bar\eps}^\infty \liminf_{\eps\to 0}\norm{\nabla v_\eps}_{L^2(\bR^d, \mud)}^2 dt\\
&\le \liminf_{\eps\to 0} \int_{\bar\eps}^\infty \norm{\nabla v_\eps}_{L^2(\bR^d, \mud)}^2 dt\le (1+\beta)\beta^2,
\end{align*}
where the first inequality stems from $\partial_{y_j} v_\eps(t,\cdot)\to \eta_j(t,\cdot)$ weakly in $L^2(\bR^d, \mud)$ for $\Leb$-a.e.\ $t\ge\bar\eps$ and all $j\in\{1,2,...,d\}$ and the norm $\|\cdot\|_{L^2(\R^d,\mud)}$ being weakly lower semicontinuous, the second inequality is due to Fatou's lemma, and the last inequality holds because of \eqref{bar eps}. 
As $\bar\eps>0$ can be chosen to be arbitrarily small, the desired result follows.
\end{proof}

A solution $Y$ to \eqref{Y GAN} can now be constructed. 

\begin{proposition}\label{prop:solution Y}
Let $\rho_0\in\cP(\R^d)$ satisfy \eqref{rho_0 condition}. Then, there exists 
a solution $Y$ to \eqref{Y GAN}. 
\end{proposition}

\begin{proof}
Let $u:[0,\infty)\to\cP(\R^d)$ be specified as in Theorem~\ref{thm:FPE}. Consider the collection of probabilities $\nu=\{\nu_t\}_{t\ge 0}$ on $\R^d$, with $d \nu_t(y) := u(t,y)dy$ for all $t\ge 0$. Recall that $u$ satisfies \eqref{narrowly continuous}, which directly implies that $\nu$ is narrowly continuous. Also, as $u$ is a weak solution to the Fokker-Planck equation \eqref{eq:FP}, $\nu$ solves (2.2) in \cite{Trevisan16} in the weak sense. Hence, to apply the ``superposition principle'' (Theorem 2.5 in \cite{Trevisan16}), it remains to verify (2.3) in \cite{Trevisan16}, which in our case takes the form
\begin{equation}\label{eq:u1}
\int_0^T \int_{\bR^d}\left|\frac{\nabla u(t,y)}{u(t,y)} - \frac{\nabla \rhod(y)+\nabla u(t,y)}{\rhod(y)+ u(t,y)}\right| u(t,y) dy dt<\infty,\quad \forall T>0.
\end{equation}
Indeed, ``$a_t$'' and ``$b_t$'' in (2.3) of \cite{Trevisan16} correspond to the diffusion and drift coefficients in \eqref{Y GAN'}, which are zero and $-\frac12 \big(\frac{\nabla u(t,y)}{u(t,y)} - \frac{\nabla \rhod(y)+\nabla u(t,y)}{\rhod(y)+ u(t,y)}\big)$ respectively. Now, on strength of $u= \rhod v$ and $v\ge 0$ (see \eqref{u=rhod v}), a direct calculation shows
\[
\int_{\bR^d}\left|\frac{\nabla u}{u} - \frac{\nabla \rhod+\nabla u}{\rhod+ u}\right| u dy = \int_{\R^d} \left|\frac{\nabla v}{1+v}\right| d\mud \le \int_{\R^d} \left|\nabla v\right| d\mud\le \norm{\nabla v}^2_{L^2(\bR^d, \mud)},
\]
where the last inequality follows from H\"{o}lder's inequality. Consequently, 
\[
\int_0^T \int_{\bR^d}\left|\frac{\nabla u(t,y)}{u(t,y)} - \frac{\nabla \rhod(y)+\nabla u(t,y)}{\rhod(y)+ u(t,y)}\right| u(t,y) dy dt\le \int_0^T \norm{\nabla v(t)}^2_{L^2(\bR^d, \mud)} dt <\infty,\quad \forall T>0, 
\]
where the finiteness follows from Lemma~\ref{lem:integrability}. That is, \eqref{eq:u1} is satisfied. 

Now, consider the coordinate mapping process $Y$ in \eqref{cmp}. Thanks to Theorem 2.5 of \cite{Trevisan16}, there exists a probability $\P$ on $(\Omega,\F)$ such that (i) $\P\circ (Y_t)^{-1} = \nu_t$ for all $t\ge 0$, and (ii) for any $f\in C^{1,2}((0,\infty)\times\R^d)$, the process
\begin{equation}\label{MartP}
t\mapsto  f(t, \omega(t))-f(0, \omega(0))-\int_0^t \left(\partial_s f-\frac12 \left(\frac{\nabla u}{u} - \frac{\nabla \rhod+\nabla u}{\rhod+ u}\right)\cdot\nabla f \right) (s,\omega(s)) ds 
\end{equation}
is a local martingale under $\P$. Note that (i) readily implies $\rho^{Y_t}=u(t)\in\cP(\R^d)$ under $\P$ for all $t\ge 0$, i.e., Definition~\ref{def:solution Y} (i) is satisfied. For each $i\in\{1,2,...,d\}$, by taking $f(t,y) = y_i$ for $y=(y_1,y_2,...,y_d)\in\R^d$ in \eqref{MartP}, we see that
\[
t\mapsto M^{(i)}_t(\omega) := (\omega(t))_i - (\omega(0))_i + \frac12 \int_0^t \left(\frac{\partial_{y_i} u}{u} - \frac{\partial_{y_i} \rhod+\partial_{y_i} u}{\rhod+ u}\right)  (s,\omega(s)) ds 
\]
is a local martingale under $\P$. By following the same arguments on p.\ 315 of \cite{KS-book-91}, we find that the quadratic variation of $M^{(i)}$ is constantly zero, i.e.\ $\langle M^{(i)}\rangle_t = 0$ for all $t\ge 0$, thanks to the lack of a diffusion coefficient in \eqref{Y GAN'}. As $M^{(i)}$ is a local martingale with zero quadratic variation under $\P$, $M^{(i)}_t = 0$ for all $t\ge 0$ $\P$-a.s. That is to say, 
\[
(\omega(t))_i  = (\omega(0))_i -\frac12 \int_0^t \left(\frac{\partial_{y_i} u}{u} - \frac{\partial_{y_i} \rhod+\partial_{y_i} u}{\rhod+ u}\right)  (s,\omega(s)) ds,\quad \hbox{for}\ \P\hbox{-a.e.}\ \omega\in\Omega. 
\]
Because this holds for all $i\in\{1,2,...,d\}$, we conclude that
\[
\omega(t) = \omega(0) -\frac12 \int_0^t \left(\frac{\nabla u}{u} -\frac{\nabla \rhod+\nabla u}{\rhod+ u}\right)  (s,\omega(s)) ds,\quad \hbox{for}\ \P\hbox{-a.e.}\ \omega\in\Omega, 
\]
which shows that Definition~\ref{def:solution Y} (ii) is also satisfied. Hence, $Y$ is a solution to \eqref{Y GAN}. 
\end{proof}

\begin{remark}
In the proof above, the probability $\P$ on $(\Omega,\F)$ can be decomposed into the two levels of randomness discussed above Definition~\ref{def:solution Y}. Let $\{Q_y\}_{y\in\R^d}$ be a {\it regular conditional probability} for $\F$ given $Y_0=\omega(0)$; that is, $\{Q_y\}_{y\in\R^d}$ is a set of probability measures on $(\Omega,\F)$ such that for any $A\in\F$, $y\mapsto Q_y(A)$ is Borel and $Q_y(A)=\P(A\mid Y_0=y)$ for $\rho_0(x)dx$-a.e.\ $y\in\R^d$. Such $\{Q_y\}_{y\in\R^d}$ indeed exists in the path space $(\Omega,\F)$; see e.g., Theorem 5.3.18 in \cite{KS-book-91}. Thus, we can express $\P$ as
\begin{equation}\label{two level}
\P(A) = \int_{\R^d} \rho_0(y) Q_y(A) dy,\quad \forall A\in\F.
\end{equation}
This explicitly shows that the randomness of the initial point $y\in\R^d$ is governed by $\rho_0\in\cP(\R^d)$. Once an initial point $y\in\R^d$ is sampled, the randomness of picking a solution to the ODE in \eqref{Y GAN'} (with $Y_0=y$ fixed) is dictated by the probability measure $Q_y$.  
\end{remark}

\begin{remark}
In view of Definition~\ref{def:solution Y} and Proposition~\ref{prop:solution Y}, the way we construct a solution to \eqref{Y GAN}---using a random selection of deterministic paths---is in line with L.C. Young's theory of generalized curves and close to the formulation in \cite{Ambrosio04}. By taking $\rho_0\equiv 1$ in \eqref{two level} (i.e., replacing the density $\rho_0\in\cP(\R^d)$ by the Lebesgue measure), one recovers the form of (5.3) in \cite{Ambrosio04}. 
\end{remark}

On the strength of the uniqueness result for the nonlinear Fokker-Planck equation \eqref{eq:FP} (see Theorem~\ref{thm:FPE}), the  uniqueness of solutions to \eqref{Y GAN} can be established accordingly. Recall the class of functions $\cC$ in \eqref{cC}. 

\begin{proposition}\label{prop:uniqueness Y}
Let $Y$ be a solution to \eqref{Y GAN} such that $\eta(t,y):=\rho^{Y_t}(y)$ satisfies
\begin{equation}
\eta\in\cC \quad\hbox{and}  \quad \nabla\eta \in L^1_{\operatorname{loc}}([0,\infty)\times \R^d).
\end{equation}
Then, $\eta(t,\cdot) = u(t,\cdot)$ $\Leb$-a.e.\ on $\R^d$ for all $t\ge 0$, 
where $u:[0,\infty)\to\cP(\R^d)$ is the unique weak solution to the Fokker-Planck equation \eqref{eq:FP} specified as in \eqref{u=rhod v}.  
\end{proposition}

\begin{proof}
Let $\P$ be the probability measure on $(\Omega, \F)$ associated with $Y$ (see Definition~\ref{def:solution Y}) and $\E$ be the expectation  under $\P$. For any $T>0$ and compact subset $K$ of $\R^d$, by Fubini-Tonelli's theorem,
\begin{align}\label{for Fubini}
\E \left[\int_{0}^T  \left|\frac{\nabla \eta}{\eta} -\frac{\nabla\rhod+\nabla \eta}{\rhod+\eta} \right| (t,Y_t) 1_{K}(Y_t) dt\right] &\le
\int_{0}^T  \int_{K} \left|\frac{\nabla \eta}{\eta} -\frac{\nabla\rhod+\nabla \eta}{\rhod+\eta} \right|\eta dy dt\notag\\
& \le \int_{0}^T  \int_{K} \left|\nabla \eta\right| +\left|\frac{\eta\nabla\rhod}{\rhod} \right| dy dt  <\infty,
\end{align}
where the finiteness follows from \eqref{eta condition} and Assumption~\ref{asm}. Now, for any $\varphi\in C^{1,2}_c((0,\infty)\times\R^d)$, whenever $T>0$ is large enough, we have
\begin{align*}
0 = \varphi(T, Y_T) - \varphi(0,Y_0) = \int_{0}^T \left(\partial_t\varphi -\frac12 \nabla \varphi\cdot \left(\frac{\nabla \eta}{\eta} -\frac{\nabla\rhod+\nabla \eta}{\rhod+\eta} \right)\right) (t,Y_t) dt,
\end{align*}
where the first equality is due to $\varphi(T, \cdot) = \varphi(0,\cdot)=0$ and the second equality follows from \eqref{Y GAN}. Taking expectation on both sides and using Fubini's theorem (applicable due to \eqref{for Fubini}) yield 
\begin{align*}
0 &= \int_{0}^T \int_{\R^d} \left(\partial_t\varphi -\frac12 \nabla \varphi\cdot \left(\frac{\nabla \eta}{\eta} -\frac{\nabla\rhod+\nabla \eta}{\rhod+\eta} \right)\right) \eta dy dt. 
\end{align*}
In view of \eqref{cL}, taking $T\to\infty$ in the above shows that $\eta$ fulfills \eqref{eq:u2}, i.e., $\eta$ is a weak solution to the nonlinear Fokker-Planck equation \eqref{eq:FP}. Thanks to $\eta\in\cC$, we conclude from Theorem~\ref{thm:FPE} that $\eta(t,\cdot)=u(t,\cdot)$ $\Leb$-a.e.\ on $\R^d$ for all $t\ge 0$. 
\end{proof}

\begin{remark}
The uniqueness in Proposition~\ref{prop:uniqueness Y} is weaker than the standard uniqueness of weak solutions to an SDE. The latter requires the law (or, the finite-dimensional distribution) of every solution $Y$ to an SDE to be identical, while the former requires only the marginal distribution of $Y_t$ to be the same for all $t\ge 0$.  
\end{remark}

The weaker notion of uniqueness in Proposition~\ref{prop:uniqueness Y} well serves our purpose. Since we intend to uncover the  data distribution $\rhod\in\cP(\R^d)$ through gradient descent in $\cP(\R^d)$ (governed by a solution $Y$ to \eqref{Y GAN}), we are curious about whether the marginal distribution $\rho^{Y_t}$ converges to $\rhod$ as $t\to\infty$. An affirmative answer is provided in the next section.


\section{Convergence to Data Distribution $\rhod$}\label{sec:to rhod}
This section is devoted to the main convergence theorem of this paper: for an appropriate solution $Y$ to \eqref{Y GAN},  
$\rho^{Y_t}\to \rhod$ in $L^1(\R^d)$ as $t\to\infty$. As the first step towards the final result, we show that the gradient descent feature encoded in \eqref{Y GAN} indeed works out: the distance between $\rho^{Y_t}$ and $\rhod$, i.e., $J(\rho^{Y_t}) = \JSD(\rho^{Y_t},\rhod)$, decreases over time.  

\begin{proposition}\label{prop:dJ/dt}
Let $Y$ be a solution to \eqref{Y GAN} such that $\eta(t,y):=\rho^{Y_t}(y)$ satisfies \eqref{eta condition}.
Then, for any $0 \le t_1 < t_2$, 
\begin{equation}\label{Jdecent-eq0}
J(\rho^{Y_{t_2}})- J(\rho^{Y_{t_1}}) \le - \int_{t_1}^{t_2} \int_{\bR^d} \left|\nabla \frac{\delta J}{\delta \rho}(\rho^{Y_t}, y)\right|^2 \rho^{Y_t}(y) dy dt\le 0.
\end{equation}
\end{proposition}

\begin{proof}
As $Y$ is a solution to \eqref{Y GAN} that fulfills \eqref{eta condition}, by Proposition~\ref{prop:uniqueness Y}, we have $\rho^{Y_t} = u(t)\in\cP(\R^d)$ for all $t\ge 0$, where $u:[0,\infty)\to\cP(\R^d)$ is the unique weak solution to the Fokker-Planck equation \eqref{eq:FP} specified as in \eqref{u=rhod v}. Recall $u\in C([0, \infty); L^1(\bR^d))$ from Corollary~\ref{coro:existence}. Also, by Remark~\ref{rem:equivalence}, $J(\cdot)=\JSD(\cdot,\rhod)$ is continuous under the $L^1(\R^d)$-norm. Hence, we may assume without loss of generality that $t_1,t_2\in\Q$. 

As $J(\cdot)=\JSD(\cdot,\rhod)$, by \eqref{JSD f} and the fact $u=\rhod v$ from \eqref{u=rhod v}, we have
\begin{align*}
J(\rho^{Y_t}) &= J(u(t)) 
= \ln 2 + \frac{1}{2} \int_{\R^d} v(t) \ln v(t) - (1+v(t))\ln (1+v(t)) d\mud,\quad \forall t\ge 0, 
\end{align*}
On the other hand, by Lemma~\ref{lem:delta J}, 
\[
\nabla \frac{\delta J}{\delta \rho}(\rho^{Y_t}, y) = \frac{1}{2} \left(\frac{\nabla u}{u} - \frac{\nabla (\rhod+u)}{\rhod + u}\right) = \frac{1}{2} \left(\frac{\nabla v}{v}-\frac{\nabla v}{1+v}\right) = \frac{1}{2}\frac{\nabla v}{v(1+v)}.
\]
Thus, \eqref{Jdecent-eq0} can be rewritten as 
\begin{equation}
\begin{aligned}
 \left[\int_{\bR^d} v(t) \ln v(t) - (1+v(t))\ln (1+v(t)) d\mud\right]_{t=t_1}^{t=t_2} \le - \frac{1}{2}\int_{t_1}^{t_2} \int_{\bR^d}\frac{|\nabla v|^2}{ v(1+v)^2} d\mud dt.\label{J decreases}
 \end{aligned}
\end{equation}
Now, recall that $v$ is constructed using the Crandall-Liggett theorem (see Proposition~\ref{prop:CL thm}) and thus satisfies \eqref{v_eps}, which involves $v_\eps:[0, \infty)\to L^1(\R^d,\mud)$ defined by $v_{\eps}(t) := (I+\eps A)^{-(i-1)} v(0)$ for $t\in [(i-1)\eps, i\eps)$ and $i\in\N$. By writing $z_i = v_\eps(i\eps)$ for $i\in\N\cup\{0\}$, we have $(I+\eps A)z_i=z_{i-1}$ for $i\in\N$.
For any $\gamma \in (0, 1)$, let us multiply the equation $z_{i-1}-z_{i}=\eps A z_i$ by $\ln (\frac{\gamma + z_i}{1+z_i})$ and integrate it w.r.t.\ $\mud$. This yields
\begin{align}
\int_{\R^d} (z_{i-1}-z_i) \ln \bigg(\frac{\gamma + z_i}{1+z_i}\bigg) d\mud &= \frac{\eps}{2}\int_{\R^d} \nabla \ln(1+z_i) \cdot \nabla \ln \left(\frac{\gamma + z_i}{1+z_i}\right) d\mud\notag\\
&= \frac{\eps}{2}\int_{\R^d} \frac{ (1-\gamma)|\nabla z_i|^2}{(\gamma+z_i)(1+z_i)^2} d\mud.\label{Jdecent-eq4}
\end{align}
Note that the integrand on the left hand side can be written as $F_{i-1}-F_i + R_i$, with
\begin{align}
F_i &:= (\gamma + z_i)\ln (\gamma+z_i) - (1+z_i)\ln(1+z_i),\label{F_i}\\
R_i &:= f(z_{i-1},z_i)\quad \hbox{where}\quad  f(a,b):= (\gamma+a)\ln\bigg(\frac{\gamma+b}{\gamma + a}\bigg) + (1+a)\ln\bigg(\frac{1+a}{1+b}\bigg)\notag
\end{align}
Observe that $f(a,b)\le 0$ for all $a,b\ge 0$. Indeed, by direct calculations, $f_a(a,b)< 0$ and $f_b(a,b)> 0$ for $a>b$, and $f_a(a,b)> 0$ and $f_b(a,b)< 0$ for $a<b$; this readily shows that $f$ is maximized as $a=b$, giving the value $f(a,a)=0$. With $R_i = f(z_{i-1},z_i)\le 0$, \eqref{Jdecent-eq4} implies
\[
\int_{\R^d} (F_{i-1}-F_i) d\mud \ge \frac{\eps}{2}\int_{\R^d} \frac{ (1-\gamma)|\nabla z_i|^2}{(\gamma+z_i)(1+z_i)^2} d\mud. 
\]
Fix $\delta>0$. As $t_1,t_2\in\Q$, there exists $0<\eps<\delta$ small enough such that $t_1, t_2$ are both integer multiples of $\eps$. Set $i_1:=t_1/\eps$ to $i_2 := t_2/\eps$. Summing up the previous inequality from $i=i_1+1$ to $i=i_2$ leads to 
\begin{equation*}
\int_{\R_d} F_{i_1} d\mud - \int_{\R_d} F_{i_2} d\mud \ge \frac{\eps}{2} \sum_{i=i_1+1}^{i_2} \int_{\R^d} \frac{ (1-\gamma)|\nabla z_i|^2}{(\gamma+z_i)(1+z_i)^2} d\mud.
\end{equation*}
By \eqref{F_i}, $z_i = v_\eps(i\eps)$, and the definition of $v_\eps$, the above can be stated in terms of $v_\eps$ as
\begin{equation}\label{Jdecent-eq7}
\begin{aligned}
&\left[\int_{\R^d} (\gamma + v_\eps(t)) \ln (\gamma + v_\eps(t)) - (1+v_\eps(t))\ln (1+v_\eps(t))d\mud\right]_{t=t_1}^{t=t_2}\\
 \le  &- \frac{1}{2}\int_{t_1+\eps}^{t_2+\eps} \int_{\R^d}  \frac{(1-\gamma) |\nabla v_\eps(t)|^2}{(\gamma + v_\eps(t))(1+v_\eps(t))^2}d\mud dt \le  - \frac{1}{2}\int_{t_1+\delta}^{t_2} \int_{\R^d}  \frac{(1-\gamma) |\nabla v_\eps(t)|^2}{(\gamma + v_\eps(t))(1+v_\eps(t))^2}d\mud dt, 
\end{aligned}
\end{equation}
where the last inequality holds as the integrand is nonnegative (recall $\gamma\in(0,1)$ and $v_\eps\ge 0$). 
By the proof of Lemma~\ref{lem:integrability} (under \eqref{bar eps} particularly), we know that $\nabla v_\eps$ converges weakly to $\nabla v$, along a subsequence, in $L^2([t_1+\delta, t_2]\times \R^d, \text{Leb}\otimes \mud)$.
This, along with \eqref{v_eps}, implies 
\[
\frac{\nabla v_\eps}{\sqrt{\gamma + v_\eps}(1+v_\eps)} \rightarrow \frac{\nabla v}{\sqrt{\gamma + v}(1+v)}  \quad \text{weakly in } L^2([t_1+\delta, t_2]\times \bR^d, \text{Leb}\otimes \mud).
\]
Since the $L^2([t_1+\delta, t_2]\times \R^d, \text{Leb}\otimes \mud)$-norm is weakly lower semicontinous, we have
\[
 \liminf_{\eps \rightarrow 0}\int_{t_1+\delta}^{t_2} \int  \frac{ |\nabla v_\eps(t)|^2}{(\gamma + v_\eps(t))(1+v_\eps(t))^2}d\mud dt \ge \int_{t_1+\delta}^{t_2} \int  \frac{ |\nabla v(t)|^2}{(\gamma + v(t))(1+v(t))^2}d\mud dt.
\]
Hence, as $\eps\rightarrow 0$ in \eqref{Jdecent-eq7}, by \eqref{v_eps}, the boundedness of $v_\eps$, and the above inequality, we obtain
\begin{align*}
\bigg[\int_{\R^d} (\gamma + v(t)) \ln (\gamma + v(t)) &- (1+v(t))\ln (1+v(t))d\mud\bigg]_{t=t_1}^{t=t_2}\\
& \le - \frac{1}{2}\int_{t_1+\delta}^{t_2} \int  \frac{(1-\gamma) |\nabla v(t)|^2}{(\gamma + v(t))(1+v(t))^2}d\mud dt.
\end{align*}
As $\delta>0$ is arbitrary, the above relation also holds for $\delta=0$. Then, as $\gamma\rightarrow 0$, by using the dominated convergence theorem on the left-hand side and Fatou's lemma on the right-hand side, we obtain \eqref{J decreases}, as desired.
\end{proof}

\begin{remark}\label{rem:no Ito's}
It it tempting to conclude from It\^{o}'s formula for a flow of measures (see e.g., Theorem 4.14 in \cite{CD-book-18-II}) 
and the identity \eqref{derivatives same} that \eqref{Jdecent-eq0} holds, and in fact with an equality. Indeed, this approach is taken in Theorem 2.9 of \cite{HRSS21}, where a similar equality (instead of an inequality) is established for the minimization of a free energy functional. Note that Theorem 4.14 in \cite{CD-book-18-II} and Theorem 2.9 in \cite{HRSS21} both require the smaller space $\cP_2(\R^d)$ and appropriate continuity and growth conditions under the 2-nd Wasserstein distance. As  these requirements are not met in our case, we take a fairly different approach in Proposition~\ref{prop:dJ/dt}: we analyze $t\mapsto J(\rho^{Y_{t}})$ through PDE \eqref{eq:v0}, relying particularly on the approximation of $v$ in \eqref{v_eps}. 
\end{remark}

By Proposition~\ref{prop:dJ/dt}, $\rho^{Y_t}$ moves closer to $\rhod$ continuously over time. The question now is whether $\rho^{Y_t}$ will in fact converge to $\rhod$. This indeed holds, at least along a subsequence. 

\begin{corollary}\label{coro:v to 1}
Let $Y$ be a solution to \eqref{Y GAN} such that $\eta(t,y):=\rho^{Y_t}(y)$ satisfies \eqref{eta condition}.
Then, there exists $\{t_n\}_{n\in\N}$ in $[0,\infty)$ with $t_n\uparrow\infty$ such that 
$\rho^{Y_{t_n}}\to \rhod$ in $L^1(\R^d)$. 
\end{corollary}

\begin{proof}
As argued at the beginning of the proof of Proposition~\ref{prop:dJ/dt}, $\rho^{Y_t} = u(t) = \rhod v(t)\in\cP(\R^d)$ for all $t\ge 0$, where $u$ and $v$ are specified as in Corollary~\ref{coro:existence}. In view of \eqref{gradient v bdd}, there exists $\{t_n\}_{n\in\N}$ in $[0,\infty)$ with $t_n\uparrow\infty$ such that $\norm{\nabla v(t_n)}_{L^2(\bR^d, \mud)}\rightarrow 0$. This, along with $0\le v(t)\le \beta$ for all $t\ge0$, implies that the sequence $\{v(t_n)\}_{n\in\N}$ is bounded in $H^1(\bR^d, \mud)$. Hence, $v_n$ converges weakly in $H^1(\bR^d, \mud)$, possibly along a subsequence (without relabeling), to some $v_\infty\in H^1(\bR^d, \mud)$. Note that ``$v(t_n)\to v_\infty$ weakly in $H^1(\bR^d, \mud)$'' readily implies ``$\nabla v(t_n) \to \nabla v_\infty$ weakly in $L^2(\R^d, \mud)$''. As we already know $\nabla v(t_n) \to 0$ strongly in $L^2(\R^d,\mud)$, we must have $\nabla v_\infty=0$, i.e., $v_\infty$ is constant on $\R^d$. 
On the other hand, by Sobolev embedding, we have $v(t_n)\to v_\infty$ strongly in $L^2_{\text{loc}}(\bR^d, \mud)$. By the uniform boundedness of $\{v(t_n)\}_{n\in\N}$, this can be upgraded to ``$v(t_n)\to v_\infty$ strongly in $L^p(\R^d, \mud)$, for all $p\in [1, \infty)$''. Recall from \eqref{integrate v=1} that $\int_{\R^d} v(t_n)d\mud=1$ for all $n\in\N$. This, together with $v_n\to v_\infty$ in $L^1(\R^d, \mud)$ and $v_\infty$ being constant on $\R^d$, entails $v_\infty\equiv 1$. It follows that $\|\rho^{Y_{t_n}}-\rhod\|_{L^1(\R^d)}=\|u(t_n)-\rhod\|_{L^1(\R^d)}=\|v(t_n)-1\|_{L^1(\R^d,\mud)}   \to 0$. 
\end{proof}

\begin{remark}\label{rem:no compactness}
In the language of dynamical systems (see e.g., Section 4.3 in \cite{Henry-book-81}), Corollary~\ref{coro:v to 1} states that the ``omega limit set'', defined by 
\[
\Theta := \{\rho\in\cP(\R^d) : \exists\{s_n\}\ \hbox{in}\ [0,\infty)\ \hbox{with}\ s_n\uparrow \infty\ \hbox{s.t.}\ \rho^{Y_{s_n}}\to \rho\ \hbox{in}\ L^1(\R^d)\},
\]
must contain $\rhod$. A result of this kind, in Step 1 in the proof of \cite[Theorem 2.11]{HRSS21}, was obtained by a compactness argument, applicable under Lipschitz and growth conditions on a McKean-Vlasov SDE. As no such conditions are met in our case, compactness is elusive and we instead rely on \eqref{gradient v bdd}, a property of the solution $v$ to PDE \eqref{eq:v0}. 
\end{remark}

Now, we are ready to prove Theorem~\ref{thm:main}, the main theoretic result of this paper.\\


\begin{proof}[Proof of Theorem~\ref{thm:main}]
By Proposition~\ref{prop:dJ/dt}, $t\mapsto J(\rho^{Y_t})\ge 0$ is nonincreasing. This implies $\ell:=\lim_{t\to\infty} J(\rho^{Y_t})\ge 0$ is well-defined; moreover, for any $\{s_n\}_{n\in\N}$ in $[0,\infty)$ with $s_n\uparrow\infty$, we have $\lim_{n\to\infty} J(\rho^{Y_{s_n}})=\ell$. In view of Corollary~\ref{coro:v to 1}, there exist $\{t_n\}_{n\in\N}$ in $[0,\infty)$ with $t_n\uparrow\infty$ such that $\rho^{Y_{t_n}}\to \rhod$ in $L^1(\R^d)$. Thanks to the continuity of $J(\cdot)$ under the $L^1(\R^d)$-norm (see Remark~\ref{rem:equivalence}), we get
$\ell = \lim_{n\to\infty} J(\rho^{Y_{t_n}}) = J(\rhod)=0
$.
Now, for any $\{s_n\}_{n\in\N}$ in $[0,\infty)$ with $s_n\uparrow\infty$, we have 
\[
\lim_{n\to\infty} \JSD(\rho^{Y_{s_n}},\rhod) = \lim_{n\to\infty} J(\rho^{Y_{s_n}})=\ell =0.
\]
By Remark~\ref{rem:equivalence} again, this implies $\rho^{Y_{s_n}}\to \rhod$ in $L^1(\R^d)$. As the sequence $\{s_n\}_{n\in\N}$ is arbitrarily picked, we conclude that $\rho^{Y_t}\to\rhod$ in $L^1(\R^d)$ as $t\to\infty$. 
\end{proof}

\acks{We would like to acknowledge support for this project
from the National Science Foundation (NSF grant DMS-2109002)
and the Natural Sciences and Engineering Research Council of Canada (NSERC Discovery Grant RGPIN-2020-06290). }





\appendix

\section{Proofs for Section~\ref{sec:preliminaries}}

\subsection{An Auxiliary Result}

\begin{lemma}\label{lem:equivalence}
For any $\rho,\bar\rho\in\cP(\R^d)$, $\TV(\rho,\bar\rho) = \frac12\|\rho-\bar\rho\|_{L^1(\R^d)}$. 
\end{lemma}

\begin{proof}
Consider $A_* := \{x\in\R^d: \rho(x)>\bar\rho(x)\}$ and observe from \eqref{TV} that
\begin{align*}
\TV(\rho,\bar\rho) &= \max\bigg\{ \int_{A_*} \left(\rho(x) - \bar\rho(x)\right) dx, \int_{A^c_*} \left(\bar\rho(x) - \rho(x)\right) dx \bigg\}.
\end{align*}
By a direct calculation, 
\begin{align*}
\int_{A_*} \left(\rho(x) -\bar\rho(x)\right) dx &= \frac12\bigg\{\int_{A_*} \left(\rho(x) - \bar\rho(x)\right) dx +\bigg(1-\int_{A_*^c} \rho(x) dx\bigg) -\bigg(1- \int_{A_*^c} \bar\rho(x) dx\bigg) \bigg\}\\
&= \frac12\bigg\{\int_{A_*} (\rho(x)-\bar\rho(x)) dx  +\int_{A_*^c} (\bar\rho(x)-\rho(x)) dx \bigg\} 
= \frac12 \|\rho-\bar\rho\|_{L^1(\R^d)}. 
\end{align*}
As we similarly have $\int_{A_*^c} (\bar\rho(x) - \rho(x)) dx = \frac12 \|\rho-\bar\rho\|_{L^1(\R^d)}$, the desired result follows. 
\end{proof}

\subsection{Proof of Proposition~\ref{prop:the gradient}}\label{subsec:proof of prop:the gradient}
Given $\xi\in C^2_c(\R^d;\R^d)$, there exists $\theta\in C^3_c(\R^d;\R)$ such that $\nabla\theta =\xi$. For any $\eps>0$, consider $\psi:\R^d\to\R$ defined by $\psi(x) = \frac12|x|^2+\eps\theta(x)$ for all $x\in\R^d$. As $\nabla\xi$ is bounded (by $\xi\in C^2_c(\R^d;\R^d)$), we can take $\eps>0$ small enough such that $\psi$ is strictly convex, or equivalently, $det(I+\eps\nabla\xi)>0$. This in turn implies that $\nabla\psi = I+\eps\xi$ is one-to-one, such that the inverse $(I+\eps\xi)^{-1}$ is well-defined. Hence, by Lemma 5.5.3  in \cite{Ambrosio-book-05}, the measure $\mu^\xi_\eps$ in \eqref{pushforward} has a density $\rho^\xi_\eps\in\cP(\R^d)$, which is given by
\begin{equation}\label{rho^xi_eps}
\rho^\xi_\eps(y) := {\rho\left((I+\eps\xi)^{-1}(y)\right)}/{det(\operatorname{Id}+\eps\nabla\xi)},\quad \forall y\in\R^d,
\end{equation}
where $\operatorname{Id}$ denotes the $d\times d$ identity matrix. It can be shown that $(y, \eps)\mapsto \rho^{\xi}_\eps(y)$ is $C^1$ on $\R^d\times [0,1]$ and satisfies $\rho^\xi_\eps(y)\mid_{\eps=0}=\rho(y)$ and $\frac{\partial\rho^\xi_\eps(y)}{\partial\eps}\mid_{\eps=0} = -\nabla\cdot (\rho(y)\xi(y))$ (see equation (10.4.7) in \cite{Ambrosio-book-05}). Now, by \eqref{delta G},
\begin{align}
\lim_{\eps\to 0} \frac{G(\rho^\xi_\eps) - G(\rho)}{\eps} &= \lim_{\eps\to 0} \int_0^1 \int_{\R^d} \frac{\delta G}{\delta \rho}\big((1-\lambda)\rho+\lambda \rho^\xi_\eps, y\big) \frac{\rho^\xi_\eps-\rho}{\eps}(y)dy d\lambda\notag\\
& =  \lim_{\eps\to 0} \int_0^1 \int_{K} \frac{\delta G}{\delta \rho}\big((1-\lambda)\rho+\lambda \rho^\xi_\eps, y\big)\frac{\partial\rho^\xi_\eps(y)}{\partial\eps}\mid_{\eps = \eps_0(y)} dy d\lambda,\label{uniform converg.}
\end{align}
where $\eps_0(y) \in [0,\eps]$ is obtained from the mean value theorem, and $K \subseteq \R^d$ is the compact support of $\xi$. The continuity of $\frac{\partial\rho^\xi_\eps(y)}{\partial\eps}$ in $(y, \eps)$ implies that $\frac{\partial\rho^\xi_\eps(y)}{\partial\eps}\mid_{\eps = \eps_0(y)}$ 
is bounded on $K$.
Since $\rho_\eps^\xi \rightarrow \rho$ uniformly, $\frac{\delta G}{\delta \rho}((1-\lambda)\rho+\lambda \rho^\xi_\eps, y)$, as a function of $(\lambda, y)$, converges uniformly to $\frac{\delta G}{\delta \rho}(\rho, y)$ on $[0,1]\times K$ as $\eps\rightarrow 0$ by assumption. 
Dominated convergence theorem then allows us to 
exchange the limit with the integral, yielding
\begin{align*}
\lim_{\eps\to 0} \frac{G(\rho^\xi_\eps) - G(\rho)}{\eps} &= -\int_0^1 \int_{\R^d} \frac{\delta G}{\delta \rho}\big(\rho, y\big) \nabla\cdot (\rho(y)\xi(y))dy d\lambda= \int_{\R^d} \nabla \frac{\delta G}{\delta \rho}\big(\rho,y\big)\cdot \xi(y)d\mu(y),
\end{align*}
where the last equality is due to integration by parts.


\subsection{Proof of Lemma~\ref{lem:delta J}}\label{subsec:proof of lem:delta J}
Recall from \eqref{JSD f} that 
we can write 
$J(\rho)=\JSD(\rho,\rhod)=\int_{\R^d} f\big(\frac{\rho(y)}{\rhod(y)}\big)\rhod(y) dy$,
where $f(x):= \frac12\big[ (x+1)\ln(\frac{2}{x+1})+x\ln x\big]$,
$\forall x\in [0,\infty)$, is a convex function satisfying $0\le f(x)\le \frac{\ln 2}{2}(x+1)$. Fix $\rho,\bar\rho\in\cP(\R^d)$. For any $\lambda\in [0,1]$, set $\rho_\lambda := (1-\lambda)\rho +\lambda \bar \rho \in \cP(\R^d)$, and define
\[g_\lambda := f'\left(\frac{\rho_{\lambda}}{\rhod}\right) (\bar\rho - \rho).\]
Using the convexity of $f$, one can show that $g_\lambda$ is increasing in $\lambda \in [0,1]$ and integrable on $\R^d$ if $\lambda \in (0,1)$. Indeed, the monotonicity is implied from $\frac{d}{d\lambda} g_\lambda =  f''\left(\frac{\rho_{\lambda}}{\rhod}\right) \frac{(\bar\rho - \rho)^2}{\rhod} \ge 0$, and the integrability is a consequence of the relation
\[-\frac{\ln 2}{2\lambda} (\rho+\rhod) \le \frac{f\left(\frac{\rho_\lambda}{\rhod}\right)-f\left(\frac{\rho}{\rhod}\right)}{\lambda} \rhod \le g_\lambda \le \frac{f\left(\frac{\bar\rho}{\rhod}\right)-f\left(\frac{\rho_{\lambda}}{\rhod}\right)}{1-\lambda} \rhod \le \frac{\ln 2}{2(1-\lambda)}(\bar \rho+\rhod).\]
Now, let $0<\lambda_1 <\lambda_2 < 1$. By the fundamental theorem of calculus and Fubini's theorem, 
\begin{align*}
J(\rho_{\lambda_2})-J(\rho_{\lambda_1})
& = \int_{\R^d}\left[f\left(\frac{\rho_{\lambda_2}(y)}{\rhod(y)}\right) - f\left(\frac{\rho_{\lambda_1}(y)}{\rhod(y)}\right)\right] \rhod(y) dy\\
& =  \int_{\R^d} \int_{\lambda_1}^{\lambda_2} g_\lambda(y) d\lambda dy =   \int_{\lambda_1}^{\lambda_2} \int_{\R^d} g_\lambda(y) dy d\lambda.
\end{align*}
Letting $\lambda_1\rightarrow 0$ and $\lambda_2\rightarrow 1$ yields
\[J(\bar\rho)-J(\rho) = \int_0^1 \int_{\R^d} g_\lambda(y) dy d\lambda = \int_0^1 \int_{\R^d}  f'\left(\frac{\rho_{\lambda}(y)}{\rhod(y)}\right) (\bar\rho(y) - \rho(y)) dy d\lambda ,\]
where the integral over $\lambda$ may be interpreted as an improper integral if either $\int_{\R^d} g_0(y) dy=-\infty$ or $\int_{\R^d} g_1(y) dy=\infty$. This readily implies  
$\frac{\delta J}{\delta \rho}(\rho, y) =  f'\big(\frac{\rho(y)}{\rhod(y)}\big)= \frac{1}{2} \ln\big( \frac{2\rho(y)}{\rho(y)+\rhod(y)}\big).$

\section{Proofs for Section~\ref{sec:FP}}

\subsection{Derivation of Lemma~\ref{lem:accretive}}\label{subsec:proof of lem:accretive}

\begin{lemma}\label{lemma:accretive}
Let $E\subseteq\R$ be convex and $g: E\to \R$ be a concave, strictly increasing function such that for any $v:\R^d\to\R$ with $v(\R^d)\subseteq E$, the following hold: 
\begin{itemize}
\item[(i)] if $v\in H^1_0(\bR^d, \mu_d)$, then $g(v)\in H^1_0(\bR^d, \mu_d)$;
\item[(ii)] if $v\in L^1(\bR^d, \mu_d)$, then $1/(g'_+)(v)\in L^1(\bR^d, \mud)$ (where $g'_+$ is the right derivative of $g$). 
\end{itemize}
Then, the operator $-\Delta_{\mud}g:D(-\Delta_{\mud}g)\to L^1(\bR^d, \mu_d)$,  with
\[
D(-\Delta_{\mud}g) := \{v\in L^1(\bR^d, \mud) \cap H^1_0(\bR^d, \mu_d): v(\R^d)\subseteq E,\  -\Delta_{\mud}g(v)\in L^1(\bR^d, \mu_d) \},
\]
is accretive in $L^1(\bR^d, \mu_d)$. That is, for any $v_1, v_2 \in D(-\Delta_{\mud}g)$ and $\lambda>0$, 
\[
\norm{v_1-v_2}_{L^1(\bR^d, \mud)}\le \norm{(v_1-\lambda \Delta_{\mud}g(v_1))-(v_2-\lambda \Delta_{\mud}g(v_2))}_{L^1(\bR^d, \mud)}.
\]
\end{lemma}

\begin{proof}
Fix $v_1, v_2 \in D(-\Delta_{\mud}g)$ and $\lambda>0$. Set $f_i := v_i-\lambda \Delta_{\mud}g(v_i)$ for $i=1,2$. For any $\varphi \in H^1(\bR^d, \mud) \cap L^\infty(\bR^d, \mud)$, by multiplying the equation $ v_i-\lambda \Delta_{\mud}g(v_i) = f_i$ by $\varphi$ and using integration by parts as in Remark~\ref{rem:Laplacian}, we get
\[
\int_{\bR^d} v_i \varphi d\mud + \lambda \int_{\bR^d} \nabla g(v_i) \cdot \nabla \varphi d\mud=\int_{\bR^d} f_i \varphi d\mud,\quad i=1,2,
\]
Note that under integration by parts, the boundary term for the second integral above vanishes because $g(v_i)\in H^1_0(\bR^d, \mud)$ by condition (i). Let us write $v:=v_1-v_2$, $h:=g(v_1)-g(v_2)$ and $f:=f_1-f_2$. We obtain from the above equation (by subtraction) that
\begin{equation}\label{eq:accretive}
\int_{\bR^d} v\varphi d\mud + \lambda \int_{\bR^d} \nabla h \cdot \nabla \varphi d\mud=\int_{\bR^d} f \varphi d\mud.
\end{equation}
Now, for each $\eps>0$, consider the function
\[
\chi_{\eps}(x) := 1_{\{|x|\le \eps\}}x/\eps+1_{\{|x|>\eps\}}\text{sgn}(x)\quad \hbox{for}\ x\in\bR,
\] 
which is a bounded, continuous, and weakly differentiable approximation of the sgn($\cdot$) function. 
Taking $\varphi=\chi_\eps(h)\in H^1(\bR^d, \mud) \cap L^\infty(\bR^d, \mud)$ in \eqref{eq:accretive} yields
\begin{align*}
&\int_{\bR^d} v\chi_\eps\left(h\right)d\mud+\lambda \int_{\bR^d} \left|\nabla h \right|^2\chi'_\eps\left(h\right)d\mud=\int_{\bR^d} f \chi_\eps\left(h\right) d\mud.
\end{align*}
As $\chi_\eps$ is nondecreasing, the middle integral above is nonnegative. Consequently, we have
\begin{align*}
\norm{v}_{L^1(\bR^d,\mud)}+\int_{\left\{\left| h\right|\le \eps\right\}} \left(v\frac{h}{\eps} -|v| \right)d\mud \le \int_{\bR^d} f \chi_\eps\left(h\right)d\mu\le \norm{f}_{L^1(\bR^d,\mud)},
\end{align*}
where we used the fact $\text{sgn}(v)=\text{sgn}(h)$. Finally, observe that the integral on the left-hand side above converges to zero as $\eps\rightarrow 0$. Indeed, 
\begin{align*}
\int_{\left\{\left|h \right|\le \eps\right\}} \left|v\frac{h}{\eps} -|v| \right|d\mud
&\le \int_{\left\{\left|h \right|\le \eps\right\}} 2|v| d\mud \le 2\eps \int_{\bR^d}  \frac{1}{g'_+(\max(v_1, v_2))} d\mud\rightarrow 0,
\end{align*}
where the second inequality follows from $|h| = |g(v_1)-g(v_2)|\ge |v_1-v_2| g'_+(\max(v_1, v_2))$, thanks to the concavity of $g$, and the convergence follows from condition (ii). We therefore conclude that $\norm{v}_{L^1(\bR^d,\mud)}\le \norm{f}_{L^1(\bR^d,\mud)}$, as desired.
\end{proof}

\begin{proof}[Proof of Lemma~\ref{lem:accretive}]
Applying Lemma~\ref{lemma:accretive} with $E=[0,\beta]$ and $g(x) = \frac{1}{2}\ln (1+x)$ (resp.\ $E=\R$ and $g(x)=x$) yields (i) (resp.\ (ii)). 
\end{proof}


\subsection{Derivation of Lemma~\ref{lem:domain condition}}\label{subsec:proof of lem:domain condition}
Let $\ab{\cdot, \cdot}$ denote the inner product in $L^2(\bR^d, \mud)$. For any $\lambda>0$, consider the bilinear form 
\[
\mathcal{B}[w,\varphi]:=\frac{\lambda}{2}\int_{\R^d} \nabla w \cdot \nabla \varphi d\mud = -\frac{\lambda}{2}\int_{\R^d} (\Delta_{\mud} w) \varphi d\mud,\quad \forall\ w,\varphi\in H^1_0(\bR^d, \mud),
\]
where the last equality follows from Remark~\ref{rem:Laplacian}.

\begin{definition}
We say $w\in H_0^1(\bR^d, \mud)\cap L^\infty(\bR^d, \mud)$ is  a
\begin{itemize}
\item [(i)] weak solution to \eqref{eq:nonlinear-w} w.r.t.\ $\mud$ if $\mathcal{B}[w, \varphi] = \ab{f+1-e^w, \varphi}$ for all $\varphi\in H^1_0(\bR^d, \mud)$;
\item[(ii)] weak subsolution to \eqref{eq:nonlinear-w} w.r.t.\ $\mud$ if  
$
\mathcal{B}[w, \varphi] \le \ab{f+1-e^w, \varphi}$ $\forall\varphi\in H^1_0(\bR^d, \mud)$, $\varphi\ge 0;
$
\item[(iii)] weak supersolution to \eqref{eq:nonlinear-w} w.r.t.\ $\mud$ if  
$\mathcal{B}[w, \varphi] \ge \ab{f+1-e^w, \varphi}$ $\forall \varphi\in H^1_0(\bR^d, \mud)$, $\varphi\ge 0.$
\end{itemize}
\end{definition}


The proof below relies on the construction of a weak solution $w^*$ to \eqref{eq:nonlinear-w} through a sequence of weak sub- and supersolutions.

\begin{proof}[Proof of Lemma~\ref{lem:domain condition}]
As $f\in \ubar{D(A)}$, we have $0\le f\le \beta$. It can then be verified directly that $\lbar{w}_0\equiv 0$ is a weak subsolution and $\ubar{w}_0 \equiv \ln (1+\beta)$ is a weak supersolution to \eqref{eq:nonlinear-w} w.r.t.\ $\mud$. Starting from $\lbar{w}_0$ and $\ubar{w}_0$, we will construct recursively a sequence $\{\lbar{w}_n, \ubar{w}_n\}_{n\in\N}$ such that $\lbar{w}_n$ and $\ubar{w}_n$, for each $n\in\N$, are the unique weak solutions in $H^1_0(\bR^d, \mud)$ w.r.t.\ $\mud$ to the linear PDEs
\begin{equation}\label{eq:subsolution}
\begin{aligned}
\alpha \lbar{w}_n-\frac{\lambda}{2} \Delta \lbar{w}_n - \frac{\lambda}{2} \nabla \ln \rhod \cdot \nabla \lbar{w}_n 
&= \alpha \lbar{w}_{n-1}+1-e^{\lbar{w}_{n-1}}+f,
\end{aligned}
\end{equation}
\begin{equation}\label{eq:supersolution}
\begin{aligned}
\alpha \ubar{w}_n-\frac{\lambda}{2} \Delta \ubar{w}_n - \frac{\lambda}{2} \nabla \ln \rhod \cdot \nabla \ubar{w}_n
&= \alpha \ubar{w}_{n-1}+1-e^{\ubar{w}_{n-1}} +f,
\end{aligned}
\end{equation}
where the constant $\alpha$ is chosen to satisfy $\alpha\ge 1+\beta$, so that $x\mapsto g(x):=\alpha x-e^x$ is strictly increasing on $(-\infty, \ln(1+\beta)]$.

First, given that $\lbar{w}_{0}, \ubar{w}_{0}\in L^\infty(\bR^d, \mud)$, the existence of unique solutions $\lbar{w}_1, \ubar{w}_{1}\in H^1_0(\bR^d, \mud)$ to \eqref{eq:subsolution}-\eqref{eq:supersolution} w.r.t.\ $\mud$ follows directly from the Lax-Milgram theorem applied to  the bilinear form $\alpha\ab{w, \varphi}+\mathcal{B}[w, \varphi]$. 
Moreover, we claim that $\lbar{w}_{1}, \ubar{w}_{1}$ again belong to $L^\infty(\bR^d, \mud)$ under the relation $\lbar{w}_{0} \le \lbar{w}_{1}\le \ubar{w}_{1}\le \ubar{w}_{0}$ $\mud$-a.e.
As $\ubar{w}_0$ is a weak supersolution to \eqref{eq:nonlinear-w} w.r.t. $\mud$, we have
\[
\alpha \ubar{w}_1-\frac{\lambda}{2} \Delta \ubar{w}_1 - \frac{\lambda}{2} \nabla \ln \rhod \cdot \nabla \ubar{w}_1
=\alpha \ubar{w}_{0}+1-e^{\ubar{w}_{0}}+f \le  
\alpha \ubar{w}_{0}-\frac{\lambda}{2} \Delta \ubar{w}_0 - \frac{\lambda}{2} \nabla \ln \rhod \cdot \nabla \ubar{w}_0
\]
in the weak form. Setting $w:=\ubar{w}_1-\ubar{w}_0\in H^1_0(\bR^d, \mud)$, we get $\alpha w-\frac{\lambda}{2} \Delta w- \frac{\lambda}{2} \nabla \ln \rhod \cdot \nabla w\le 0$ in the weak form. That is,
$
\alpha\ab{w, \varphi}+\mathcal{B}[w, \varphi]\le 0$, for all $\varphi\in H^1_0 (\bR^d, \mud)$, $\varphi\ge 0.
$
Taking $\varphi=w^+\in H^1_0(\bR^d, \mud)$ in the ineqaulity leads to the following implication:
\begin{equation}\label{MP}
\int_{\{w>0\}} \left(\alpha w^2 +\frac{\lambda}{2} |\nabla w|^2\right) d\mud \le 0\quad  \implies\quad   w\le 0\ \mud\hbox{-a.e.} 
\end{equation}
We then conclude $\ubar{w}_1\le \ubar{w}_0$ $\mud$-a.e.
Similarly, by the weak subsolution property of $\lbar{w}_0$, we may again use \eqref{MP} (with $w:=\lbar{w}_0- \lbar{w}_1$) to get $\lbar{w}_1\ge \lbar{w}_0$ $\mud$-a.e. Finally, thanks to $\lbar{w}_{0}\le \ubar{w}_{0}$, 
\begin{align*}
\alpha \lbar{w}_{1}-\frac{\lambda}{2} \Delta \lbar{w}_{1} - \frac{\lambda}{2} \nabla \ln \rhod \cdot \nabla \lbar{w}_{1}&=\alpha \lbar{w}_{0}+1-e^{\lbar{w}_{0}}+f\\
&\le \alpha \ubar{w}_{0}+1-e^{\ubar{w}_{0}}+f = \alpha \ubar{w}_{1}-\frac{\lambda}{2} \Delta \ubar{w}_1 - \frac{\lambda}{2} \nabla \ln \rhod \cdot \nabla \ubar{w}_1.
\end{align*}
Applying \eqref{MP} again (with $w:= \lbar{w}_1 - \ubar{w}_1$) gives $\lbar{w}_1\le \ubar{w}_1$ $\mud$-a.e. This finishes the proof of our claim that $\lbar{w}_{0} \le \lbar{w}_{1}\le \ubar{w}_{1}\le \ubar{w}_{0}$ $\mud$-a.e. As the same arguments above hold true for each iterate $n\in\N$, we obtain $\lbar{w}_{n-1}\le \lbar{w}_{n}\le \ubar{w}_{n}\le \ubar{w}_{n-1}$ $\mud$-a.e. for all $n\in\N$.

Now, for each $n\in\N$, thanks to $\lbar{w}_{n-1}\le \lbar{w}_{n}$,
\begin{align*}
\alpha \lbar{w}_{n+1}-\frac{\lambda}{2} \Delta \lbar{w}_{n+1} - \frac{\lambda}{2} \nabla \ln \rhod \cdot \nabla \lbar{w}_{n+1}&=\alpha \lbar{w}_{n}+1-e^{\lbar{w}_{n}}+f\\
&\hspace{-0.8in}\ge \alpha \lbar{w}_{n-1}+1-e^{\lbar{w}_{n-1}}+f=  \alpha \lbar{w}_{n}-\frac{\lambda}{2} \Delta \lbar{w}_{n} - \frac{\lambda}{2} \nabla \ln \rhod \cdot \nabla \lbar{w}_n.
\end{align*}
Hence, applying \eqref{MP} (with $w:= \lbar{w}_n - \lbar{w}_{n+1}$) gives $\lbar{w}_{n}\le \lbar{w}_{n+1}$ $\mud$-a.e. A similar argument shows $\ubar{w}_{n+1}\le \ubar{w}_{n}$ $\mud$-a.e. That is, $\{\lbar{w}_n\}_{n\in\N}$ is monotonically increasing, $\{\ubar{w}_n\}_{n\in\N}$ is monotonically decreasing, and all the functions are $[0, \ln(1+\beta)]$-valued. Thus, ${w}^*:=\lim_{n\to\infty} \ubar{w}_n$ $\mud$-a.e. is well-defined. We claim that ${w}^*$ is a weak solution to \eqref{eq:nonlinear-w} w.r.t.\ $\mud$.\footnote{One could alternatively show that $\lim_{n\to\infty} \lbar{w}_n$ is a weak solution to \eqref{eq:nonlinear-w} w.r.t.\ $\mud$.}

Note that the operator $L:=(\alpha I -\frac{\lambda}{2}(\Delta+\nabla \ln \rhod \cdot \nabla))^{-1}$ is continuous from $L^2(\bR^d, \mud)$ to $H^1_0(\bR^d, \mud)$. Indeed, given $f\in L^2(\R^d,\mud)$, by taking the test function $\varphi=w\in H^1_0(\bR^d, \mud)$ in the weak form of $\alpha w -\frac{\lambda}{2}(\Delta w+\nabla \ln \rhod \cdot \nabla w)=f$, we deduce from Remark~\ref{rem:Laplacian} that
\[
\alpha \norm{w}^2_{L^2(\bR^d, \mud)}+\frac{\lambda}{2}\norm{\nabla w}^2_{L^2(\bR^d, \mud)}=\int_{\R^d} f wd\mud\le \frac{1}{2\alpha}\norm{f}^2_{L^2(\bR^d, \mud)}+\frac{\alpha}{2}\norm{w}^2_{L^2(\bR^d, \mud)},
\]
where the inequality follows from Young's inequality. 
It follows that
\begin{equation}\label{eq:H1-bound}
\min(\alpha, \lambda) \norm{w}^2_{H^1_0(\bR^d, \mud)}\le \frac{1}{\alpha}\norm{f}^2_{L^2(\bR^d, \mud)},
\end{equation}
which implies the desired continuity of $L$. Now, by rewriting \eqref{eq:supersolution} as $\ubar{w}_n=L(\alpha \ubar{w}_{n-1}+1-e^{\ubar{w}_{n-1}}+f)$, we conclude from ${w}^*=\lim_{n\to\infty} \ubar{w}_n$ $\mud$-a.e., the bounded convergence theorem, and the continuity of $L$ that $\ubar{w}_n\to \ubar{w}:=L(\alpha {w}^*+1-e^{{w}^*}+f)$ in $H^1_0(\bR^d, \mud)$. By passing to a subsequence (without relabeling), we have $\ubar{w}_n\to \ubar{w}$ $\mud$-a.e., which entails ${w}^*=\ubar{w}$ $\mud$-a.e. Hence, we have $w^*\in H^1_0(\bR^d, \mud)$ and ${w}^*=L(\alpha {w}^*+1-e^{{w}^*}+f)$, i.e.\ $w^*$ is a weak solution to \eqref{eq:nonlinear-w} w.r.t.\ $\mud$. By a direct calculation, $v:=e^{w^*}-1\in H^1_0(\bR^d, \mud)$ is a weak solution to $(I+\lambda A)v=f$ w.r.t.\ $\mud$. Finally, as ${w}^*=\lim_{n\to\infty} \ubar{w}_n$ by construction satisfies $0\le {w}^*\le \ln(1+\beta)$, $v$ fulfills $0\le v\le \beta$ and thus lies in $D(A)$. 
\end{proof}


\subsection{Proof of Proposition~\ref{prop:uniqueness}}\label{subsec:proof of prop:uniqueness}
Fix $\eps>0$. By Remark~\ref{rem:Laplacian}, we may conclude from Lax-Milgram's theorem that for any $f\in L^2(\bR^d, \mud)$, there is a unique weak solution $w\in H^1_0(\bR^d, \mud)$ to $(\eps I-\Delta_{\mud})w=f$ w.r.t.\ $\mud$. That is, the operator $\Gamma_\eps :=(\eps I-\Delta_{\mud})^{-1}: L^2(\bR^d, \mud) \to H^1_0(\bR^d, \mud)$ is well-defined. As $\ln \rhod \in H^1(\bR^d, \mud)$ (by Assumption~\ref{asm}), let us take $\{b_n\} $ in $C^\infty_c(\bR^d)$ such that $b_n\to \nabla \ln \rhod$ in $L^2(\bR^d, \mud)$. Then, for each $f\in L^2(\bR^d, \mud)$, we consider the smoothed problem
\begin{equation}\label{A^n}
(\eps I-\Delta^n_{\mud}) w =f \quad \text{with}\ \ \Delta^n_{\mud}:= \Delta + b_n\cdot \nabla,
\end{equation}
and denote by $\Gamma^n_\eps f :=(\eps I-\Delta^n_{\mud})^{-1}f$ the unique weak solution $w\in H^1_0(\bR^d, \mud)$ w.r.t. $\mud$. 

{\bf Step 1:} We show that for any $f\in L^\infty(\bR^d, \mud)$, 
\begin{equation}\label{eq:Gamma-n-conv}
\Gamma^n_\eps f\rightarrow \Gamma_\eps f \quad \text{in}\ \ L^1(\bR^d, \mud).
\end{equation}
Take $\{f_k\}$ in $C^\infty_c(\R^d)$ such that $f_k\to f$ in $L^1(\bR^d,\mud)$. Without loss of generality, we may assume $|f_k| \le \norm{f}_{L^\infty(\R^d,\mud)}$. For each $k\in\N$, by noting  $\Gamma^n_\eps f_k\in C^\infty(\R^d)$, we conclude from the maximum principle \cite[Theorem 2.9.2]{Krylov-book-96} that 
\begin{equation}\label{Gamma^n bdd}
\norm{\Gamma^n_\eps f_k}_{L^\infty(\R^d,\mud)}\le \frac{1}{\eps}\norm{f_k}_{L^\infty(\R^d,\mud)} \le \frac{1}{\eps}\norm{f}_{L^\infty(\R^d,\mud)}. 
\end{equation}
Now,  in view of \eqref{A^n} and \eqref{Laplacian}, the equation $(\eps I-\Delta^n_{\mud})\Gamma^n_\eps f_k=f_k$ can be written as
\[
\big(\eps I-\Delta_{\mud} - (b_n-\nabla\ln\rhod)\cdot\nabla\big)\Gamma^n_\eps f_k=f_k. 
\]
Let us multiply this equation by $\Gamma^n_\eps f_k$ and integrate it w.r.t.\ $\mud$. With the aid of Remark~\ref{rem:Laplacian}, Young's inequality, \eqref{Gamma^n bdd}, and H\"{o}lder's inequality, we get
\begin{align}\label{gradient esti.}
&\eps \norm{\Gamma^n_\eps f_k}^2_{L^2(\mud)}+\norm{\nabla \Gamma^n_\eps f_k}^2_{L^2(\mud)}=\int_{\R^d} f_k \Gamma^n_\eps f_k d\mud + \int_{\R^d} \Gamma^n_\eps f_k (b_n-\nabla \ln \rhod)\cdot \nabla \Gamma^n_\eps f_k d\mud\notag\\
&\le \frac{1}{2\eps}\norm{f_k}^2_{L^2(\mud)}+\frac{\eps}{2} \norm{\Gamma^n_\eps f_k}^2_{L^2(\mud)}+\frac{1}{\eps}\norm{f}_{L^\infty(\mud)}\norm{b_n-\nabla \ln \rhod}_{L^2(\mud)}\norm{\nabla \Gamma^n_\eps f_k}_{L^2(\mud)}\\
&\le \frac{1}{2\eps}\norm{f}^2_{L^\infty(\mud)}+\frac{\eps}{2} \norm{\Gamma^n_\eps f_k}^2_{L^2(\mud)}+\frac{1}{2\eps^2}\norm{f}^2_{L^\infty(\mud)}\norm{b_n-\nabla \ln \rhod}^2_{L^2(\mud)}+\frac{1}{2}\norm{\nabla \Gamma^n_\eps f_k}^2_{L^2(\mud)}.\notag
\end{align}
It follows that
\begin{equation*}
\eps \norm{\Gamma_\eps^n f_k}^2_{L^2(\mud)}+ \norm{\nabla \Gamma^n_\eps f_k}^2_{L^2(\mud)}\le \frac{1}{\eps}\norm{f}^2_{L^\infty(\mud)} +\frac{1}{\eps^2}\norm{f}^2_{L^\infty(\mud)}\norm{b_n-\nabla \ln \rhod}^2_{L^2(\mud)}<\infty.
\end{equation*}
This shows that $\{\Gamma^n_\eps f_k\}_{k\in\N}$ is bounded in $H^1_0(\bR^d, \mud)$. Hence, by passing to a subsequence (without relabeling), $\Gamma^n_\eps f_k$ converges to some $\hat w$ weakly in $H^1_0(\bR^d, \mud)$ as $k\to\infty$.
This implies that if we let $k\rightarrow \infty$ in the weak form of $(\eps I-\Delta^n_{\mud})\Gamma^n_\eps f_k=f_k$, we will get $(\eps I-\Delta^n_{\mud})\hat w=f$ in the weak form, which indicates $\hat w=\Gamma^n_\eps f$. That is, we have $\Gamma^n_\eps f_k\to \Gamma^n_\eps f$ weakly in $H^1_0(\bR^d, \mud)$ as $k\to\infty$. By taking the test function $\varphi = \Gamma^n_\eps f\in H^1_0(\bR^d, \mud)$ in the weak form of $(\eps I-\Delta^n_{\mud}) \Gamma^n_\eps f =f$, we can repeat the gradient estimate \eqref{gradient esti.}, with $f_k$ replaced by $f$, thereby obtaining
\begin{equation}\label{eq:Gamma-n-2}
\eps \norm{\Gamma_\eps^n f}^2_{L^2(\mud)}+ \norm{\nabla \Gamma^n_\eps f}^2_{L^2(\mud)}\le \frac{1}{\eps}\norm{f}^2_{L^\infty(\mud)} +\frac{1}{\eps^2}\norm{f}^2_{L^\infty(\mud)}\norm{b_n-\nabla \ln \rhod}^2_{L^2(\mud)}.
\end{equation} 
By subtracting the equation $(\eps I-\Delta_{\mud})\Gamma_\eps f=f$ from $(\eps I-\Delta^n_{\mud})\Gamma^n_\eps f=f$, we get
\[
\eps(\Gamma^n_\eps f-\Gamma_\eps f)-\Delta (\Gamma^n_\eps f-\Gamma_\eps f) - \nabla \ln \rhod\cdot \nabla (\Gamma^n_\eps f-\Gamma_\eps f) + (\nabla \ln \rhod-b_n) \cdot \nabla \Gamma^n_\eps f=0.
\]
This implies $\left(\eps I-\Delta_{\mud}\right)(\Gamma^n_\eps f-\Gamma_\eps f)=(b_n-\nabla \ln \rhod) \cdot \nabla \Gamma^n_\eps f$, i.e.\ $\Gamma_\eps((b_n-\nabla \ln \rhod) \cdot \nabla \Gamma^n_\eps f) = \Gamma^n_\eps f-\Gamma_\eps f$. As $\Gamma_\eps 0 = 0$ trivially, we deduce from Lemma~\ref{lem:accretive} (ii) that
\begin{align}
\eps\norm{\Gamma^n_\eps f-\Gamma_\eps f}_{L^1(\bR^d, \mud)}& \le \norm{(b_n-\nabla \ln \rhod) \cdot \nabla \Gamma^n_\eps f}_{L^1(\bR^d, \mud)}\notag\\
& \le \norm{b_n-\nabla \ln \rhod}_{L^2(\bR^d, \mud)}\norm{\nabla\Gamma^n_\eps f}_{L^2(\bR^d, \mud)}\rightarrow 0  \quad \text{as } n\rightarrow \infty,\label{conv}
\end{align}
where the convergence follows from $b_n\to \nabla \ln \rhod$ in $L^2(\bR^d, \mud)$ and the $L^2(\bR^d, \mud)$-boundedness of $\{\nabla\Gamma^n_\eps f\}_{n\in\N}$ (thanks to \eqref{eq:Gamma-n-2}). This readily gives \eqref{eq:Gamma-n-conv}.

{\bf Step 2:} Let $v_1, v_2\in C([0,\infty); L^1(\bR^d, \mud))\cap L_+^\infty([0,\infty)\times \bR^d,\mud)$ be two weak solutions to \eqref{eq:v0} w.r.t.\ $\mud$. By setting $v:=v_1-v_2$ and $h:=\ln \frac{1+v_1}{1+v_2}$, we aim to show that $(\Gamma_\eps v)_t=\frac{1}{2}(\eps \Gamma_\eps h- h)$ in the weak form, i.e.
\begin{equation}\label{step 2 to show}
\int_0^\infty \int_{\bR^d}  \left(\psi_t \Gamma_\eps v+\frac{1}{2} \psi (\eps \Gamma_\eps h- h)\right) d\mud dt=0  \quad \forall\, \psi\in C^{\infty}_c((0,\infty)\times \bR^d).
\end{equation}
By construction, we have $v(0)=0$ and $v_t=\frac{1}{2}\Delta_{\mud} h$ in the weak form, i.e.
\begin{equation}\label{hehe}
\int_0^\infty \int_{\bR^d} \left(v\varphi_t +\frac{1}{2} h \Delta_{\mud} \varphi\right) d\mud dt=0,\quad \forall\varphi\in C^{\infty}_c((0,\infty)\times \bR^d),
\end{equation}
where we use the calculation in Remark~\ref{rem:Laplacian} twice. For any $\psi \in C^\infty_c((0,\infty)\times \bR^d)$, define $\varphi(t,\cdot):=\Gamma^n_\eps \psi(t,\cdot)$ for all $t\ge 0$. That is, $\varphi(t,\cdot)$ is the solution to the (linear) smoothed problem \eqref{A^n} (with $f(\cdot)=\psi(t,\cdot)\in C^\infty_c(\bR^d)$), which implies $\varphi(t,\cdot)\in C^\infty(\bR^d)$. By differentiating the equation $(\eps I-\Delta^n_{\mud}) \Gamma^n_\eps \psi(t,\cdot) =\psi(t,\cdot)$ with respect to $t$, we observe that
\begin{equation}\label{partial t}
\frac{\partial^m}{\partial t^m} \varphi(t,\cdot) =  \frac{\partial^m}{\partial t^m} \Gamma^n_\eps \psi(t,\cdot) = \Gamma^n_\eps \bigg(\frac{\partial^m}{\partial t^m} \psi(t,\cdot) \bigg),\quad \forall m\in\N. 
\end{equation}
As a result, $\varphi \in C^\infty((0,\infty)\times \bR^d)$. We will further show that $\varphi$ is compactly supported. Let $\text{supp}(\psi)\subseteq (0,\infty)\times\R^d$ be the compact support of $\psi$, $H\subset (0,\infty)$ be a bounded open interval containing the projection of $\text{supp}(\psi)$ onto $(0,\infty)$, and $B\subset \R^d$ be an open ball containing the projection of $\text{supp}(\psi)$ onto $\bR^d$. For each $t\in H$, consider the Dirichlet problem $(\eps I-\Delta^n_{\mud}) w(\cdot)=\psi(t,\cdot)$ in $B$ and $w=0$ on $\partial B$, and denote by $\hat\varphi(t,\cdot)\in C^\infty(B)$ the unique solution. Now, we extend the domain of $\hat\varphi$ to $(0,\infty)\times\bR^d$ by setting $\hat \varphi(t,y):=0$ whenever $t\in(0,\infty)\setminus H$ or $y\in B^c$. Then, $\hat\varphi$ is compactly supported and for any $t>0$, $\hat \varphi(t,\cdot)\in H^1_0(\bR^d, \mud)$ satisfies $(\eps I-\Delta^n_{\mud}) \hat \varphi(t,\cdot)=\psi(t,\cdot)$ on $\bR^d$. By the uniqueness of solutions to \eqref{A^n} in $H^1_0(\bR^d, \mud)$, we must have $\varphi(t,\cdot)=\hat \varphi(t,\cdot)$ for all $t>0$. Consequently, $\varphi=\Gamma^n_\eps\psi$ is compactly supported with $\text{supp}(\varphi)\subseteq H\times B$ for all $n\in\N$. With $\varphi=\Gamma^n_\eps\psi \in C_c^\infty((0,\infty)\times \bR^d)$, we obtain from \eqref{hehe} that
\begin{equation}\label{hehe'}
\int_0^\infty \int_{\bR^d} \bigg(v (\Gamma^n_\eps \psi)_t +\frac{1}{2} h \Delta_{\mud} (\Gamma^n_\eps \psi)\bigg) d\mud dt=0,\quad \forall \psi\in C^{\infty}_c((0,\infty)\times \bR^d).
\end{equation}
In view of \eqref{Laplacian}, \eqref{A^n}, and the equation $(\eps I-\Delta^n_{\mud})\Gamma^n_\eps \psi=\psi$, we have
\[
\Delta_{\mud}(\Gamma^n_\eps\psi) = \Delta^n_{\mud}(\Gamma^n_\eps\psi)-(b_n-\nabla\ln\rhod)\cdot\nabla\Gamma^n_\eps\psi=\eps \Gamma^n_\eps \psi- \psi-(b_n-\nabla\ln\rhod)\cdot\nabla\Gamma^n_\eps\psi.
\]
This, together with $(\Gamma^n_\eps \psi)_t =\Gamma^n_\eps \psi_t$ (by \eqref{partial t} with $m=1$), allows us to rewrite \eqref{hehe'} as
\begin{equation*}
\int_0^\infty \int_{\bR^d} \bigg(v \Gamma^n_\eps \psi_t +\frac{1}{2} h (\eps \Gamma^n_\eps \psi- \psi)\bigg) d\mud dt= \int_0^\infty \int_{\bR^d} \frac{1}{2} h (b_n-\nabla\ln\rhod)\cdot\nabla\Gamma^n_\eps\psi\ d\mud dt.
\end{equation*}
As $h$ is bounded and $\text{supp}(\nabla\Gamma^n_\eps\psi)\subset H\times E$ for all $n\in\N$, we see that the right-hand side above vanishes as $n\to\infty$, by using H\"{o}lder's inequality and arguing as in \eqref{conv} for the convergence. It follows that as $n\to\infty$, the previous equation becomes
\begin{equation}\label{step 2 to show'}
\int_0^\infty \int_{\bR^d} \left(v \Gamma_\eps \psi_t +\frac{1}{2} h (\eps \Gamma_\eps \psi- \psi)\right) d\mud dt=0, 
\end{equation}
where the convergence of the left-hand side follows from the boundedness of $v$ and $h$, $\text{supp}(\Gamma^n_\eps\psi)\subset H\times E$ for all $n\in\N$, and \eqref{eq:Gamma-n-conv}.
Finally, observe that $\Gamma_\eps$ is symmetric in the sense that 
\begin{equation}\label{sym}
\ab{\Gamma_\eps f, g}=\ab{f, \Gamma_\eps g},\quad \forall f,g\in L^2(\R^d,\mud). 
\end{equation}
Indeed, by taking the test function $\varphi\in H^1(\R^d,\mud)$ as $\Gamma_\eps g$ (resp.\ $\Gamma_\eps f$) in the weak form of $f = (\eps I-\Delta_{\mud})\Gamma_\eps f$ (resp.\ $g = (\eps I-\Delta_{\mud})\Gamma_\eps g$), we get 
\begin{equation}\label{sym proof}
\ab{f,\Gamma_\eps g}=\eps \ab{\Gamma_\eps f,\Gamma_\eps g}+\ab{\nabla \Gamma_\eps f,\nabla\Gamma_\eps g}=\ab{g,\Gamma_\eps f}, 
\end{equation}
where we use the calculation in Remark~\ref{rem:Laplacian}. Hence, \eqref{step 2 to show'} is equivalent to \eqref{step 2 to show}, as desired. 

{\bf Step 3:} Now we set out to prove \eqref{uniqueness}. 
Define $g^\eps:[0,\infty)\to\R$ by 
\begin{equation}\label{g^eps}
g^\eps(t):=\ab{\Gamma_\eps v(t), v(t)}=\eps\norm{\Gamma_\eps v(t)}^2_{L^2(\bR^d, \mud)}+\norm{\nabla \Gamma_\eps v(t)}^2_{L^2(\bR^d, \mud)},
\end{equation}
where the last equality follows by taking $f=g=v(t)$ in \eqref{sym proof}. 
Rearranging the equation $(\eps I-\Delta_{\mud})\Gamma_\eps h(t) = h(t)$ yields $\Delta_{\mud}(\Gamma_\eps h(t))=\eps \Gamma_\eps h(t)-h(t)$. By a direct calculation, this in turn implies $(\eps I-\Delta_{\mud})(\eps \Gamma_\eps h-h)=\Delta_{\mud} h$, i.e.\ $\Gamma_\eps (\Delta_{\mud} h)=\eps \Gamma_\eps h-h$. Hence, by recalling $(\Gamma_\eps v)_t=\Gamma_\eps v_t$ and $v_t=\frac{1}{2}\Delta_{\mud} h$, we have 
\begin{align*}
(g^\eps)'=\ab{(\Gamma_\eps v)_t, v}+\ab{\Gamma_\eps v, v_t}&=\left\langle\frac{1}{2}(\eps \Gamma_\eps h- h), v\right\rangle+\left\langle\Gamma_\eps v, \frac{1}{2}\Delta_{\mud} h\right\rangle\\
&=\left\langle\frac{1}{2}(\eps \Gamma_\eps h- h), v\right\rangle+\left\langle v, \frac{1}{2}(\eps \Gamma_\eps h- h)\right\rangle\\
& = \left\langle \eps \Gamma_\eps h- h, v\right\rangle = \ab{\eps h, \Gamma_\eps v} -\ab{h,v}\le \eps\ab{h, \Gamma_\eps v},
\end{align*}
where the third and fifth equalities follow from \eqref{sym} and the inequality is due to $hv = \ln(\frac{1+v_1}{1+v_2}) (v_1-v_2)\ge 0$. This, together with $g^\eps(0)=0$ (as $v(0)=0$), gives
\begin{align*}
g^\eps(t)\le \int_0^t \eps\ab{h(s), \Gamma_\eps v(s)} ds &\le \frac{\eps}{2} \norm{v}_{L^\infty(\R^d,\mud)} \int_0^t \left(1+\norm{\Gamma_\eps v(s)}^2_{L^2(\bR^d, \mud)}\right) ds\\
&\le \frac{1}{2} \norm{v}_{L^\infty(\R^d,\mud)} \left(\eps t + \int_0^t g^\eps(s)ds\right),
\end{align*}
where the second inequality follows from $|h|\le |v|$ (due to concavity of $\ln(\cdot)$ and $v_i\ge 0$) and Young's inequality, and the third inequality stems from \eqref{g^eps}. Then, by \eqref{g^eps} and Gr\"onwall's inequality,
\[
\eps\norm{\Gamma_\eps v}^2_{L^2(\bR^d, \mud)}+\norm{\nabla \Gamma_\eps v}^2_{L^2(\bR^d, \mud)}=g^\eps(t)\le \frac{\eps t}{2} \norm{v}_{L^\infty(\R^d,\mud)}\exp\left( \frac{t}{2} \norm{v}_{L^\infty(\R^d,\mud)}\right)\rightarrow 0\quad \text{as}\ \eps\rightarrow 0. 
\]
This implies $\eps \Gamma_\eps v, \nabla \Gamma_\eps v \rightarrow 0$ in $L^2(\bR^d, \mud)$, uniformly in $t$ on compact intervals. Thus, taking any test function $\varphi\in C^\infty_c ((0,\infty)\times \bR^d)$ in the weak form of $v=(\eps I -\Delta_{\mud}) \Gamma_\eps v$ gives
\[
\int_0^\infty \int_{\bR^d} v \varphi d\mud dt= \int_0^\infty \int_{\bR^d} \left(\eps \Gamma_\eps v \varphi +\nabla \Gamma_\eps v \cdot \nabla \varphi\right) d\mud dt \rightarrow 0\quad \hbox{as}\ \eps\to 0,
\]
which gives \eqref{uniqueness}.

\vskip 0.2in
\bibliography{refs}

\begin{thebibliography}{39}
\providecommand{\natexlab}[1]{#1}
\providecommand{\url}[1]{\texttt{#1}}
\expandafter\ifx\csname urlstyle\endcsname\relax
  \providecommand{\doi}[1]{doi: #1}\else
  \providecommand{\doi}{doi: \begingroup \urlstyle{rm}\Url}\fi

\bibitem[Ambrosio(2004)]{Ambrosio04}
Luigi Ambrosio.
\newblock Transport equation and {C}auchy problem for {$BV$} vector fields.
\newblock \emph{Invent. Math.}, 158\penalty0 (2):\penalty0 227--260, 2004.

\bibitem[Ambrosio et~al.(2005)Ambrosio, Gigli, and
  Savar\'{e}]{Ambrosio-book-05}
Luigi Ambrosio, Nicola Gigli, and Giuseppe Savar\'{e}.
\newblock \emph{Gradient flows in metric spaces and in the space of probability
  measures}.
\newblock Lectures in Mathematics ETH Z\"{u}rich. Birkh\"{a}user Verlag, Basel,
  2005.

\bibitem[Ansari et~al.(2020)Ansari, Scarlett, and Soh]{Ansari20}
Abdul~Fatir Ansari, Jonathan Scarlett, and Harold Soh.
\newblock A characteristic function approach to deep implicit generative
  modeling.
\newblock In \emph{Proceedings of the IEEE/CVF Conference on Computer Vision
  and Pattern Recognition (CVPR)}, June 2020.

\bibitem[Ansari et~al.(2021)Ansari, Ang, and Soh]{Ansari21}
Abdul~Fatir Ansari, Ming~Liang Ang, and Harold Soh.
\newblock Refining deep generative models via discriminator gradient flow.
\newblock In \emph{International Conference on Learning Representations}, 2021.

\bibitem[Arjovsky and Bottou(2017)]{AB17}
Martin Arjovsky and Leon Bottou.
\newblock Towards principled methods for training generative adversarial
  networks.
\newblock In \emph{International Conference on Learning Representations},
  volume~5, 2017.

\bibitem[Arjovsky et~al.(2017)Arjovsky, Chintala, and Bottou]{WGAN}
Martin Arjovsky, Soumith Chintala, and L{\'e}on Bottou.
\newblock {W}asserstein generative adversarial networks.
\newblock In \emph{Proceedings of the 34th International Conference on Machine
  Learning}, volume~70, pages 214--223, 06--11 Aug 2017.

\bibitem[Barbu(2010)]{Barbu-book-2010}
Viorel Barbu.
\newblock \emph{Nonlinear differential equations of monotone types in {B}anach
  spaces}.
\newblock Springer Monographs in Mathematics. Springer, New York, 2010.

\bibitem[Barbu and R\"{o}ckner(2020)]{BR20}
Viorel Barbu and Michael R\"{o}ckner.
\newblock From nonlinear {F}okker-{P}lanck equations to solutions of
  distribution dependent {SDE}.
\newblock \emph{Ann. Probab.}, 48\penalty0 (4):\penalty0 1902--1920, 2020.

\bibitem[Bi\'{n}kowski et~al.(2018)Bi\'{n}kowski, Sutherland, Arbel, and
  Gretton]{Binkowski18}
Miko\l{l}aj Bi\'{n}kowski, Danica~J. Sutherland, Michael Arbel, and Arthur
  Gretton.
\newblock Demystifying {MMD} {GAN}s, 2018.
\newblock Preprint, available at https://arxiv.org/abs/1801.01401.

\bibitem[Br\'{e}zis and Crandall(1979)]{BC79}
Ha\"{\i}m Br\'{e}zis and Michael~G. Crandall.
\newblock Uniqueness of solutions of the initial-value problem for
  {$u_{t}-\Delta \varphi (u)=0$}.
\newblock \emph{J. Math. Pures Appl. (9)}, 58\penalty0 (2):\penalty0 153--163,
  1979.

\bibitem[Cardaliaguet et~al.(2015)Cardaliaguet, Delarue, Lasry, and
  Lions]{CDLL15}
Pierre Cardaliaguet, François Delarue, Jean-Michel Lasry, and Pierre-Louis
  Lions.
\newblock The master equation and the convergence problem in mean field games.
\newblock 2015.
\newblock Preprint, available at https://arxiv.org/abs/1509.02505.

\bibitem[Carmona and Delarue(2018{\natexlab{a}})]{CD-book-18-I}
Ren\'{e} Carmona and Fran\c{c}ois Delarue.
\newblock \emph{Probabilistic theory of mean field games with applications.
  {I}}, volume~83 of \emph{Probability Theory and Stochastic Modelling}.
\newblock Springer, Cham, 2018{\natexlab{a}}.
\newblock Mean field FBSDEs, control, and games.

\bibitem[Carmona and Delarue(2018{\natexlab{b}})]{CD-book-18-II}
Ren\'{e} Carmona and Fran\c{c}ois Delarue.
\newblock \emph{Probabilistic theory of mean field games with applications.
  {II}}, volume~84 of \emph{Probability Theory and Stochastic Modelling}.
\newblock Springer, Cham, 2018{\natexlab{b}}.
\newblock Mean field games with common noise and master equations.

\bibitem[Delarue et~al.(2019)Delarue, Lacker, and Ramanan]{DLR19}
Fran\c{c}ois Delarue, Daniel Lacker, and Kavita Ramanan.
\newblock From the master equation to mean field game limit theory: a central
  limit theorem.
\newblock \emph{Electron. J. Probab.}, 24:\penalty0 Paper No. 51, 54, 2019.

\bibitem[Denton et~al.(2015)Denton, Chintala, szlam, and Fergus]{Denton15}
Emily~L Denton, Soumith Chintala, arthur szlam, and Rob Fergus.
\newblock Deep generative image models using a {L}aplacian pyramid of
  adversarial networks.
\newblock In \emph{Advances in Neural Information Processing Systems},
  volume~28, 2015.

\bibitem[Endres and Schindelin(2003)]{ES03}
D.M. Endres and J.E. Schindelin.
\newblock A new metric for probability distributions.
\newblock \emph{IEEE Transactions on Information Theory}, 49\penalty0
  (7):\penalty0 1858--1860, 2003.

\bibitem[Evans(1998)]{Evans-book-98}
Lawrence~C. Evans.
\newblock \emph{Partial differential equations}, volume~19 of \emph{Graduate
  Studies in Mathematics}.
\newblock American Mathematical Society, Providence, RI, 1998.

\bibitem[Gao et~al.(2019)Gao, Jiao, Wang, Wang, Yang, and Zhang]{Gao19}
Yuan Gao, Yuling Jiao, Yang Wang, Yao Wang, Can Yang, and Shunkang Zhang.
\newblock Deep generative learning via variational gradient flow.
\newblock In \emph{Proceedings of the 36th International Conference on Machine
  Learning}, volume~97 of \emph{Proceedings of Machine Learning Research},
  pages 2093--2101, 2019.

\bibitem[Gao et~al.(2020)Gao, Huang, Jiao, and Liu]{Gao20}
Yuan Gao, Jian Huang, Yuling Jiao, and Jin Liu.
\newblock Learning implicit generative models with theoretical guarantees.
\newblock 2020.
\newblock Preprint. Available at https://arxiv.org/abs/2002.02862.

\bibitem[Goodfellow et~al.(2014)Goodfellow, Pouget-Abadie, Mirza, Xu,
  Warde-Farley, Ozair, Courville, and Bengio]{Goodfellow14}
Ian Goodfellow, Jean Pouget-Abadie, Mehdi Mirza, Bing Xu, David Warde-Farley,
  Sherjil Ozair, Aaron Courville, and Yoshua Bengio.
\newblock Generative adversarial nets.
\newblock In \emph{Advances in Neural Information Processing Systems 27}, pages
  2672--2680. 2014.

\bibitem[Goodfellow(2016)]{Goodfellow16}
Ian~J. Goodfellow.
\newblock {NIPS} 2016 tutorial: Generative adversarial networks.
\newblock 2016.
\newblock Available at https://arxiv.org/abs/1701.00160.

\bibitem[Gulrajani et~al.(2017)Gulrajani, Ahmed, Arjovsky, Dumoulin, and
  Courville]{Gulrajani17}
Ishaan Gulrajani, Faruk Ahmed, Martin Arjovsky, Vincent Dumoulin, and Aaron~C
  Courville.
\newblock Improved training of wasserstein {GAN}s.
\newblock In \emph{Advances in Neural Information Processing Systems},
  volume~30, 2017.

\bibitem[Henry(1981)]{Henry-book-81}
Daniel Henry.
\newblock \emph{Geometric theory of semilinear parabolic equations}, volume 840
  of \emph{Lecture Notes in Mathematics}.
\newblock Springer-Verlag, Berlin-New York, 1981.

\bibitem[Heusel et~al.(2017)Heusel, Ramsauer, Unterthiner, Nessler, and
  Hochreiter]{Heusel17}
Martin Heusel, Hubert Ramsauer, Thomas Unterthiner, Bernhard Nessler, and Sepp
  Hochreiter.
\newblock {GAN}s trained by a two time-scale update rule converge to a local
  nash equilibrium.
\newblock In \emph{Advances in Neural Information Processing Systems},
  volume~30, 2017.

\bibitem[Hu et~al.(2021)Hu, Ren, \v{S}i\v{s}ka, and Szpruch]{HRSS21}
Kaitong Hu, Zhenjie Ren, David \v{S}i\v{s}ka, and \L~ukasz Szpruch.
\newblock Mean-field {L}angevin dynamics and energy landscape of neural
  networks.
\newblock \emph{Ann. Inst. Henri Poincar\'{e} Probab. Stat.}, 57\penalty0
  (4):\penalty0 2043--2065, 2021.

\bibitem[Karatzas and Shreve(1991)]{KS-book-91}
Ioannis Karatzas and Steven~E. Shreve.
\newblock \emph{Brownian motion and stochastic calculus}, volume 113 of
  \emph{Graduate Texts in Mathematics}.
\newblock Springer-Verlag, New York, second edition, 1991.

\bibitem[Krylov(1996)]{Krylov-book-96}
N.~V. Krylov.
\newblock \emph{Lectures on elliptic and parabolic equations in {H}\"{o}lder
  spaces}, volume~12 of \emph{Graduate Studies in Mathematics}.
\newblock American Mathematical Society, Providence, RI, 1996.

\bibitem[Ledig et~al.(2017)Ledig, Theis, Huszar, Caballero, Cunningham, Acosta,
  Aitken, Tejani, Totz, Wang, and Shi]{Ledig17}
Christian Ledig, Lucas Theis, Ferenc Huszar, Jose Caballero, Andrew Cunningham,
  Alejandro Acosta, Andrew Aitken, Alykhan Tejani, Johannes Totz, Zehan Wang,
  and Wenzhe Shi.
\newblock Photo-realistic single image super-resolution using a generative
  adversarial network.
\newblock In \emph{Proceedings of the IEEE Conference on Computer Vision and
  Pattern Recognition (CVPR)}, July 2017.

\bibitem[Li et~al.(2017)Li, Chang, Cheng, Yang, and Poczos]{MMDGAN2}
Chun-Liang Li, Wei-Cheng Chang, Yu~Cheng, Yiming Yang, and Barnabas Poczos.
\newblock Mmd gan: Towards deeper understanding of moment matching network.
\newblock In \emph{Advances in Neural Information Processing Systems 30}, pages
  2203--2213. 2017.

\bibitem[McCulloch and Wagner(2020)]{MW20}
Josie McCulloch and Christian Wagner.
\newblock On the choice of similarity measures for type-2 fuzzy sets.
\newblock \emph{Information Sciences}, 510:\penalty0 135--154, 2020.

\bibitem[Metz et~al.(2017)Metz, Poole, Pfau, and Sohl-Dickstein]{Metz17}
Luke Metz, Ben Poole, David Pfau, and Jascha Sohl-Dickstein.
\newblock Unrolled generative adversarial networks.
\newblock In \emph{International Conference on Learning Representations}, 2017.

\bibitem[Miyato et~al.(2018)Miyato, Kataoka, Koyama, and Yoshida]{Miyato18}
Takeru Miyato, Toshiki Kataoka, Masanori Koyama, and Yuichi Yoshida.
\newblock Spectral normalization for generative adversarial networks.
\newblock 2018.
\newblock Preprint, available at https://arxiv.org/abs/1802.05957.

\bibitem[Mroueh and Nguyen(2021)]{MN21}
Youssef Mroueh and Truyen Nguyen.
\newblock On the convergence of gradient descent in {GAN}s: {MMD} {GAN} as a
  gradient flow.
\newblock In \emph{Proceedings of The 24th International Conference on
  Artificial Intelligence and Statistics}, volume 130 of \emph{Proceedings of
  Machine Learning Research}, pages 1720--1728, 2021.

\bibitem[Nielsen and Nock(2014)]{NN14}
Frank Nielsen and Richard Nock.
\newblock On the chi square and higher-order chi distances for approximating
  f-divergences.
\newblock \emph{IEEE Signal Processing Letters}, 21\penalty0 (1):\penalty0
  10--13, 2014.

\bibitem[\"{O}sterreicher and Vajda(2003)]{OV03}
Ferdinand \"{O}sterreicher and Igor Vajda.
\newblock A new class of metric divergences on probability spaces and its
  applicability in statistics.
\newblock \emph{Ann. Inst. Statist. Math.}, 55\penalty0 (3):\penalty0 639--653,
  2003.

\bibitem[Salimans et~al.(2016)Salimans, Goodfellow, Zaremba, Cheung, Radford,
  Chen, and Chen]{Salimans16}
Tim Salimans, Ian Goodfellow, Wojciech Zaremba, Vicki Cheung, Alec Radford,
  Xi~Chen, and Xi~Chen.
\newblock Improved techniques for training gans.
\newblock In \emph{Advances in Neural Information Processing Systems},
  volume~29, 2016.

\bibitem[Trevisan(2016)]{Trevisan16}
Dario Trevisan.
\newblock Well-posedness of multidimensional diffusion processes with weakly
  differentiable coefficients.
\newblock \emph{Electron. J. Probab.}, 21:\penalty0 Paper No. 22, 41, 2016.

\bibitem[Vondrick et~al.(2016)Vondrick, Pirsiavash, and Torralba]{Vondrick16}
Carl Vondrick, Hamed Pirsiavash, and Antonio Torralba.
\newblock Generating videos with scene dynamics.
\newblock In \emph{Advances in Neural Information Processing Systems},
  volume~29, 2016.

\bibitem[Wiese et~al.(2020)Wiese, Knobloch, Korn, and Kretschmer]{Wiese20}
Magnus Wiese, Robert Knobloch, Ralf Korn, and Peter Kretschmer.
\newblock Quant {GAN}s: deep generation of financial time series.
\newblock \emph{Quantitative Finance}, 20\penalty0 (9):\penalty0 1419--1440,
  2020.

\end{thebibliography}

\end{document}